\documentclass{article}

     \usepackage[preprint,nonatbib]{neurips_2021}

\usepackage[utf8]{inputenc} 
\usepackage[T1]{fontenc}    
\usepackage{hyperref}       
\usepackage{url}            
\usepackage{booktabs}       
\usepackage{amsfonts}       
\usepackage{nicefrac}       
\usepackage{microtype}      
\usepackage{xcolor}         

\usepackage{cmap}
\usepackage[english]{babel}
\usepackage{amssymb,amsmath,amsthm}
\usepackage{amsthm}
\usepackage{bbm}
\usepackage{microtype}
\usepackage{xcolor}

\usepackage{enumitem}
\usepackage[olditem,oldenum]{paralist}
\usepackage{tikz}
\usetikzlibrary{shapes.geometric}
\usepackage{mathrsfs}
\makeatletter
\newcommand*\bigcdot{\mathpalette\bigcdot@{.6}}
\newcommand*\bigcdot@[2]{\mathbin{\vcenter{\hbox{\scalebox{#2}{$\m@th#1\bullet$}}}}}
\makeatother

\DeclareRobustCommand{\score}[2]{%
  \pgfmathsetmacro\pgfxa{#1 + 1}%
  \tikzstyle{scorestars}=[star, star points=5, star point ratio=2.25, draw, inner sep=1.3pt, anchor=outer point 3]%
  \begin{tikzpicture}[baseline]
    \foreach \i in {1, ..., #2} {
      \pgfmathparse{\i<=#1 ? "yellow" : "gray"}
      \edef\starcolor{\pgfmathresult}
      \draw (\i*1.75ex, 0) node[name=star\i, scorestars, fill=\starcolor]  {};
   }
  \end{tikzpicture}%
}

\usepackage{algorithm,algorithmic} 
\usepackage{subfigure}
\usepackage{booktabs}
\usepackage{hyperref}

\usepackage{graphicx}
\graphicspath{ {./figure4paper/} }

\newcommand{\option}[1]{{\color[rgb]{.4,0,.8}#1}} 
\newcommand{\authorcomment}[2]{{\color[rgb]{#1}#2}}
\newcommand{\NDY}[1]{\authorcomment{0.0,0.8,0.4}{[NdY: #1]}}

\renewcommand{\option}[1]{#1}
\renewcommand{\NDY}[1]{}

\newcommand{\augmentelargeur}[1]{
\addtolength{\evensidemargin}{-#1}
\addtolength{\oddsidemargin}{-#1}
\addtolength{\textwidth}{#1}
\addtolength{\textwidth}{#1}
}
\augmentelargeur{15mm}

\newtheorem{thm}{Theorem} 
\newtheorem{proposition}[thm]{Proposition}
\newtheorem{definition}[thm]{Definition}
\newtheorem{remark}[thm]{Remark}

\newtheorem*{notation*}{Notation}

\DeclareMathOperator*{\argmax}{arg\,max}

\newcommand{\abs}[1]{\left\lvert#1\right\rvert}
\newcommand{\norm}[1]{\left\lVert#1\right\rVert}

\newcommand{\1}{\mathbbm{1}}  
\def\d{\operatorname{d}\!{}}  

\def\R{{\mathbb{R}}}
\def\C{{\mathbb{C}}}
\newcommand{\deq}{\mathrel{\mathop{:}}=}  
\newcommand{\from}{\colon} 
\def\eps{\varepsilon}
\renewcommand{\epsilon}{\varepsilon}
\renewcommand{\phi}{\varphi}
\let\oldPr\Pr
\renewcommand{\Pr}{\oldPr\nolimits}
\DeclareMathOperator{\E}{\mathbb{E}}  
\DeclareMathOperator{\Var}{Var}
\DeclareMathOperator{\Cov}{Cov}
\DeclareMathOperator{\Id}{Id}

\DeclareMathOperator{\diag}{diag}
\DeclareMathOperator{\rank}{rank}
\newcommand{\transp}[1]{#1^{\!\top}\!} 
\DeclareMathOperator{\Succ}{Succ}

\def\del{\operatorname{\delta}\hspace{-.4ex}{}} 

\newcommand{\lipnorm}[1]{\norm{#1}_\mathrm{Lip}}
\newcommand{\KRnorm}[1]{\norm{#1}_\mathrm{KR}}

\title{Learning One Representation to Optimize All Rewards}

\author{
  Ahmed Touati\thanks{Work done during an internship at Facebook Artificial Intelligence Research Paris.} \\
  Mila, University of Montreal\\
  \texttt{ahmed.touati@umontreal.ca} \\
  \And
  Yann Ollivier \\
  Facebook Artificial Intelligence Research\\
  Paris\\
  \texttt{yol@fb.com} \\
}

\begin{document}

\maketitle

\begin{abstract}
We introduce the \emph{forward-backward} (FB) representation of the dynamics of a reward-free Markov
decision process. It  provides explicit near-optimal policies for any reward
specified a posteriori. During an
unsupervised phase, we
use reward-free interactions with the environment to learn two
representations via off-the-shelf deep learning methods and temporal
difference (TD) learning. In the test phase, a reward representation
is estimated either from reward observations or an explicit reward
description (e.g., a target state). The optimal policy for that
reward is directly obtained from these representations, with no planning.
We assume access to an exploration scheme or replay buffer for the first phase.

The corresponding unsupervised loss is well-principled: if training is perfect, the
policies obtained are provably optimal for any reward function.  With
imperfect training, the sub-optimality is proportional to the
unsupervised approximation error. The FB representation learns
long-range relationships between states and actions, via a predictive
occupancy map, without having to synthesize states as in model-based
approaches.

This is a step towards learning controllable agents in arbitrary
black-box stochastic environments. This approach compares well to
goal-oriented RL algorithms on discrete and continuous mazes, pixel-based
MsPacman, and the FetchReach virtual robot arm. We also illustrate how
the agent can immediately adapt to new tasks
beyond goal-oriented RL.~\footnote{Code: \url{https://github.com/ahmed-touati/controllable_agent}}
\end{abstract}

\section{Introduction}

We consider one kind of unsupervised
reinforcement learning problem: Given a Markov decision process (MDP) but no reward
information, is it possible to learn and store a compact object that, for
any reward function specified later, provides the optimal policy for that
reward, with a minimal amount of additional computation?  In a sense,
such an object would encode in a compact form the solutions of all possible planning problems in
the environment. This is a step towards building agents
that are fully controllable after first exploring their environment in an
unsupervised way.

Goal-oriented RL methods~\cite{andrychowicz2017hindsight, plappert2018multi} compute policies for a series of
rewards specified in advance (such as reaching a set of target states), but
cannot adapt in real time to new rewards, such as weighted combinations
of target states or dense rewards.

Learning a model of the world is another possibility, but it still
requires explicit planning for each new reward; moreover, synthesizing
accurate trajectories of states over long time ranges has proven
difficult~\cite{talvitie2017self, ke2018modeling}.

Instead, we exhibit an object that is both simpler to learn than a model of the world, and contains the information to recover near-optimal policies for any reward provided a posteriori, without a planning phase. 

\cite{borsa2018universal} learn optimal policies for all rewards that are
linear combinations of a finite number of feature functions provided in
advance by the user. 
This limits applications: e.g., goal-oriented tasks
would require one feature per goal state, thus using infinitely many
features in continuous spaces.  We reuse a policy parameterization from
\cite{borsa2018universal}, but introduce a novel representation with better
properties, based on state occupancy prediction instead of expected
featurizations.
We use theoretical advances on successor state learning from
\cite{successorstates}. We obtain the following.

\begin{compactenum}[\hspace{0pt}\textbullet]
    \setlength{\itemsep}{0pt}
\item We prove the existence of a learnable ``summary'' of a reward-free
discrete or continuous MDP, that provides an explicit formula for optimal
policies for any reward specified later. This takes the form of a pair of
representations $F\from S\times A\times Z\to Z$ and $B\from S\times A\to Z$
from state-actions into a representation space $Z\simeq \R^d$, with
policies $\pi_z(s)\deq \argmax_a \transp{F(s,a,z)}z$.
Once a reward is
specified, a value of $z$ is computed from reward values and $B$;
then $\pi_z$ is used.
Rewards may be specified either explicitly as a function, or as target
states, or by samples as in usual RL setups.

\item We provide a well-principled unsupervised loss for $F$ and
$B$. If FB training is perfect, then the policies are provably
optimal for all rewards (Theorem~\ref{thm:main}).  With imperfect
training, sub-optimality is proportional to the FB training error
(Theorems~\ref{thm:approx}--\ref{thm:pointwiseapprox}). \option{In finite spaces, perfect training is
possible with large enough dimension $d$
(Proposition~\ref{prop:existence}).}

Explicitly, $F$ and $B$ are trained so that $\transp{F(s,a,z)}B(s',a')$
approximates the long-term probability to reach $s'$ from $s$ if
following $\pi_z$. 
This is akin to a
model of the environment, without synthesizing state trajectories.

\item We provide a TD-like algorithm to train $F$ and $B$ for this unsupervised loss,
with function approximation, adapted from recent methods for
successor states \cite{successorstates}.  
No sparse rewards are used:
every transition reaches some state $s'$, so every step is
exploited.
As usual with TD, learning seeks
a fixed point but the loss itself is not observable.

\item We prove viability of the method on several environments from mazes
to pixel-based MsPacman and a virtual robotic arm.  For single-state
rewards (learning to reach arbitrary states), we provide quantitative
comparisons with
goal-oriented methods such as HER.  (Our method is not a substitute
for HER: in principle they could be combined, with HER improving replay
buffer management for our method.)
For more general rewards, which cannot be tackled a posteriori by trained goal-oriented models, we provide qualitative examples.

\item We also illustrate qualitatively the sub-optimalities
(long-range behavior is preserved but local blurring of rewards occurs)
and the representations learned.
\end{compactenum}

\section{Problem and Notation}
\label{sec:notation}

Let $\mathcal{M}=(S,A,P,\gamma)$ be
a reward-free Markov decision process
with state space $S$ (discrete or continuous),
action space $A$ (discrete for simplicity, but this is not essential),
transition probabilities $P(s'|s,a)$ from state $s$ to $s'$ given action
$a$,
and
discount factor $0 < \gamma < 1$ \cite{sutton2018reinforcement}.  If
$S$ is finite,
$P(s'|s,a)$ can be viewed as a matrix; in general, for each $(s,a)\in
S\times A$, $P(\d s'|s,a)$ is a probability measure on $s'\in S$.  The
notation $P(\d s'|s,a)$ covers all cases.

Given $(s_0,a_0)\in S\times A$ and a policy
$\pi\from S\to \mathrm{Prob}(A)$, we denote
$\Pr(\cdot|s_0,a_0,\pi)$ and $\E[\cdot|s_0,a_0,\pi]$ the probabilities and expectations under state-action sequences
$(s_t,a_t)_{t \geq 0}$ starting with $(s_0,a_0)$ and following policy $\pi$
in the environment, defined by sampling $s_t\sim P(\d
s_t|s_{t-1},a_{t-1})$ and $a_t\sim \pi(s_t)$.

For any policy $\pi$ and state-action $(s_0,a_0)$, define
the \emph{successor measure} $M^\pi(s_0,a_0,\cdot)$ as the measure over
$S\times A$ representing the expected
discounted time spent in each set $X\subset S\times A$:
\begin{equation}
\label{eq:defM}
M^\pi(s_0,a_0,X)\deq \sum_{t\geq 0} \gamma^t \Pr\left((s_t,a_t)\in X \mid
s_0,\,a_0,\,\pi
\right)
\end{equation}
for each $X\subset S\times A$.  Viewing
$M$ as a measure deals with both discrete and continuous spaces.

Given a reward function $r\from S\times A\to \R$,
the $Q$-function of $\pi$ for $r$ is $Q_r^\pi(s_0,a_0)\deq \sum_{t\geq 0}
\gamma^t \E [r(s_t,a_t)|s_0,a_0,\pi]$.  We assume that rewards are bounded, so
that all $Q$-functions are well-defined. We state the
results for deterministic reward functions, but this is not
essential.
\option{We abuse notation and write greedy policies as $\pi(s)=\argmax_a Q(s,a)$ instead of
$\pi(s)\in \argmax_a Q(s,a)$. Ties may be broken any way.}

We consider the following informal problem: Given a reward-free
MDP $(S,A,P,\gamma)$, can we compute a convenient learnable object $E$ such that,
once a reward function $r\from S\times A\to \R$ is specified, we can
easily (with no planning) compute, from $E$ and $r$, a policy $\pi$ whose
performance
is
close to maximal?

\section{Encoding All Optimal Policies via the Forward-Backward Representation}
\label{sec:thm}

We first present forward-backward (FB) representations of a reward-free MDP as a way to summarize
all optimal policies via explicit formulas. The resulting learning procedure is described in
Section~\ref{sec:algo}.

\paragraph{Core idea.} The main algebraic idea is as follows. Assume, at
first,
that $S$ is finite. For a fixed
policy, the $Q$-function depends lineary on the reward:
namely, $Q^\pi_r(s,a)=\sum_{s',a'}
M^\pi(s,a,s',a')r(s',a')$ where $M^\pi(s,a,s',a')= \sum_{t\geq 0}
\gamma^t \Pr\left((s_t,a_t)=(s',a')|s,a,\pi\right)$.
This rewrites as $Q^\pi_r=M^{\pi}r$ viewing
everything as vectors and matrices indexed by state-actions.

Now let $(\pi_z)_{z\in \R^d}$ be any family of
policies parameterized by $z$. Assume that for each $z$, we can find
$d\times (S\times A)$-matrices $F_z$ and $B$ such that
$M^{\pi_z}=\transp{F_z}B$. Then $Q^{\pi_z}_r=\transp{F_z}Br$.
Specializing to $z_R\deq Br$, the $Q$-function of policy $\pi_{z_R}$ on
reward $r$ is $Q^{\pi_{z_R}}_r=\transp{F_{z_R}}\,z_R$. So far $\pi_z$ was unspecified; but
if we define $\pi_z(s)\deq \argmax_a
(\transp{F_z}\,z)_{sa}$ at
each state $s$, then by definition, $\pi_{z_R}$ is the greedy policy with respect to
$\transp{F_{z_R}}\,z_R$. At the same time, $\transp{F_{z_R}}\,z_R$ is the $Q$-function of
$\pi_{z_R}$ for reward $r$: thus,
$\pi_{z_R}$ is the greedy policy of its own $Q$-function, and is
therefore optimal for reward $r$.

Thus, if we manage to find $F$, $B$, and $\pi_z$ such that
$\pi_z=\argmax \transp{F_z}\,z$ and $\transp{F_z}B=M^{\pi_z}$ for all
$z\in \R^d$, then we obtain the optimal policy for any reward $r$, just by
computing $Br$ and applying policy $\pi_{Br}$.

This criterion on $(F,B,\pi_z)$ is
entirely unsupervised. Since $F$ and $B$ depend on $\pi_z$ but $\pi_z$ is defined via $F$, this
is a fixed point equation. 
An exact solution exists for $d$ 
large enough (Appendix,
Prop.~\ref{prop:existence}), while a smaller $d$ provides lower-rank
approximations $M^{\pi_z}\approx \transp{F_z}B$.
In Section~\ref{sec:algo} we present a well-grounded algorithm to learn such $F$, $B$, and $\pi_z$.

In short, we learn two representations $F$ and $B$ such that
$\transp{F(s_0,a_0,z)}B(s',a')$ is approximately the long-term
probability $M^{\pi_z}(s_0,a_0,s',a')$ to reach $(s',a')$ if starting at $(s_0,a_0)$ and following
policy $\pi_z$. 
Then all optimal policies can be computed from $F$ and $B$.
We think of $F$ as a representation of the future of a state, and $B$ as the ways to
reach a state (Appendix~\ref{sec:succpred}): if $\transp{F}B$ is large,
then the second state is reachable from the first.
This is akin to a
model of the environment, without synthesizing state trajectories.



\paragraph{General statement.} In continuous spaces with function approximation, $F_z$ and $B$ become functions $S\times A\to \R^d$
instead of matrices; since $F_z$ depends on $z$, $F$ itself is a function
$S\times A\times \R^d\to \R^d$. The sums over states will be replaced with expectations
under the data distribution $\rho$.


\begin{definition}[Forward-backward representation]
\label{def:fb}
Let $Z=\R^d$ be a representation space, and let $\rho$ be a measure on
$S\times A$. 
A pair of
functions $F\from S\times A\times Z\to Z$ and $B\from S\times A\to Z$,
together with a parametric family of policies $(\pi_z)_{z\in Z}$, is
called a \emph{forward-backward representation} of the MDP with respect
to $\rho$, if
the following conditions hold for any $z\in Z$ and $(s,a),(s_0,a_0)\in
S\times A$:
\begin{align}
\pi_z(s)=\argmax_a \transp{F(s,a,z)}z,\,
\qquad
M^{\pi_z}(s_0,a_0,\d s,\d a)=
\transp{F(s_0,a_0,z)}B(s,a)\rho(\d s,\d a)
\label{eq:FBdef}
\end{align}
where $M^\pi$ is the successor measure defined in \eqref{eq:defM}, and
the last equality is between measures.
\end{definition}

\begin{thm}[FB representations encode all optimal policies]
\label{thm:main}
Let $(F,B,(\pi_z))$ be a forward-backward representation of a reward-free MDP
with respect to some measure $\rho$.

Then, for any bounded reward function $r\from S\times A\to \R$, the following
holds. Set
\begin{equation}
\label{eq:zR}
z_R\deq \int_{s,a} r(s,a)B(s,a) \,\rho(\d s,\d a).
\end{equation}
assuming the integral exists.
Then $\pi_{z_R}$ is an optimal policy for reward $r$ in the MDP.
Moreover, the optimal $Q$-function $Q^\star$ for reward $r$ is
$Q^\star (s,a)=\transp{F(s,a,z_R)}z_R$.
\end{thm}

For instance, for a single reward located at state-action $(s,a)$, the
optimal policy is $\pi_{z_R}$ with $z_R=B(s,a)$. (In that case the factor $\rho(\d
s,\d a)$ does not matter because scaling the reward does not change the
optimal policy.)

We present in Section~\ref{sec:algo} an algorithm to learn FB
representations. The measure $\rho$ will be the distribution of state-actions
visited in a training set or under an exploration policy: then $z_R=\E_{(s,a)\sim
\rho} [r(s,a)B(s,a)]$ can be obtained
by sampling from visited states.

In finite spaces, exact FB representations exist,  provided the dimension $d$ is larger than
$\#S\times \#A$ (Appendix,
Prop.~\ref{prop:existence}).
In infinite spaces, arbitrarily good approximations can be obtained by
increasing $d$, corresponding to a rank-$d$ approximation of the
cumulated transition probabilities $M^\pi$.  Importantly,
the optimality guarantee extends to approximate $F$
and $B$, with optimality gap
proportional to $\transp{F}B-M^{\pi_z}/\rho$
(Appendix,
Theorems~\ref{thm:approx}--\ref{thm:pointwiseapprox} \option{with various
norms on $\transp{F}B-M^\pi/\rho$}).
For instance,
if, for some reward $r$, the error
$\abs{\transp{F(s_0,a_0,z_R)}B(s,a)-M^{\pi_{z_R}}(s_0,a_0,\d s,\d a)/\rho(\d
s,\d a)}$ is at
most $\eps$ on average over $(s,a)\sim \rho$ for every $(s_0,a_0)$, then
$\pi_{z_R}$ is $3\eps\norm{r}_\infty/(1-\gamma)$-optimal for $r$.

These results justify using some norm over
$\abs{\transp{F}B-M^{\pi_z}/\rho}$,
averaged over $z\in \R^d$, as a
training loss for unsupervised reinforcement learning. (Below, we average
over $z\in \R^d$ from a 
fixed rescaled Gaussian. If prior information is
available on the rewards $r$, the corresponding distribution of $z_R$ may be used
instead.)


If $B$ is fixed in advance and only $F$ is
learned, the method has similar properties to successor features based
on $B$ (Appendix~\ref{sec:succpred}). But one may set a large $d$ and let
$B$ be learned: arguably, by
Theorem~\ref{thm:main}, the resulting features
``linearize'' optimal policies as much as possible.
The features learned in $F$ and
$B$
may have broader interest.


\section{Learning and Using Forward-Backward Representations}
\label{sec:algo}

Our algorithm starts with
an \emph{unsupervised learning phase}, where we
learn the representations $F$ and $B$ in a reward-free way, by
observing state transitions in the environment, generated from any
exploration scheme. Then, in a
\emph{reward
estimation phase}, we estimate a policy parameter $z_R=\E [r(s,a)B(s,a)]$
from some reward
observations, or directly set $z_R$ if the reward is known (e.g., set
$z_R=B(s,a)$ to reach
a known target $(s,a)$). In the \emph{exploitation phase}, we
directly use the policy $\pi_{z_R}(s)=\argmax_a \transp{F(s,a,z_R)}z_R$.

\paragraph{The unsupervised learning phase.}
No rewards are used in this phase, and no
family of tasks has to be specified manually.
$F$ and $B$ are trained off-policy from observed transitions in the
environment. The first condition of FB representations, $\pi_z(s)=\argmax_a
\transp{F(s,a,z)}z$, is just taken as the definition of $\pi_z$ given $F$. In turn, $F$ and
$B$ are trained so that the second condition \eqref{eq:FBdef},
$\transp{F(\cdot,z)}B=M^{\pi_z}/\rho$,
holds for every $z$. Here $\rho$ is the (unknown) distribution of
state-actions in the training data.
Training  is based on the Bellman equation
for the successor measure $M^\pi$,
\begin{equation}
M^\pi(s_0,a_0,\{(s',a')\})=
\1_{s_0=s',\,a_0=a'}\,+\gamma \E_{s_1\sim P(\d s_1|s_0,a_0)}
M^\pi(s_1,\pi(s_1),\{(s',a')\}).
\end{equation}

We leverage a well-principled
algorithm from \cite{successorstates} in the single-policy setting: it
learns the successor measure
of a policy $\pi$ without using
the sparse reward $\1_{s_0=s',\,a_0=a'}$ (which would vanish in
continuous spaces). Other successor measure 
algorithms could be used, such as
C-learning \cite{Clearning}.

The algorithm from \cite{successorstates} uses a parametric model $m_\theta^\pi(s_0,a_0,s',a')$ to
represent $M^\pi(s_0,a_0,\d s', \d a')\approx
m_\theta^\pi(s_0,a_0,s',a')\rho(\d s',\d a')$. It is not necessary to
know $\rho$, only to sample states from it.
Given an observed transition
$(s_0,a_0,s_1)$ from the training set, generate an action $a_1\sim
\pi(a_1|s_1)$, and
sample another
state-action $(s',a')$ from the training set, independently from
$(s_0,a_0,s_1)$. Then update the parameter $\theta$ by
$\theta\gets \theta+\eta \del \theta$ with learning rate $\eta$ and
\begin{equation}
\label{eq:mtd}
\del \theta\deq \partial_\theta m^\pi_\theta(s_0,a_0,s_0,a_0) + 
\partial_\theta m^\pi_\theta(s_0,a_0,s',a')
\,\times \left(
\gamma \,m^\pi_\theta(s_1,a_1,s',a')
-m^\pi_\theta(s_0,a_0,s',a')\right)
\end{equation}
This computes the density $m^\pi$ of $M^\pi$ with respect to the
distribution $\rho$ of state-actions in the training set. Namely,
the true successor state density $m^\pi=M^\pi/\rho$ is a fixed point of
\eqref{eq:mtd}
in expectation \cite[Theorem 6]{successorstates} \option{(and is the only fixed point in the tabular or
overparameterized case)}.
Variants exist, such as using a target network for
$m^\pi_\theta(s_1,a_1,s',a')$ on the right-hand side, as in DQN. 

Thus, we first
choose a parametric model $F_\theta,
B_\theta$ for the representations $F$ and $B$, and set
$m_\theta^{\pi_z}(s_0,a_0,s',a')\deq
\transp{F_\theta(s_0,a_0,z)}B_\theta(s',a')$. Then we iterate the update
\eqref{eq:mtd} over many state-actions and values of $z$.
This results in Algorithm~\ref{FB algo}.
At each step, a value of $z$ is picked at random, together with a batch
of transitions
$(s_0,a_0,s_1)$ and a batch of state-actions 
$(s',a')$ from the training set, with $(s',a')$ independent from $z$ and
$(s_0,a_0,s_1)$.

For sampling $z$, we use a fixed distribution (rescaled Gaussians, see
Appendix~\ref{sec:setup}). Any number of values of $z$ may be sampled: this does not
use up training samples. We use a target network with soft updates
(Polyak averaging) as
in DDPG. For training we also replace the greedy policy
$\pi_z=\argmax_a \transp{F(s,a,z)}z$ with a regularized
version $\pi_z= \mathrm{softmax}(\transp{F(s,a,z)}z/\tau)$ with fixed
temperature $\tau$ (Appendix~\ref{sec:setup}). Since there is
unidentifiability between $F$ and $B$ (Appendix, Remark~\ref{rem:normB}),
we normalize $B$ via an auxiliary loss in Algorithm~\ref{FB algo}.

For exploration in this phase, we use the policies being learned: the
exploration policy chooses a random value of $z$ 
from some distribution (e.g., Gaussian), and
follows $\pi_z$ for some time (Appendix, 
Algorithm~\ref{FB algo}). However, the algorithm can also work from
an existing dataset of off-policy transitions.

\paragraph{The reward estimation phase.} Once rewards are available, we estimate a reward
representation (policy parameter) $z_R$ by weighing the representation $B$
by the reward:
\begin{equation}
\label{eq:zR_outline}
z_R\deq\E [r(s,a) B(s,a)]
\end{equation}
where the expectation must be computed over the same distribution $\rho$ of state-actions
$(s,a)$
used to learn $F$ and $B$ (see Appendix~\ref{sec:rhotest} for using a different
distribution).
Thus, if the reward is black-box as in standard RL
algorithms, then the exploration policy has to be run again for some
time, and $z_R$ is obtained by averaging
$r(s,a)B(s,a)$ over the states visited.
An approximate value for $z_R$ still provides an approximately
optimal policy (Appendix, Prop.~\ref{prop:approxr} and
Thm.~\ref{thm:approxzR}).

If the reward is known explicitly, this phase is unnecessary. For
instance, if the reward is to reach a target state-action $(s_0,a_0)$ while avoiding some
forbidden state-actions $(s_1,a_1),...,(s_k,a_k)$, one may directly set
\begin{equation}
z_R=B(s_0,a_0)-\lambda \sum B(s_i,a_i)
\end{equation}
where the constant $\lambda$ sets the negative reward for 
forbidden states
and adjusts for the unknown $\rho(\d s_i,\d a_i)$ factors in
\eqref{eq:zR}.
This can be used for goal-oriented RL.

If the reward is known algebraically as a function $r(s,a)$, then $z_R$ may be computed by averaging the function $r(s,a)B(s,a)$ over a
replay buffer from the unsupervised training phase. We may also use
a reward model $\hat r(s,a)$ of $r(s,a)$ trained on some reward
observations from any source.


\paragraph{The exploitation phase.} Once the reward representation $z_R$
has been
estimated, the $Q$-function is estimated as
\begin{equation}
Q(s,a)=
\transp{F(s,a,z_R)}\,z_R.
\end{equation}
The corresponding policy $\pi_{z_R}(s)=\argmax_a
Q(s,a)$ is used for exploitation.

Fine-tuning was not needed in our experiments, but it is possible to fine-tune the
$Q$-function using actual rewards,
by setting $Q(s,a)=
\transp{F(s,a,z_R)}z_R + q_\theta(s,a)$ where the fine-tuning model
$q_\theta$ is initialized to $0$
and learned via any standard $Q$-learning method.

\paragraph{Incorporating prior information on rewards in $B$.} Trying to plan in
advance for all possible rewards in an arbitrary environment may be too
generic and problem-agnostic, and become difficult in large environments,
requiring long exploration and a large $d$ to accommodate all rewards.
%
In practice, we are often interested in rewards depending, not on the full
state, but only on a part or some features of the state (e.g., a few
components of the state, such as the position of an agent, or its
neighbordhood, rather than the full environment).

If this is known in
advance, the representation $B$ can be trained on that part of the state
only, with the same theoretical guarantees (Appendix,
Theorem~\ref{thm:main2}). $F$ still needs to use the full state
as input.  This way, the FB model of the transition probabilities
\eqref{eq:defM} only has to learn the future probabilities of the part of
interest in the future states, based on the full initial state
$(s_0,a_0)$.
%
Explicitly,
if $\phi\from S\times A
\to G$ is a feature map to some features $g=\phi(s,a)$, and if we know
that the reward will be a function $R(g)$, then Theorem~\ref{thm:main}
still holds
with $B(g)$ everywhere instead of $B(s,a)$, and with the successor
measure $M^\pi(s_0,a_0,\d g)$ instead of $M^\pi(s_0,a_0,\d s',\d a')$ (Appendix,
Theorem~\ref{thm:main2}). \option{Learning 
is done
by replacing $\partial_\theta m_\theta^\pi(s_0,a_0,s_0,a_0)$ with
$\partial_\theta m_\theta^\pi(s_0,a_0,\phi(s_0,a_0))$ in the first term in
\eqref{eq:mtd} \cite{successorstates}.}
Rewards can be arbitrary functions of $g$, so this is more general than
\cite{borsa2018universal} which only considers rewards linear in $g$.
For instance, in MsPacman below, we let $g$ be the 2D position $(x,y)$ of the agent,
so we can optimize any reward function that depends on this position.



\paragraph{Limitations.} 
First, this method does not solve exploration: it assumes access
to a good exploration strategy. (Here we used the policies $\pi_z$ with
random values of $z$, corresponding to random rewards.)

Next, this task-agnostic approach is relevant if the
reward is not known in advance, but may not bring the best performance on
a particular reward. 
Mitigation strategies include: increasing $d$; using prior information on
rewards by including relevant variables into $B$, as
discussed above; and fine-tuning the $Q$-function at test time based on
the initial $\transp{F}B$ estimate.

As reward functions are represented by a
$d$-dimensional vector $z_R=\E[r.B]$, some information about the reward
is necessarily lost.
Any reward uncorrelated to $B$ is treated as $0$. 
The dimension $d$ controls how many types of rewards can be optimized
well.
A priori, a large $d$ may be required.
Still, in the experiments, $d\approx 100$ manages
navigation in a pixel-based environment with a huge state space.
Appendix~\ref{sec:dimension} argues theoretically that
$d=2n$ is enough for navigation on an $n$-dimensional grid. 
The algorithm is linear
in $d$, so $d$ can be taken as large as the neural network models can
handle.



We expect this method to have an implicit bias for long-range
behavior (spatially smooth rewards), while
local details of the reward function may be blurred. Indeed,
$\transp{F}B$ is optimized to approximate the successor measure
$M^\pi=\sum_t \gamma^t P_\pi^t$ with $P^t_\pi$ the $t$-step transition kernel
for each policy $\pi$. The rank-$d$ approximation will favor large
eigenvectors of $P_\pi$, i.e., small eigenvectors of the Markov chain
Laplacian $\Id-\gamma P_\pi$.  These loosely correspond to long-range
(low-frequency) behavior \cite{mahadevan2007proto}:
presumably, $F$ and $B$ will learn spatially smooth rewards first.
Indeed, experimentally, a small $d$ leads to spatial blurring of rewards
and $Q$-functions (Fig.~\ref{figure: grid heatmap }).
Arguably, without any prior information this is
a reasonable prior. \cite{stachenfeld2017hippocampus} have argued for the
cognitive relevance of
low-dimensional approximations of successor representations.

Variance is a potential issue in larger environments, although this did
not arise in our experiments. Learning $M^\pi$ requires sampling a
state-action $(s_0,a_0)$ and an independent state-action $(s',a')$. In
large spaces, most state-action pairs will be unrelated. A possible
mitigation is to combine FB with strategies such as Hindsight Experience Replay
\cite{andrychowicz2017hindsight}
to select goals related to the current state-action.
The following may help a lot:
the update of $F$ and $B$ decouples as an expectation over
$(s_0,a_0)$, times an expectation over $(s',a')$. Thus, by estimating
these expectations by a moving average over a dataset, it is easy to have
many pairs $(s_0,a_0)$ interact with many $(s',a')$. The cost is handling
full $d\times d$ matrices.  This will be explored in future work.

\section{Experiments}
\label{sec:exp}

We first consider the task of reaching arbitrary goal states. For this,
we can make quantitative comparisons to existing goal-oriented baselines.
Next, we illustrate qualitatively some tasks that cannot be tackled a
posteriori by goal-oriented methods, such as introducing forbidden
states. Finally, we illustrate some of the representations learned.

\subsection{Environments and Experimental Setup}
We run our experiments on a selection of environments that are diverse in term of state space dimensionality, stochasticity and dynamics.
\begin{compactenum}[\hspace{0pt}\textbullet]
    \setlength{\itemsep}{0pt}
    \item Discrete Maze is the classical gridworld with four rooms. States are represented by one-hot unit vectors.
    \item Continuous Maze is a two dimensional environment with impassable walls. States are represented by their Cartesian coordinates $(x,y) \in [0,1]^2$. The execution of one of the actions moves
the agent in the desired direction, but with normal random noise added to the position of the agent.
\item FetchReach is a variant of the simulated robotic arm environment
from~\cite{plappert2018multi} using discrete actions instead of continuous actions.  States are 10-dimensional vectors consisting of positions and velocities of robot joints.
\item Ms.\ Pacman is a variant of the Atari 2600 game Ms.\ Pacman, where an episode ends when the agent is captured by a monster~\cite{rauber2018hindsight}. 
States are obtained by processing the raw visual input directly from the screen. Frames
are preprocessed by cropping, conversion to grayscale and downsampling to $84\times84$ pixels. A state $s_t$ is the concatenation of $(x_{t-12}, x_{t-8}, x_{t-4}, x_t)$ frames, i.e. an $84\times84\times4$ tensor.  
An action
repeat of 12 is used. As Ms.\ Pacman is not originally a
multi-goal domain, we define the goals as the 148 reachable coordinates
$(x,y)$ on the screen; these can be reached only by learning to avoid monsters.
\end{compactenum}
For all environments, we run algorithms for 800 epochs, with three different random seeds. Each epoch
consists of 25 cycles where we interleave between gathering some amount
of transitions, to add to the replay buffer, and performing 40 steps of
stochastic gradient descent on the model parameters. To collect
transitions, we generate episodes using some behavior policy. For both
mazes, we use a uniform policy while for FetchReach and Ms.\ Pacman, we
use an $\epsilon$-greedy policy with respect to the current approximation $F(s, a, z)^\top z$ for a sampled $z$. At evaluation time, $\eps$-greedy policies are also used, with a smaller
$\eps$. More details are given in Appendix~\ref{sec:setup}.

\subsection{Goal-Oriented Setting: Quantitative Comparisons}

We investigate the 
FB representation
over goal-reaching tasks and compare it to goal-oriented baselines:
DQN\footnote{Here DQN is short for goal-oriented DQN, $Q(s, a,
g)$.}, and DQN with HER when needed. We define sparse reward
functions. For Discrete Maze, the reward function is equal to one when
the agent's state is equal exactly to the goal state. For Discrete Maze,
we measured the quality of the obtained policy to be the ratio between
the true expected discounted reward of the policy for its goal  and the
true
optimal value function, on average over all states.
For the other environments, the reward function is equal to one when the
distance of the agent's position and the goal position is below some
threshold, and zero otherwise. We assess policies by computing
the average success rate, i.e the average number of times the agent
successfully reaches its goal. 

%

\begin{minipage}{\linewidth}
\centering
\begin{minipage}{0.45\linewidth}
\begin{figure}[H]
    \includegraphics[width=0.48\textwidth]{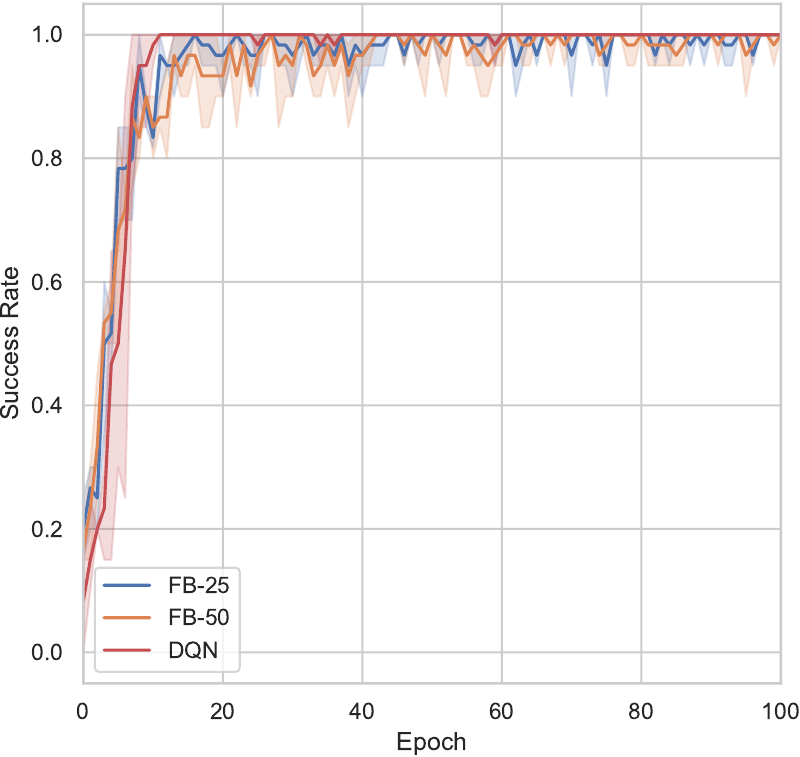}
     \includegraphics[width=0.48\textwidth]{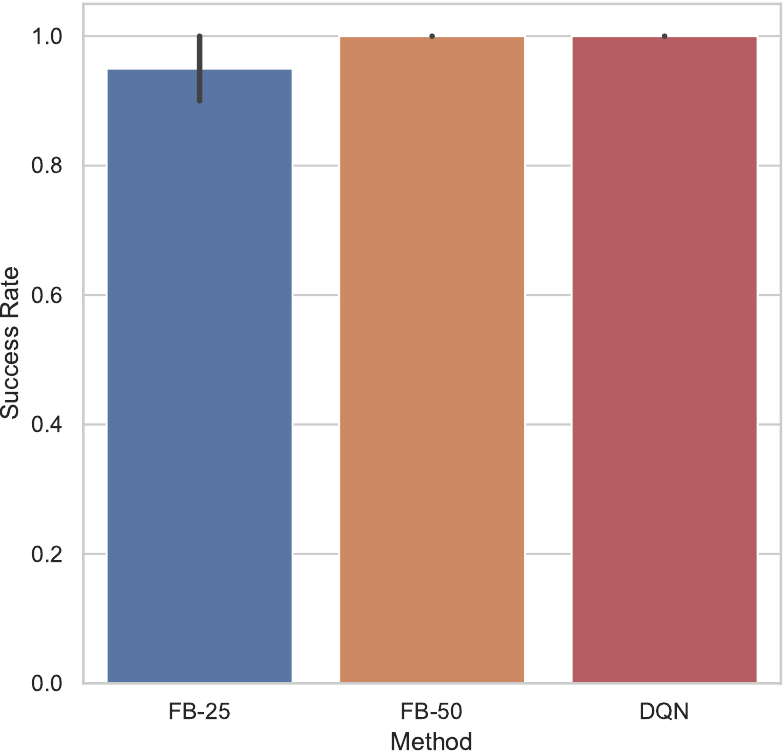} 
  
    \caption{\label{figure: perf in fetchreach} \small Comparative performance of FB for different dimensions and DQN in FetchReach. \textbf{Left}: success rate averaged over 20 randomly selected goals as function of the first 100 training epochs. \textbf{Right}: success rate averaged over 20 random goals after 800 training epochs.}    
\end{figure}
\end{minipage}
\hspace{0.05\linewidth}
\begin{minipage}{0.45\linewidth}
\centering
\begin{figure}[H]
    \includegraphics[width=0.48\textwidth]{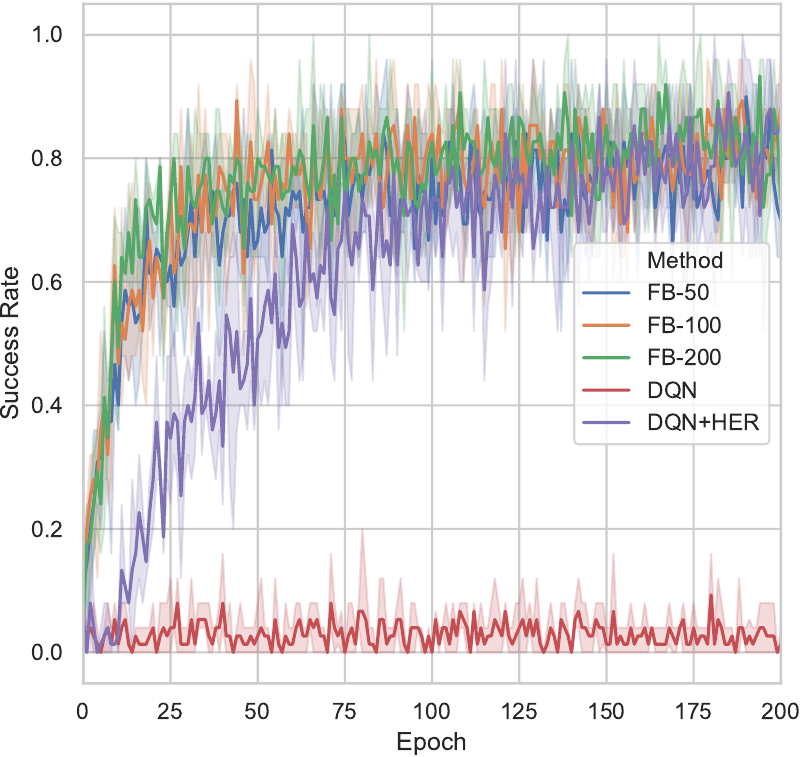} 
     \includegraphics[width=0.48\textwidth]{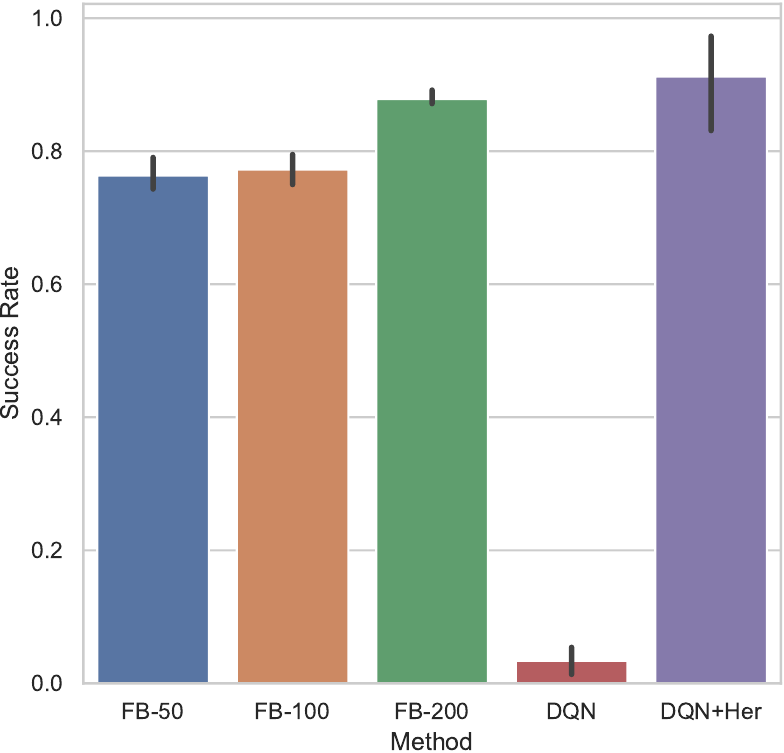} 
  
    \caption{ \label{figure: perf in pacman} \small Comparative performance of FB for different dimensions and DQN in Ms.\ Pacman. \textbf{Left}: success rate averaged over 20 randomly selected goals as function of the first 200 training epochs. \textbf{Right}: success rate averaged over the goal space after 800 training epochs.}
\end{figure}
\end{minipage}
\end{minipage}

%
%
%

%


Figs.~\ref{figure: perf in fetchreach} and~\ref{figure: perf in
pacman} show the comparative performance of FB for different dimensions
$d$, and DQN  respectively in  FetchReach and Ms.\
Pacman (similar results in Discrete and Continuous Mazes are provided in
Appendix~\ref{sec:setup}). In Ms.\ Pacman, DQN totally fails to learn and
we had to add HER to make it work. The performance of FB consistently increases with the
dimension $d$ and the best dimension matches the performance
of the goal-oriented baseline. 

In Discrete Maze, we observe a drop of performance for $d=25$
(Appendix~\ref{sec:setup}, Fig.~\ref{fig:discretemazeresults}):
this is due to the spatial smoothing induced by the small rank
approximation and the reward being nonzero only if the agent
is exactly at the goal. This spatial blurring is clear on
heatmaps for $d=25$ vs $d=75$
(Fig.~\ref{figure: grid heatmap }). With $d=25$ the agent often stops
right next to its goal.

To evaluate the sample efficiency of FB, after each epoch, we
evaluate the agent on 20 randomly selected goals. Learning curves are
reported in Figs.~\ref{figure: perf in fetchreach} and~\ref{figure:
perf in pacman} (left). In all environments, we observe
no loss in sample efficiency compared to the goal-oriented baseline. In
Ms.\ Pacman, FB even learns faster than DQN+HER.

\subsection{More Complex Rewards: Qualitative Results}

We now investigate FB's ability to generalize to new tasks that cannot be solved
by an already trained goal-oriented model: reaching a goal with forbidden
states imposed a posteriori,
reaching the nearest of two goals, and choosing between a small, close
reward and a large, distant one.

First, for the task of reaching a target position $g_0$
\score{1}{1} while avoiding some forbidden positions $g_1, \ldots g_k$
\tikz\draw[red,fill=red] (0,0) circle (.7ex);, we  set $z_R =
B(g_1) - \lambda \sum_{i=1}^k B(g_i)$ and run the corresponding
$\eps$-greedy policy defined by $F(s, a, z_R)^\top z_R$.
Fig.~\ref{figure: neg reward} shows the resulting trajectories, which
succeed at solving the task for the different domains. In Ms.\ Pacman,
the path is suboptimal (though successful) due to the sudden
appearance of a monster along the optimal path. (We  only plot the
initial frame; see the full series of frames along the trajectory in
Appendix~\ref{sec:setup}, Fig.~\ref{fig:fullpacmanframes}.) Fig.~\ref{figure: mix reward in continuous maze} (left)
provides a contour plot of $\max_{a \in A}F(s, a, z_R)^\top z_R$ for the
continuous maze and shows the landscape shape around the forbidden regions.

Next, we consider the task of reaching the closest target among two
equally rewarding positions $g_0$ and $g_1$, by setting $z_R = B(g_0) +
B(g_1)$. The optimal $Q$-function is
\emph{not} a linear combination of the $Q$-functions for $g_0$ and $g_1$.
Fig.~\ref{figure: equal
reward} shows successful trajectories generated by the policy
$\pi_{z_R}$.
On the contour
plot of $\max_{a \in A}F(s, a, z_R)^\top z_R$ in Fig.~\ref{figure: mix
reward in continuous maze} (right), the two rewarding positions appear as
basins of attraction.  Similar results for a third task are shown in
Appendix~\ref{sec:setup}: introducing a ``distracting'' small reward next to the initial
position of the agent, with a larger reward further away.


\begin{minipage}{\linewidth}
\centering
\begin{minipage}{0.45\linewidth}
\begin{figure}[H]
	\centering
    \includegraphics[width=0.45\textwidth]{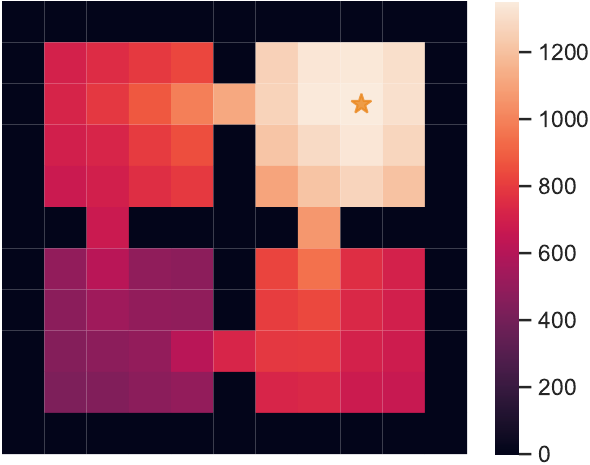} 
     \includegraphics[width=0.45\textwidth]{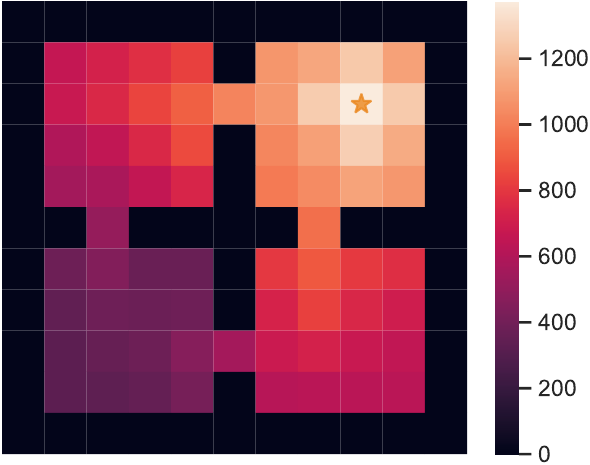} 
     
    \caption{ \label{figure: grid heatmap } \small Heatmap of $\max_{a}
    F(s, a, z_R)^\top z_R$ for $z_R = B(\text{\score{1}{1}})$
    \textbf{Left}: $d = 25$. \textbf{Right}: $d = 75$.}
\end{figure}

\end{minipage}
\hspace{0.05\linewidth}
\begin{minipage}{0.45\linewidth}
\begin{figure}[H]  
\centering   
     \includegraphics[width=0.45\textwidth, trim = {1cm 2cm 1cm 3cm}]{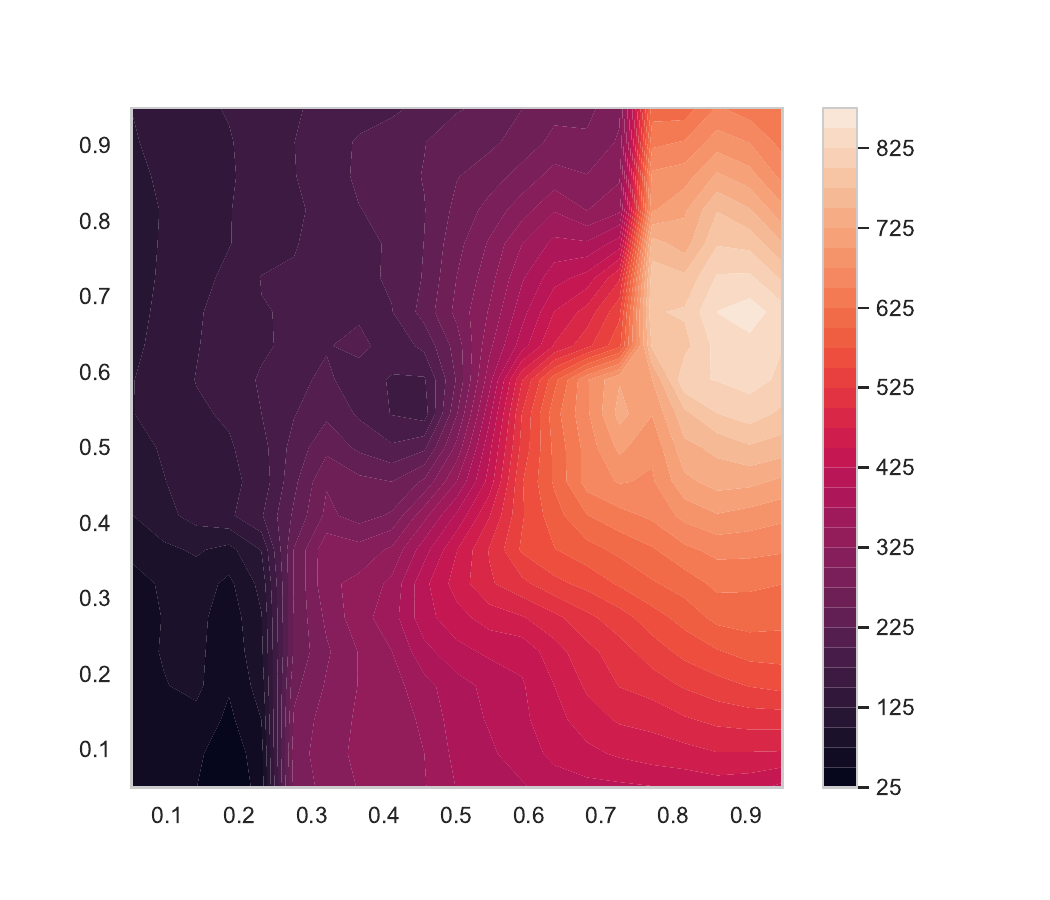} 
    \includegraphics[width=0.45\textwidth, trim = {1cm 2cm 1cm 3cm}]{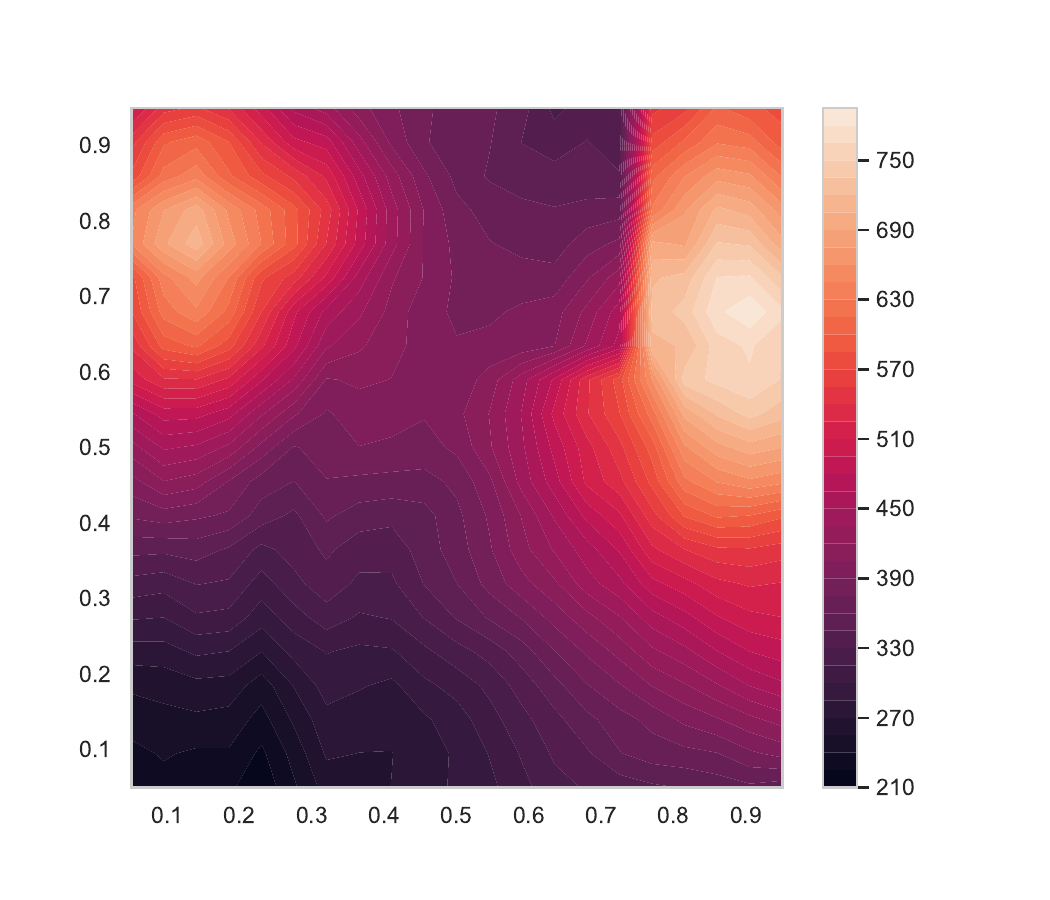}
    \caption{\label{figure: mix reward in continuous maze} \small Contour
    plot of  $\max_{a \in A}F(s, a, z_R)^\top z_R$ in Continuous Maze.
    \textbf{Left}: for the task of reaching a target while avoiding a
    forbidden region, \textbf{Right}: for two equally rewarding targets.}
\end{figure}
\end{minipage}
\end{minipage}


\begin{minipage}{\linewidth}
\centering
\begin{minipage}{0.45\linewidth}
\begin{figure}[H] 
\centering
    \includegraphics[width=0.30\textwidth,  trim = {1cm 1cm 1cm 4cm}]{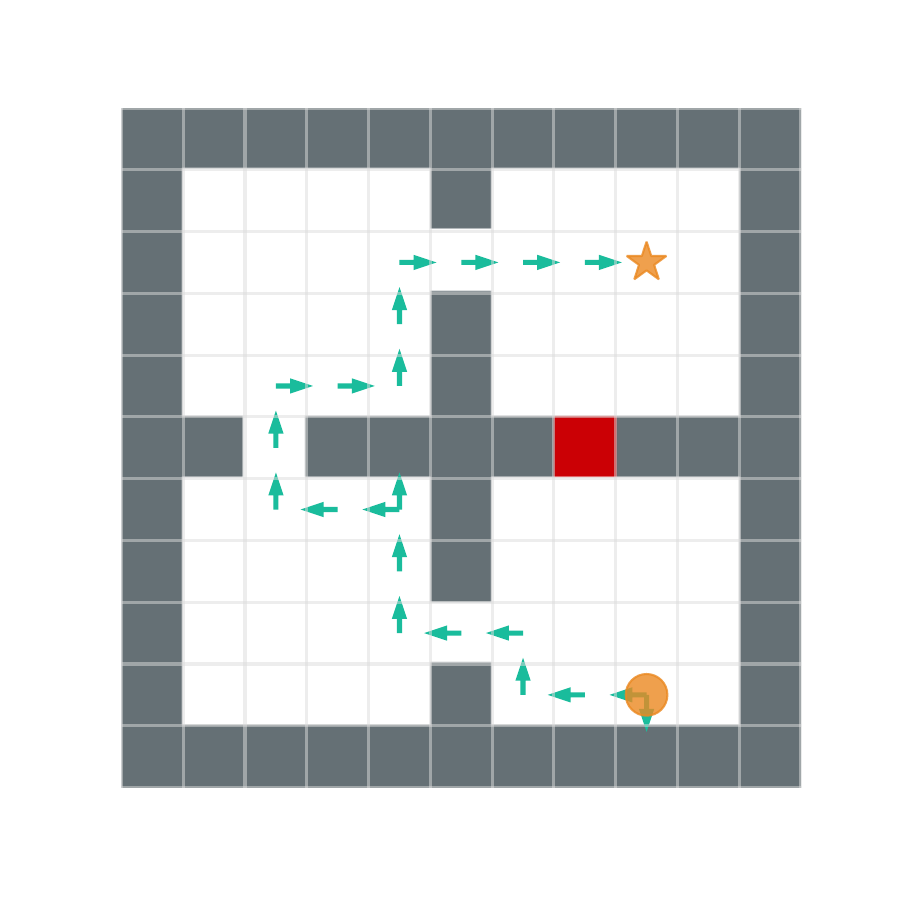}
     \includegraphics[width=0.30\textwidth, trim = {1cm 1cm 1cm 4cm}]{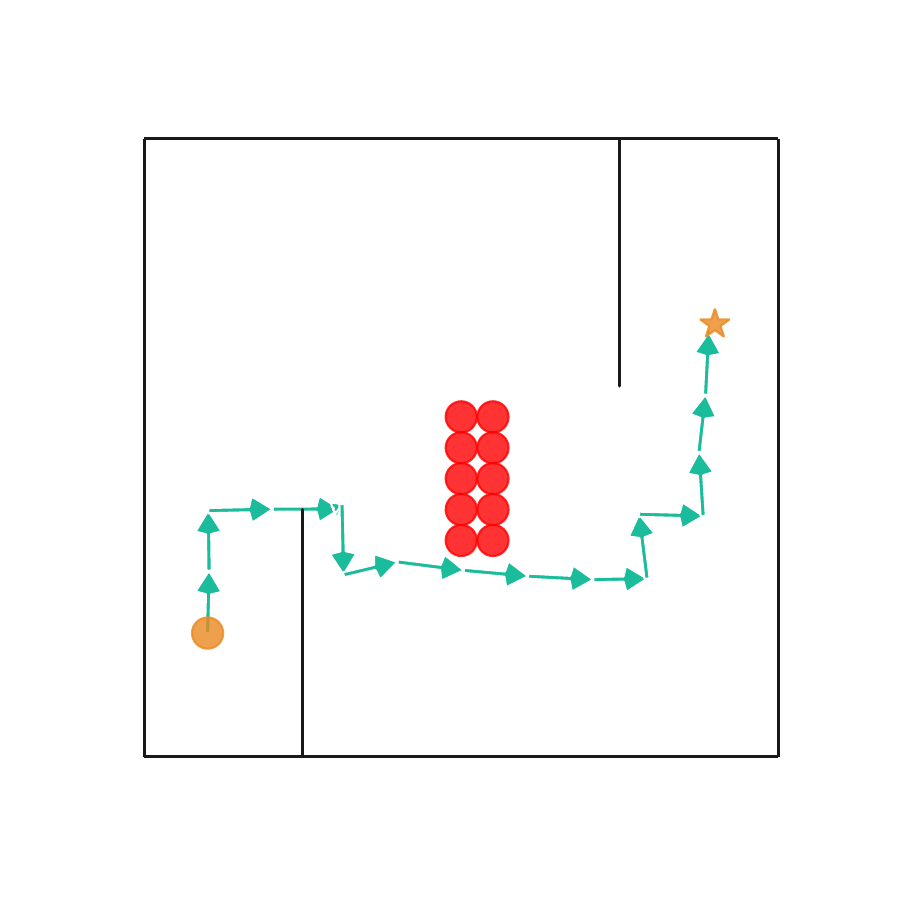}
     \includegraphics[width=0.30\textwidth,  trim = {1cm 1cm 1cm 4cm}]{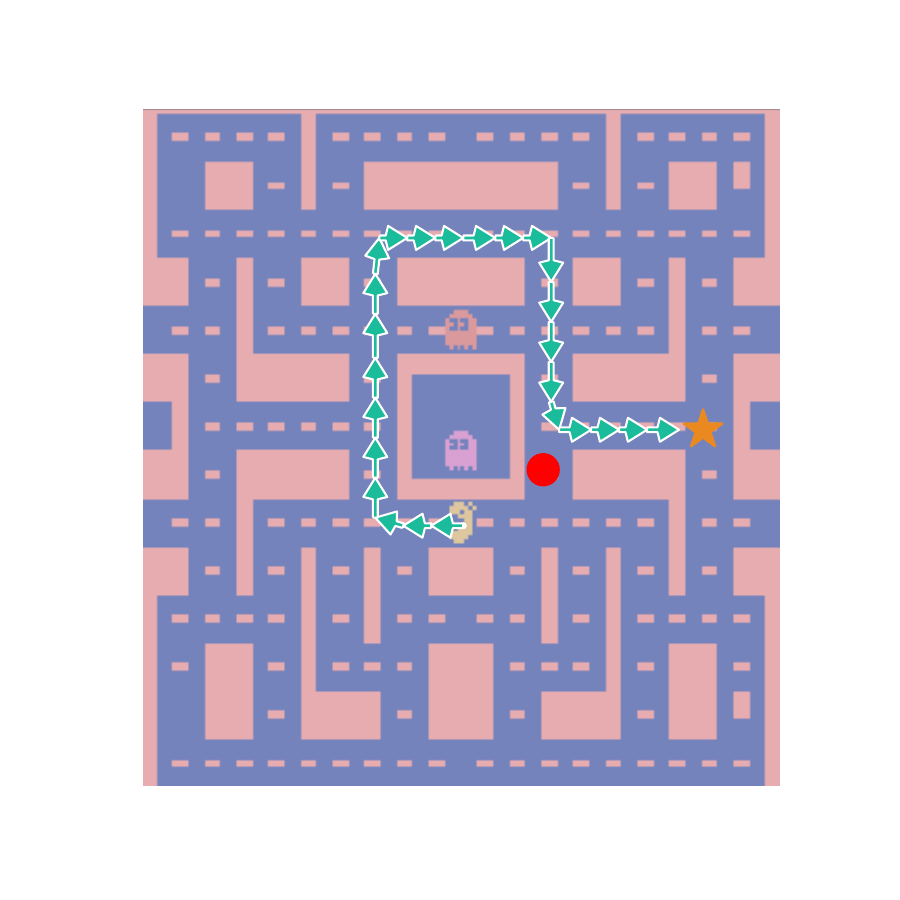} 
    \caption{\label{figure: neg reward} \small Trajectories generated by
    the
    $\transp{F}B$ policies for the task of reaching a target position
    (star shape \score{1}{1} while avoiding forbidden positions
    (red shape \protect\tikz\protect\draw[red,fill=red] (0,0) circle (.7ex);)
    }
\end{figure}
\end{minipage}
\hspace{0.05\linewidth}
\begin{minipage}{0.45\linewidth}
\begin{figure}[H] 
\centering
    \includegraphics[width=0.30\textwidth,  trim = {1cm 1cm 1cm 4cm}]{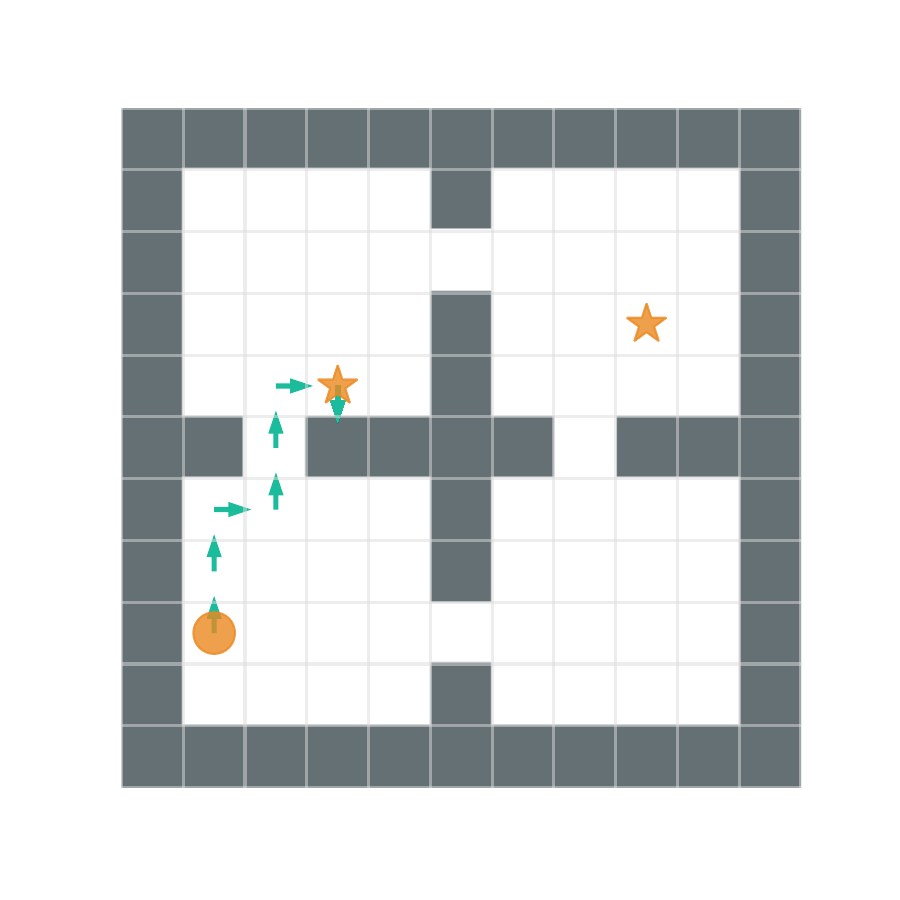}
     \includegraphics[width=0.30\textwidth, trim = {1cm 1cm 1cm 4cm}]{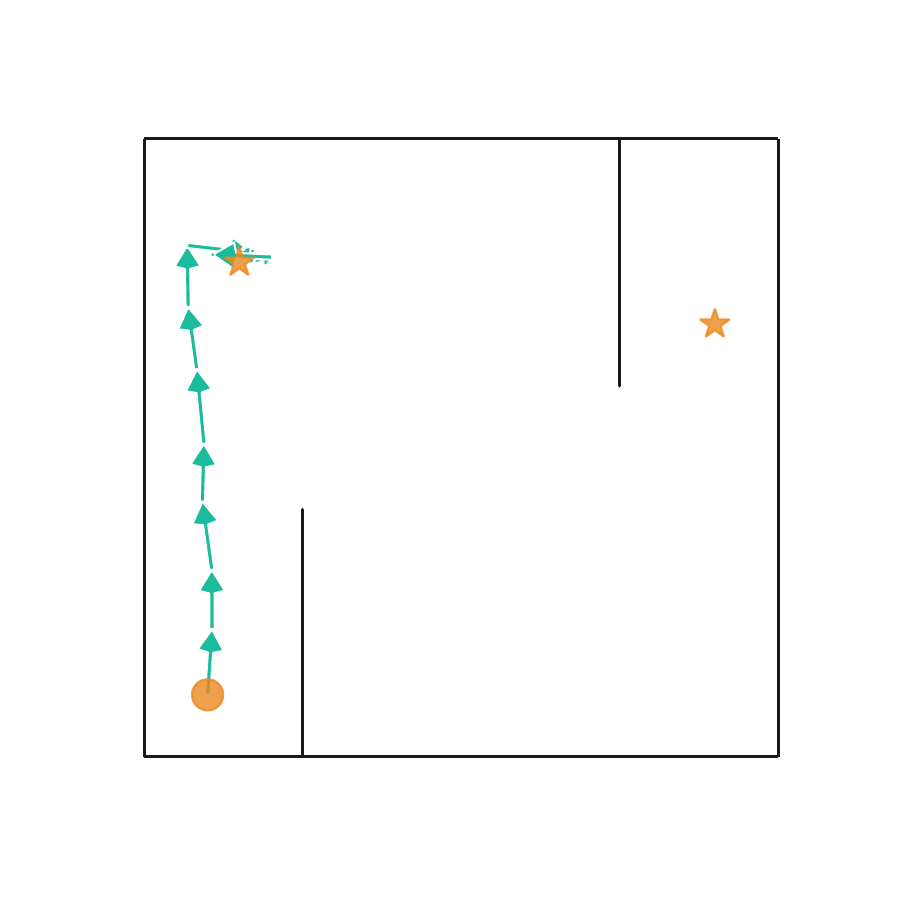}
     \includegraphics[width=0.30\textwidth,  trim = {1cm 1cm 1cm 4cm}]{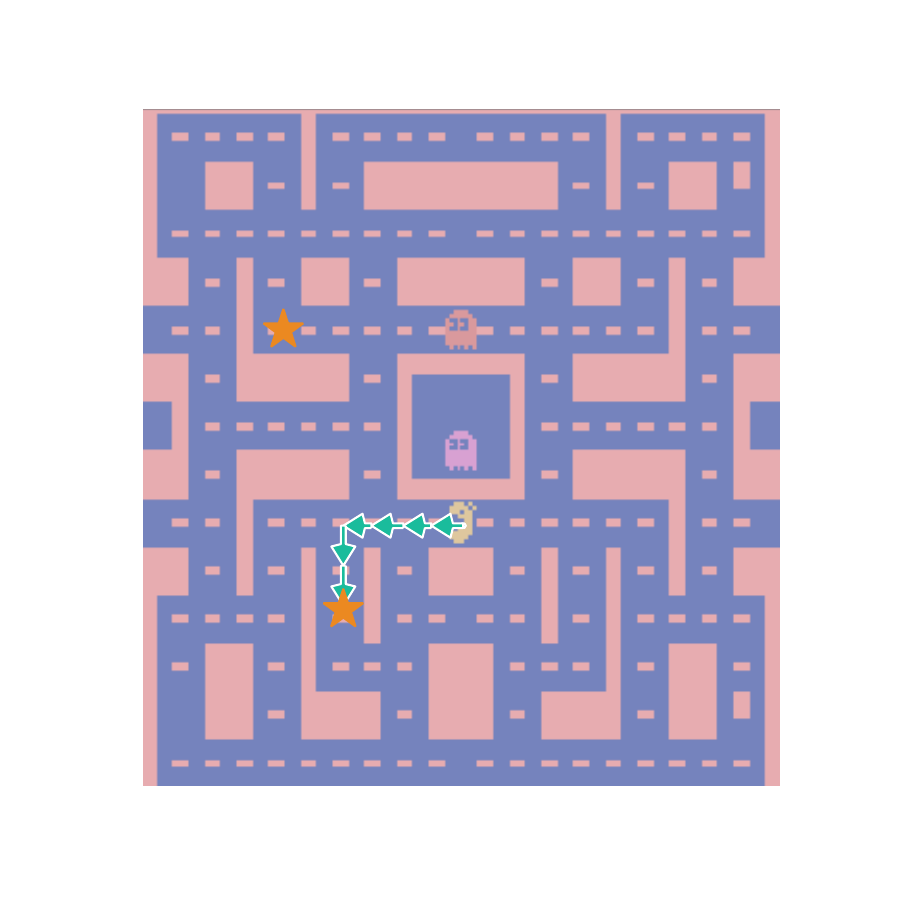} 
    \caption{\label{figure: equal reward} \small Trajectories generated
    by the $\transp{F}B$ policies for the task of reaching the closest among
    two equally rewarding positions (star shapes \score{1}{1}).
    (Optimal
    $Q$-values are not linear over such mixtures.)}
\end{figure}

\end{minipage}
\end{minipage}

\subsection{Embedding Visualizations}

\begin{figure}[t]
  \centering
  \includegraphics[width=0.26\textwidth]{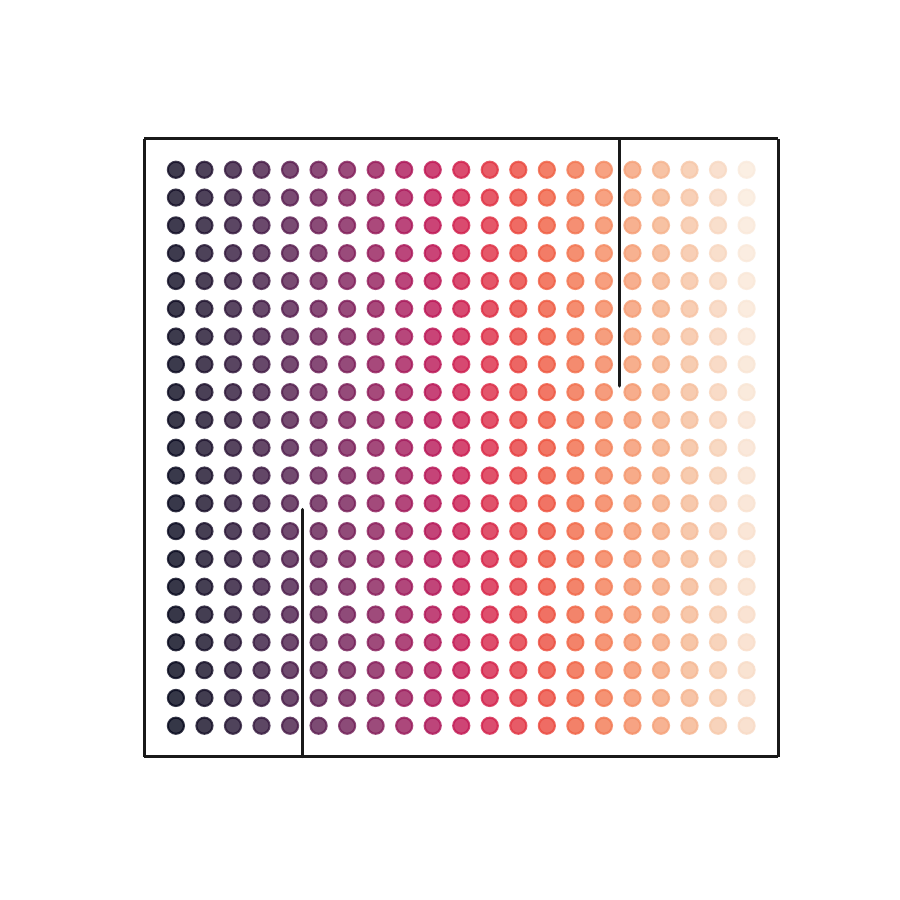}
    \includegraphics[width=0.25\textwidth,  trim = {1.5cm 1.5cm 1.5cm 1.5cm}]{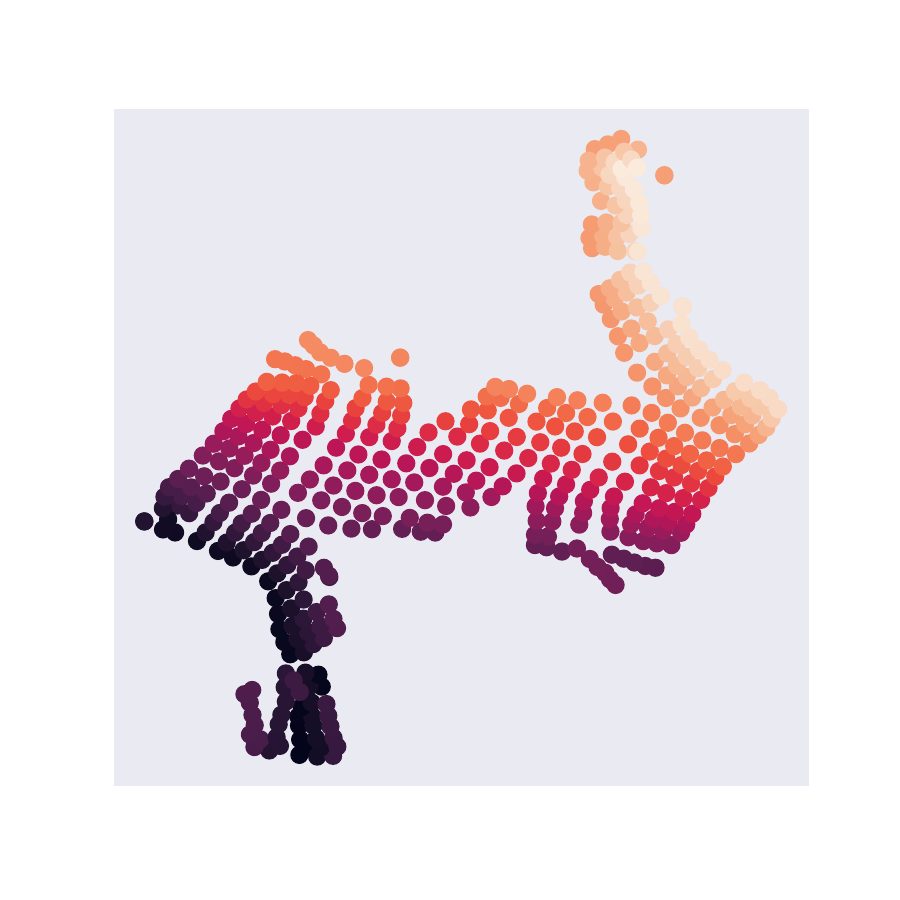}
     \includegraphics[width=0.25\textwidth,  trim = {1.5cm 1.5cm 1.5cm 1.5cm}]{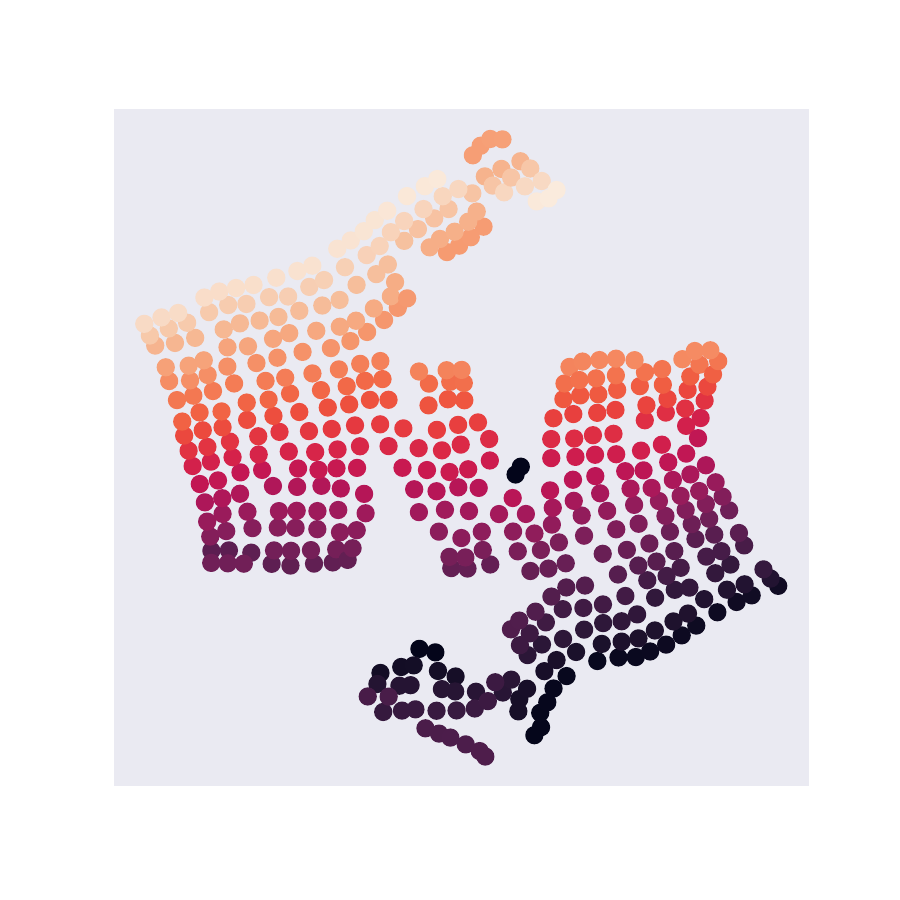}
    \caption{\label{figure: t-sne} Visualization of FB embedding vectors
    on Continuous Maze after projecting them in two-dimensional space
    with t-SNE. 
    \textbf{Left}: the states to be mapped. \textbf{Middle}: the $F$
    embedding. \textbf{Right}: the $B$ embedding. The walls appear
    as large dents; the smaller dents correspond to the number of steps
    needed to get past a wall.}
\end{figure}

We visualize the learned FB state embeddings for Continuous Maze by
projecting them into 2-dimensional space using
t-SNE~\cite{van2008visualizing} in Fig.~\ref{figure: t-sne}. For the
forward embeddings, we set $z=0$ corresponding to the uniform policy. We
can see that FB partitions states according to the topology induced by
the dynamics: states on opposite sides of walls are separated in the
representation space and states on the same side lie together.
Appendix~\ref{sec:setup}
includes
embedding visualizations for different $z$ and for Discrete Maze
and  Ms.\ Pacman.


\section{Related work}
\cite{borsa2018universal} learn optimal policies for rewards that
are linear combinations of a finite number of feature
functions provided in advance by the user. 
\NDY{a bit repetitive of intro}
%
This approach 
cannot tackle generic
rewards or
goal-oriented RL: this would require introducing one feature per possible
goal state, requiring infinitely many features in continuous
spaces.

Our approach does not require user-provided features describing the
future tasks, thanks to using successor \emph{states}
\cite{successorstates} where
\cite{borsa2018universal} use successor \emph{features}. Schematically,
and omitting actions, successor features start with user-provided
features $\phi$, then learn $\psi$ such that $\psi(s_0)=\sum_{t\geq 0}
\gamma^t \E [\phi(s_t)\mid s_0]$.
This limits applicability to rewards that are
linear combinations of $\phi$.  Here we use successor \emph{state}
probabilities,
namely, we learn two representations $F$ and $B$ such that
$\transp{F(s_0)}B(s')=\sum_{t\geq 0} \gamma^t \Pr(s_t=s'\mid s_0)$.  This does
not require any user-provided input.

\option{Thus we learn two representations instead of one. The learned
backward representation $B$ is absent from \cite{borsa2018universal}. $B$
plays a different role than the user-provided features $\phi$ of
\cite{borsa2018universal}: if the reward is known a priori to depend only on some
features $\phi$, we learn $B$ on top of $\phi$, which
represents all rewards that depend linearly or nonlineary on $\phi$. Up
to a change of variables,
\cite{borsa2018universal} is recovered by setting $B=\Id$
on top of $\phi$, or $B=\phi$ and
$\phi=\Id$, and then only training $F$.}

We use a similar parameterization of policies by $\transp{F(s,a,z)}z$ as
in \cite{borsa2018universal}, for similar reasons, although $z$ 
encodes a different object.

Successor representations where first defined in
\cite{dayan1993improving} for finite spaces\option{, corresponding to an older
object from Markov chains, the
fundamental matrix \cite{kemenysnell1960,
bremaud1999markov,grinstead1997introduction}}. 
\option{\cite{stachenfeld2017hippocampus} argue for their relevance for cognitive
science.}
For successor
representations in continuous spaces, a
finite number of features $\phi$ are specified first; this can be used
for generalization within a family of tasks, e.g.,
\cite{barreto2017successor,zhang2017deep,grimm2019disentangled,hansen2019fast}. 
\cite{successorstates} moves from successor features to successor states
by providing pointwise occupancy map estimates even in continuous spaces, without using
the sparse reward $\1_{s_t=s'}$.
We borrow a successor state learning algorithm from \cite{successorstates}.
\cite{successorstates}
also introduced simpler versions of $F$ and $B$ for a single, fixed
policy; \cite{successorstates} does not consider the every-optimal-policy
setting.

There is a long literature on goal-oriented RL. For instance,
\cite{pmlr-v37-schaul15} learn goal-dependent value functions\option{,
regularized via an
explicit matrix factorization}. Goal-dependent value functions have been investigated in earlier works such as~\cite{foster2002structure} and~\cite{sutton2011horde}. Hindsight experience replay (HER)~\cite{andrychowicz2017hindsight} improves the sample efficiency of multiple goal learning with sparse rewards.
A family
of rewards has to be specified beforehand, such as reaching arbitrary
target states.  Specifying rewards a posteriori is not possible:
for instance, learning to reach target states does not extend to reaching
the nearest among several goals, reaching a goal while avoiding forbidden
states, or maximizing any dense reward.

Hierarchical methods such as options \cite{sutton1999between} can be used
for multi-task RL problems. However, policy learning on
top of the options  is still needed after the task is known.

For finite state spaces, \cite{jin2020reward} use reward-free
interactions to build a training set that
summarizes a finite environment, in the sense that
any optimal policies later computed on this training set instead of
the true environment are provably $\eps$-optimal, for any reward.  They
prove tight bounds on the necessary set size. Policy learning still
has to be done afterwards for each reward. 


\NDY{Mahadevan 2007 proto value functions + subsequent work? maybe not.
builds a basis for value functions by a small rank approx of the
Laplacian of a discretized MDP, similar to FB, but
still needs to learn value function afterwards via least squares policy
iteration, so it's different}

\paragraph{Acknowledgments.} The authors would like to thank Léonard
Blier, Diana Borsa, Alessandro Lazaric, Rémi Munos, Tom Schaul, Corentin
Tallec, Nicolas Usunier, and the anonymous reviewers for numerous
comments, technical questions, references, and invaluable suggestions for
presentation that led to an improved text.

\section{Conclusion}

The FB representation is a learnable mathematical object that
``summarizes'' a reward-free MDP. It
provides near-optimal policies for any reward specified a
posteriori, without planning. It is learned from black-box reward-free
interactions with the environment. In practice, this unsupervised method performs comparably to
goal-oriented methods for reaching arbitrary goals, but is also able to
tackle more complex rewards in real time.
\option{The representations learned encode the MDP dynamics and may have
broader interest.}
\NDY{improve?}

\bibliographystyle{alpha}
\bibliography{biblio}


\newpage

\section*{Checklist}


\begin{enumerate}

\item For all authors...
\begin{enumerate}
  \item Do the main claims made in the abstract and introduction accurately reflect the paper's contributions and scope?
    \answerYes{}
  \item Did you describe the limitations of your work?
    \answerYes{}, Section~\ref{sec:algo}
  \item Did you discuss any potential negative societal impacts of your work?
    \answerNo{This is a theoretical work on fully generic methods for
    reinforcement learning. The potential impact is the same as that of reinforcement
    learning in general.}
  \item Have you read the ethics review guidelines and ensured that your paper conforms to them?
    \answerYes{}
\end{enumerate}

\item If you are including theoretical results...
\begin{enumerate}
  \item Did you state the full set of assumptions of all theoretical results?
    \answerYes{}
	\item Did you include complete proofs of all theoretical results?
    \answerYes{}, Appendix~\ref{appendix: proof}.
\end{enumerate}

\item If you ran experiments...
\begin{enumerate}
  \item Did you include the code, data, and instructions needed to reproduce the main experimental results (either in the supplemental material or as a URL)?
    \answerYes{}  in the supplemental material.
  \item Did you specify all the training details (e.g., data splits, hyperparameters, how they were chosen)?
    \answerYes{}  Appendix~\ref{sec:setup}
	\item Did you report error bars (e.g., with respect to the random seed after running experiments multiple times)?
    \answerYes{}
	\item Did you include the total amount of compute and the type of resources used (e.g., type of GPUs, internal cluster, or cloud provider)?
    \answerNo{}
\end{enumerate}

\item If you are using existing assets (e.g., code, data, models) or curating/releasing new assets...
\begin{enumerate}
  \item If your work uses existing assets, did you cite the creators?
    \answerYes{} Appendix~\ref{sec:env}
  \item Did you mention the license of the assets?
    \answerNo{}
  \item Did you include any new assets either in the supplemental material or as a URL?
    \answerNo{}
  \item Did you discuss whether and how consent was obtained from people whose data you're using/curating?
    \answerNA{}
  \item Did you discuss whether the data you are using/curating contains personally identifiable information or offensive content?
    \answerNA{}
\end{enumerate}

\item If you used crowdsourcing or conducted research with human subjects...
\begin{enumerate}
  \item Did you include the full text of instructions given to participants and screenshots, if applicable?
    \answerNA{}
  \item Did you describe any potential participant risks, with links to Institutional Review Board (IRB) approvals, if applicable?
    \answerNA{}
  \item Did you include the estimated hourly wage paid to participants and the total amount spent on participant compensation?
    \answerNA{}
\end{enumerate}

\end{enumerate}


\newpage

\appendix

\section*{Outline of the Supplementary Material}

The Appendix is organized as follows.
\begin{itemize}
\item Appendix~\ref{appendix: algo} presents the pseudo-code of the unsupervised phase of FB algorithm.
\item Appendix~\ref{appendix: extended results} provides extended theoretical  results on approximate solutions and general goals:
		\begin{itemize}
			\item Section~\ref{sec: FB with goal} formalizes the forward-backward representation with a goal or feature space.
			\item Section~\ref{sec:dimension} establishes the
			existence  of exact FB representations in finite
			spaces, and discusses the influence of the
			dimension $d$.
			\item Section~\ref{sec:approx} shows how approximate solutions provide approximately optimal policies.
			\item Section~\ref{sec:succpred} shows how $F$
			and $B$ are successor and predecessor features of
			each other, and how the policies are optimal for
			rewards linearly spanned by $B$.
			\item Section~\ref{sec:rhotest} explains how to
			estimate $z_R$ at test time from a state
			distribution different from the training
			distribution.
			\item Section~\ref{sec:diracs} presents a note of the measure $M^\pi$ and its density $m^\pi$.
		\end{itemize}
		
\item Appendix~\ref{appendix: proof} provides proofs of all theoretical results above.
\item Appendix~\ref{sec:setup} provides additional information about our experiments:
		\begin{itemize}
			\item Section~\ref{sec:env} describes the environments.
			\item Section~\ref{sec:archi} describes the different architectures used for FB as well as the goal-oriented DQN.
			\item Section~\ref{sec:implem} provides implementation details and hyperparameters.
			\item Section~\ref{sec:addit_results} provides additional experimental results. 
		\end{itemize}
\end{itemize}

\newpage
\section{Algorithm}
\label{appendix: algo}

The unsupervised phase of the FB algorithm is described in
Algorithm~\ref{FB algo}.

The loss function for the Bellman equation on $F$ and $B$ appears on line
19, while the auxiliary loss function for orthonormalization of $B$ appears on
line 21.

\begin{algorithm}[h!]
\caption{FB algorithm: Unsupervised Phase}\label{FB algo}
\begin{algorithmic}[1]
\STATE \textbf{Inputs:} replay buffer $\mathcal{D}$ , Polyak coefficient $\alpha$  , $\nu$ a probability distribution over $\mathbb{R}^d$, randomly initialized networks $F_\theta$  and $B_\omega$, learning rate $\eta$, mini-batch size $b$, number of episodes $E$, number of gradient updates $N$, temperature $\tau$ and regularization coefficient $\lambda$.

\FOR{$m = 1, \ldots$}
\STATE {\bf  \textcolor{gray!50!blue}{/* \texttt{Collect $E$ episodes}}}
\FOR{episode $e=1, \ldots E$}
\STATE Sample $z \sim \nu$
\STATE Observe an initial state $s_0$
\FOR{$t =  1, \ldots $}
\STATE Select an action $a_t$ according to some behaviour policy (e.g the $\epsilon$-greedy with respect to  $F_\theta(s_t, a, z)^\top z$ ) 
\STATE Observe next state $s_{t+1}$
\STATE Store transition $(s_t, a_t, s_{t+1})$ in the replay buffer $\mathcal{D}$
\ENDFOR
\ENDFOR
\STATE {\bf  \textcolor{gray!50!blue}{/* \texttt{Perform $N$ stochastic gradient descent updates}}}
\FOR{ $ n= 1 \ldots N$}
\STATE Sample a mini-batch of transitions $\{(s_i, a_i, s_{i+1})\}_{i \in I} \subset \mathcal{D}$ of size $|I| = b$.
\STATE Sample a mini-batch of target state-action pairs  $\{(s'_i, a'_i)\}_{i \in I} \subset \mathcal{D}$ of size $|I| = b$.
\STATE Sample a mini-batch of $\{z_i \}_{i \in I} \sim \nu$ of size $|I| = b$.
\STATE Set $\pi_{z_i}( \cdot \mid s_{i+1} ) = \texttt{softmax} (F_{\theta^{-}}(s_{i+1}, \cdot, z_i)^\top z_i / \tau )$

\STATE $ \mathscr{L} (\theta, \omega) = $\\ $\frac{1}{2 b^2} \sum_{i, j \in I^2} \left( F_\theta(s_i, a_i,z_i)^\top B_\omega(s'_j, a'_j)  - \gamma \sum_{a \in A} \pi_{z_i} (a \mid s_{i+1}) \cdot F_{\theta^{-}}(s_{i+1}, a ,z_i)^\top B_{\omega^{-}}(s'_j, a'_j)  \right)^2 - \frac{1}{b} \sum_{i \in I} F_\theta(s_i, a_i,z_i)^\top  B_\omega(s_i, a_i)  $

\STATE {\bf  \textcolor{gray!50!blue}{/* \texttt{Compute orthonormality regularization loss}}}
\STATE $ \mathscr{L}_{\texttt{reg}} (\omega) = \frac{1}{b^2} \sum_{i, j \in I^2} B_{\omega}(s_i, a_i)^\top \texttt{stop-gradient}(B_{\omega}(s'_j, a'_j) ) \cdot \texttt{stop-gradient}  (B_{\omega}(s_i, a_i)^\top B_{\omega}(s'_j, a'_j)) -  \frac{1}{b} \sum_{i \in I}
B_{\omega}(s_i, a_i)^\top \texttt{stop-gradient}(B_{\omega}(s_i, a_i) )$
\STATE Update $\theta \leftarrow \theta - \eta \nabla_{\theta}  \mathscr{L}(\theta, \omega)$ and $\omega \leftarrow \omega - \eta \nabla_{\omega}  (\mathscr{L}(\theta, \omega) + \lambda \cdot  \mathscr{L}_{\texttt{reg}}(\omega))$
\ENDFOR

\STATE {\bf  \textcolor{gray!50!blue}{/* \texttt{Update target network parameters}}}
\STATE  $\theta^{-} \leftarrow \alpha \theta^{-} + (1-\alpha) \theta$ 
\STATE  $\omega^{-} \leftarrow \alpha \omega^{-} + (1-\alpha) \omega$
\ENDFOR

\end{algorithmic}
\end{algorithm}

\newpage
\section{Extended Results: Approximate Solutions and General Goals}
\label{appendix: extended results}

\begin{notation*}
In general, we denote by $M^\pi$ the successor measure of policy $\pi$ as
defined in \eqref{eq:defM}, and by $m^\pi$ its density, if it exists,
with respect to a reference measure $\rho$. Namely,
\begin{equation}
\label{eq:defm}
M^\pi(s_0,a_0,\d s,\d a)=m^\pi(s_0,a_0,s,a)\rho(\d s,\d a).
\end{equation}
Thus, the defining property of
forward-backward representations (Definition~\ref{def:fb}) is
$\transp{F(s_0,a_0,z)}B(s,a)=m^{\pi_z}(s_0,a_0,s,a)$.

We use the same convention for parametric models, with $M^\pi_\theta$ a
measure and $m^\pi_\theta$ its density.
The reference measure $\rho$ is fixed and may be unknown (typically
$\rho$ is the distribution of
state-actions in a training set or under an exploration policy).
\end{notation*}

\subsection{The Forward-Backward Representation With a Goal or Feature Space}
\label{sec: FB with goal}

Here we state a generalization of Theorem~\ref{thm:main} covering some
extensions mentioned in the text.

First, this covers rewards known to 
only 
depend on certain features $g=\phi(s,a)$ of the state-action $(s,a)$,
where $\phi$ is a known function
with values in some goal state $G$ (for instance, rewards depending
only on some components of the state). Then it is enough to compute $B$
as a function of the goal $g$.  Theorem~\ref{thm:main} corresponds to
$\phi=\Id$.  This is useful to introduce prior information when
available, resulting in a smaller model $(F,B)$.

\option{ This also recovers successor features as in
\cite{borsa2018universal}, defined by user-provided features $\phi$.
Indeed, fixing $B$ to $\Id$ and setting our $\phi$ to the $\phi$ of
\cite{borsa2018universal} (or fixing $B$ to their $\phi$ and our $\phi$
to $\Id$) will represent the same set of rewards and policies as in
\cite{borsa2018universal}, namely, optimal policies for rewards linear in
$\phi$
(although with a slightly different learning
algorithm and up to a linear change of variables for $F$ and $z$ given
by the covariance of $\phi$, see Appendix~\ref{sec:succpred}).
More generally, keeping the same
$\phi$ but letting $B$ free (with larger $d$) can provide optimal
policies for rewards that are arbitrary functions of $\phi$, linear or
not.  }

For this, we extend successor state measures to values in goal
spaces, representing the discounted time spent at each goal by the policy. Namely, 
given a policy $\pi$, let $M^\pi$ be the the successor state measure of
$\pi$ over goals $g$:
\begin{equation}
\label{eq:Mgoal}
M^\pi(s,a,\d g)\deq \sum_{t\geq 0} \gamma^t \Pr\left(
\phi(s_t,a_t)\in \d g \mid s_0=s,\,a_0=a,\,\pi
\right)
\end{equation}
for each state-action $(s,a)$ and each measurable set $\d g\subset G$.
This will be the object approximated by
$\transp{F(s,a,z)}B(g)$.

Second, we use a more general model of
successor states: instead of $m\approx \transp{F}B$ we use
$m\approx 
\transp{F}B+\bar m$ where $\bar m$ does not depend on the action, so that the
$\transp{F}B$ part only computes advantages. This lifts the
constraint that the model of $m$ has rank at most $d$, because there is
no restriction on the rank on $\bar m$: the rank restriction only applies
to the advantage function.

Third, we state a form of policy improvement for the FB representation.
Namely, the $Q$-function $\transp{F(s,a,z_R)}z_R$ for a given reward can
be computed as a supremum over all values of $z$.

For simplicity we state the result with deterministic rewards, but this
extends to stochastic rewards, because the expectation $z_R$ will be the
same.

\begin{definition}[Extended forward-backward representation of an MDP]
\label{def:fb2}
Consider an MDP with state space $S$ and action space
$A$.
Let $\phi\from S\times A\to G$ be a function from state-actions to some
goal space $G=\R^k$.

Let $Z=\R^d$ be some representation space. Let
\begin{equation}
F\from S\times A\times Z\to Z,\qquad B\from G\to Z,\qquad
\bar m\from S\times Z\times G\to \R
\end{equation}
be three functions. For each $z\in Z$, define the policy
\begin{equation}
\pi_z (a|s) \deq \argmax_a \transp{F(s,a,z)}z.
\end{equation}

Let $\rho$ be any measure over the goal space $G$.

We say that $F$, $B$, and $\bar m$ are an \emph{extended forward-backward
representation} of the MDP with respect to $\rho$, if the following
holds:
for any $z\in Z$, any state-actions $(s,a)$, and any goal
$g\in G$, one has
\begin{equation}
\label{eq:extendedFB}
M^{\pi_z}(s,a,\d g)=\left(\transp{F(s,a,z)}B(g)+\bar m(s,z,g)\right)\rho(\d g)
\end{equation}
where $M^{\pi_z}$ is the successor state measure \eqref{eq:Mgoal} of
policy $\pi_z$.
\end{definition}

\begin{thm}[Forward-backward representation of an MDP, with features as
goals]
\label{thm:main2}
Consider an MDP with state space $S$ and action space
$A$.
Let $\phi\from S\times A\to G$ be a function from state-actions to some
goal space $G=\R^k$.

Let $F$, $B$, and $\bar m$ be an extended forward-backward representation
of the MDP with respect to some measure $\rho$ over $G$.

Then the following holds.
Let $R\from S\times A \to \R$ be any bounded reward function, and assume that
this reward function depends only on $g=\phi(s,a)$, namely, that there exists a
function $r\from G\to \R$ such that $R(s,a)=r(\phi(s,a))$.
Set
\begin{equation}
\label{eq:zR2}
z_R\deq \int_{g\in G} r(g)B(g) \,\rho(\d g)
\end{equation}
assuming the integral exists.

Then:
\begin{enumerate}
\item $\pi_{z_R}$ is an optimal policy for reward $R$ in the MDP.

\item For any $z\in Z$, the $Q$-function of policy $\pi_z$ for the reward
$R$ is equal to
\begin{equation}
\label{eq:Qfunc_general}
Q^{\pi_z}(s,a)=\transp{F(s,a,z)}z_R+\bar V^{z}(s)
\end{equation}
and the optimal $Q$-function $Q^\star_R$ is obtained when $z=z_R$:
\begin{equation}
Q^\star_R(s,a)=\transp{F(s,a,z_R)}z_R+\bar V^{z_R}(s).
\end{equation}

Here
\begin{equation}
\bar V^{z}(s)\deq \int_{g\in G} \bar m(s,z,g)r(g)\,\rho(\d g)
\end{equation}
and in particular $\bar V=0$ if $\bar m=0$.

The advantages $Q^{\pi_z}(s,a)-Q^{\pi_z}(s,a')$ do not
depend on $\bar V$, so computing $\bar V$ is not necessary to obtain the
policies.

\item If $\bar m=0$, then for any state-action $(s,a)$ one has
\NDY{this does not work with $\bar m$!!!}
\begin{equation}
Q^\star_R(s,a)=\transp{F(s,a,z_R)}z_R
=
\sup_{z\in Z} \transp{F(s,a,z)}z_R. 
\end{equation}
\end{enumerate}
\end{thm}

(We do not claim that $\bar V$ is the value function and $\transp{F}z_R$
the advantage function, only that the sum is the $Q$-function. When $\bar
m=0$, the term $\transp{F}z_R$ is the whole $Q$-function.)

The last point of the theorem is a form of policy improvement. 
Indeed, by the second point, $\transp{F(s,a,z)}z_R$ is the estimated
$Q$-function of policy $\pi_z$ for rewards $r$. This may be useful if $z_R$
falls outside of the training distribution for $F$: then the values of
$F(s,a,z_R)$ may not be safe to use. In that case, it may be useful to
use a finite set $Z'\subset Z$ of values of $z$ closer to the training
distribution, and use the estimate $\sup_{z\in Z'} \transp{F(s,a,z)}z_R$
instead of $\transp{F(s,a,z_R)}z_R$ for the optimal $Q$-function. A similar option has been used, e.g., in
\cite{borsa2018universal}, but in the end it was not necessary in our
experiments.

\begin{remark}
Formally, the statement holds for arbitrary $\rho$, but it only makes
sense if $\rho$ has full support (or at least covers all reachable parts
of the state space): \eqref{eq:extendedFB} requires the support of
$M^{\pi_z}$ to be included in that of $\rho$. Otherwise, FB
representations may not exist and the statement is empty.
\end{remark}

\subsection{Existence of Exact $FB$ Solutions, Influence of Dimension
$d$, Uniqueness}
\label{sec:dimension}

\paragraph{Existence of exact $FB$ representations in finite spaces.}
We now prove existence of an exact solution for finite spaces if the
representation dimension $d$ is at least $\#S\times \#A$. Solutions are
never unique: one may always multiply $F$ by an invertible matrix $C$ and
multiply $B$ by $(\transp{C})^{-1}$, see Remark~\ref{rem:normB} below (this allows us to impose
orthonormality of $B$ in the experiments).

The constraint $d\geq \#S\times \#A$ can be largely overestimated
depending on the tasks of interest, though. For instance, we
prove below that in an $n$-dimensional toric grid $S=\{1,\ldots,k\}^n$,
$d=2n$ is enough to obtain optimal policies for 
reaching every target state (a set of tasks smaller than optimizing
all possible rewards).

\begin{proposition}[Existence of an exact $FB$ representation for finite
state spaces]
\label{prop:existence}
Assume that the state and action spaces $S$ and $A$ of an MDP are finite.
Let $Z=\R^d$ with $d\geq \#S\times \#A$. Let $\rho$ be any measure 
on $S\times A$, with $\rho(s,a)>0$ for any $(s,a)$.

Then there exists $F\from S\times A\times Z\to Z$ and $B\from S\times
A\to Z$, such that $\transp{F}B$ is equal to the successor state density
of $\pi_z$ with respect to $\rho$:
\begin{equation}
\transp{F}(s,a,z)B(s',a')=\sum_{t\geq 0} \gamma^t \frac{\Pr((s_t,a_t)=
(s', a') \mid s_0=s,\,a_0=a,\,\pi_z)}{\rho(s',a')}
\end{equation}
for any $z\in Z$ and any state-actions $(s,a)$ and $(s',a')$, where
$\pi_z$ is defined as in Definition~\ref{def:fb} by $\pi_z(s)=\argmax_a
\transp{F(s,a,z)}z$.
\end{proposition}

\paragraph{A small dimension $d$ can be enough for navigation: examples.}
In practice, even a small $d$ can be enough to get optimal policies
for reaching arbitrary many states (as opposed to optimizing all possible
rewards). Let us give an example with $S$ a toric $n$-dimensional grid of
size $k$.

Let us start with $n=1$.  Take $S=\{0,\ldots,k-1\}$ to be a length-$k$
cycle with three actions $a\in \{-1,0,1\}$ (go left, stay in place, go
right). Take $d=2$, so that $Z=\R^2\simeq\C$.

We consider the tasks of reaching an arbitrary target state $s'$, for
every $s'\in S$. Thus the goal state is $G=S$ in the notation of
Theorem~\ref{thm:main2}, and $B$ only depends on $s'$. The policy for
such a reward is $\pi_{z_R}=\pi_{B(s')}$.

For a state $s\in \{0,\ldots,k-1\}$ and action $a\in
\{-1,0,1\}$, define
\begin{equation}
F(s,a,z)\deq e^{2i\pi(s+a)/k},\qquad B(s)\deq e^{2i\pi s/k}.
\end{equation}

Then one checks that $\pi_{B(s')}$ is the optimal policy for reaching
$s'$, for every $s'\in S$.
Indeed, $\transp{F(s,a,z_R)}z_R=\cos (2\pi (s+a-s')/k)$. This is
maximized for the action $a$ that brings $s$ closer to $s'$.

So the policies will be optimal for reaching every target $s'\in S$,
despite the dimension being only $2$.

By taking the product of $n$ copies of this example, this also works on
the $n$-dimensional toric grid $S=\{0,\ldots,k-1\}^n$ with $2n+1$ actions
(add $\pm1$ in each direction or stay in place), with a representation of
dimension $d=2n$ in $\C^n$, namely, by taking $B(s)_j\deq e^{2i\pi
s_j/k}$ for ecah direction $j$ and likewise for $F$. Then $\pi_{B(s')}$
is the optimal policy for reaching $s'$ for every $s'\in S$.

\NDY{do we learn such representations in practice?}

\NDY{do we make official statements for this?}

More generally, if one is only interested in the optimal policies for 
reaching states, then it is easy to show that there exist functions
$F\from S\times A\to Z$ and $B\from S\to Z$
such that the policies $\pi_z$ describe the optimal policies to reach
each state: it is enough that $B$ be injective (typically requiring
$d=\dim(S)$). Indeed, for any state $s\in S$, let $\pi^\star_s$ be the
optimal policy to reach $s$. We want $\pi_z$ to be equal to
$\pi^{\star}_s$ for $z=B(s)$ (the value of $z_R$ for a reward located at
$s$). This translates as $\argmax_a
\transp{F(s',a,B(s))}B(s)=\pi^\star_s(s')$ for every other state $s'$.
This is realized just by letting $F$ be any function such that
$F(s',\pi^\star_s(s'),B(s))\deq B(s)$ and $F(s',a,B(s))\deq -B(s)$ for
every other action $a$. As soon as $B$ is injective, there exists such a
function $F$. \option{(Unfortunately, we are not able to show that the learning
algorithm reaches such a solution.)}

\bigskip

Let us turn to uniqueness of $F$ and $B$.

\begin{remark}
\label{rem:normB}
Let $C$ be an invertible $d\times d$ matrix. Given $F$ and $B$ as in
Theorem~\ref{thm:main}, define
\begin{equation}
B'(s,a)\deq CB(s,a),\qquad F'(s,a,z)\deq (\transp{C})^{-1}F(s,a,C^{-1}z)
\end{equation}
together with the policies $\pi'_z(s)\deq \argmax_a \transp{F'(s,a,z)}z$.
For each reward $r$, define $z'_R\deq \E_{(s,a)\sim \rho}
[r(s,a)B'(s,a)]$.

Then this operation does not change the policies or estimated $Q$-values: for any reward,  we have $\pi'_{z'_R}=\pi_{z_R}$, and
$\transp{F'(s,a,z'_R)}z'_R=\transp{F(s,a,z_R)}z_R$.

In particular, assume that the components of $B$ are linearly
independent. Then,
taking $C=
\left(\E_{(s,a)\sim \rho} B(s,a)\transp{B(s,a)}\right)^{-1/2}$, $B'$ is
$L^2(\rho)$-orthonormal. So up to reducing the dimension $d$ to $\rank(B)$, we can always assume that
$B$ is $L^2(\rho)$-orthonormal, namely, that $\E_{(s,a)\sim \rho}
B(s,a)\transp{B(s,a)}=\Id$.
\end{remark}

Reduction to orthonormal $B$ will be useful in some proofs below.
Even after imposing that $B$ be orthonormal, solutions are not unique, as
one can still apply a rotation matrix on the variable $z$.

For a single policy $\pi_z$, the $\transp{F}B$ decomposition may be further standardized in several
ways: after a linear change of variables in $\R^d$, and up to decreasing
$d$ by removing unused directions in $\R^d$, one may assume either that
$\Cov_\rho B=\Id$ and $\Cov_\rho F$ is diagonal, or that $\Cov_\rho
F=\Cov_\rho B$ is diagonal, or write the decomposition as
$\transp{\tilde F}D\tilde B$ with $D$ diagonal and $\Cov_\rho \tilde
F=\Cov_\rho \tilde B=\Id$, thus corresponding to an approximation of the
singular value decomposition of the successor measure $M^{\pi_z}$ in
$L^2(\rho)$. However, since $B$ is shared between all values of $z$ and
all policies $\pi_z$, it is a priori not possible to realize this for all
$z$ simultaneously.

\subsection{Approximate Solutions Provide Approximately Optimal Policies}
\label{sec:approx}

Here we prove that the optimality in Theorems~\ref{thm:main}
and~\ref{thm:main2} is robust to approximation errors during training: 
approximate solutions still provide approximately
optimal policies.
We
deal first with the approximation errors on $(F,B)$ during unsupervised
training, then on $z_R$ during the reward
estimation phase (in case the reward is not known explicity).

\subsubsection{Influence of Approximate $F$ and $B$: Optimality Gaps are
Proportional to $m^\pi-\transp{F}B$}

In continuous spaces, Theorems~\ref{thm:main} and \ref{thm:main2} are somewhat spurious: the equality
$\transp{F}B=m$ that defines FB representations will never hold exactly with finite representation
dimension $d$. 
Instead, $\transp{F}B$ will only be a rank-$d$ approximation of $m$.
Even in finite spaces, since $F$ and $B$ are learned by a neural
network, we can only expect that $\transp{F}B\approx m$ in general.
Therefore, we provide results that extend Theorems~\ref{thm:main} and
\ref{thm:main2} to approximate training and approximate FB
representations.

Optimality gaps are directly controlled by the error $m^\pi-\transp{F}B$ on the
solution $\transp{F}B$.
We provide this result for different notions of
approximations between $m^\pi$ and $\transp{F}B$. First, in sup norm over $(s,a)$
but in expectation
over $(s',a')$ (so that a perfect model of the successor states $(s',a')$ of each
$(s,a)$ is not necessary, only an average model). Second, 
for the weak topology on
measures (this is the most relevant in continuous spaces: for instance, a
Dirac measure can be approached by a continuous model in the weak
topology). 
Finally, we provide pointwise estimates instead of only norms: for any reward, we show that the optimality gaps at each state can be 
bounded by an explicit matrix product directly involving the FB error matrix
$m^{\pi}-\transp{F}B$ (Theorem~\ref{thm:pointwiseapprox}).

$F$ and $B$ are trained such that 
$\transp{F(s,a,z)}B(s',a')$ approximates the successor state density
$m^{\pi_{z}}(s,a,s',a')$. In the simplest case, we prove that
if for some reward $R$,
\begin{equation}
\E_{(s',a')\sim \rho} \abs{\transp{F(s,a,z_R)}B(s',a')-
m^{\pi_{z_R}}(s,a,s',a')}\leq \eps
\end{equation}
for every $(s,a)$,
then the optimality gap of policy $\pi_{z_R}$ is at most
$(3\eps/(1-\gamma))\sup\abs{R}$ for that reward (Theorem~\ref{thm:approx},
first case).

In continuous spaces, $m^\pi$ is usually a distribution
(Appendix~\ref{sec:diracs}), so such an approximation will not hold, and it is better to work on the measures
themselves rather than their densities, namely, to compare
$\transp{F}B\rho$ to $M^\pi$ instead of $\transp{F}B$ to $m^\pi$. We prove that if $\transp{F}B\rho$ is close
to $M^\pi$ in the weak topology,
 then the resulting policies are
optimal for any Lipschitz reward.\footnote{This also holds for continuous
rewards, but the Lipschitz assumption yields an explicit bound in
Theorem~\ref{thm:approx}.}

Remember that a sequence of nonnegative measures $\mu_n$ converges weakly to $\mu$ if
for any bounded, continuous function $f$, $\int f(x)\mu_n(d x)$ converges
to $\int f(x)\mu(\d x)$ (\cite{bogachev2007measure}, \S8.1). The associated
topology can be defined via the
following \emph{Kantorovich--Rubinstein} norm on nonnegative measures
(\cite{bogachev2007measure}, \S8.3)
\begin{equation}
\KRnorm{\mu-\mu'}\deq \sup \left\{
\abs{\int f(x)\mu(\d x)-\int f(x)\mu'(\d x)} : \,
f\text{ $1$-Lipschitz function with }\sup \abs{f}\leq 1
\right\}
\end{equation}
where we have equipped the state-action space with any metric compatible
with its topology.\footnote{The Kantorovich--Rubinstein norm is closely related to the $L^1$ Wasserstein
distance on probability distributions, but slightly more general as it does
not require the distance functions to be integrable: the Wasserstein
distance metrizes weak convergence among those probability measures such that
$\E[d(x,x_0)]<\infty$.}

\bigskip

The following theorem states that if $\transp{F}B$ approximates
the successor state density of the policy $\pi_z$ (for various sorts of
approximations), then $\pi_z$ is approximately optimal. Given a reward
function $r$ on state-actions, we denote
\begin{equation}
\norm{r}_\infty\deq \sup_{(s,a)\in S\times A} \abs{r(s,a)}
\end{equation}
and
\begin{equation}
\lipnorm{r}\deq \sup_{(s,a)\neq (s',a')\in S\times A}
\frac{r(s,a)-r(s',a')}{d((s,a),(s',a'))}
\end{equation}
where we have chosen any metric on state-actions.

The first statement is for any bounded reward. The second statement only
assumes an $\transp{F}B$ approximation in the weak topology but only applies to
Lipschitz rewards. The third statement is more general and is how we
prove the first two: weaker assumptions on $\transp{F}B$ work on a
stricter class of rewards.


\begin{thm}[If $F$ and $B$ approximate successor states, then the
policies $\pi_z$ yield approximately optimal returns]
\label{thm:approx}
Let $F\from S\times A\times Z\to Z$ and $B\from S\times A\to Z$
be any functions, and define the policy $\pi_z(s)=\argmax_a
\transp{F(s,a,z)}z$ for each $z\in Z$.

Let $\rho$ be any positive probability distribution on $S\times A$, and for each
policy $\pi$, let $m^\pi$ be the density of the successor state measure
$M^\pi$ of $\pi$ with respect to $\rho$. Let
\begin{equation}
\hat m^z(s,a,s',a')\deq \transp{F(s,a,z)}B(s',a'),\qquad \hat M^z(s,a,\d
s',\d a')\deq \hat m^z(s,a,s',a')\rho(\d s',\d a')
\end{equation}
be the estimates of $m$ and $M$ obtained via the model $F$ and $B$.

Let $r\from S\times A\to \R$ be any bounded reward function. Let $V^\star$
be the optimal value function for this reward $r$. Let $\hat V^{\pi_z}$
be the value function of policy $\pi_z$ for this reward. Let
$z_R=\E_{(s,a)\sim \rho} [r(s,a)B(s,a)]$.

Then:
\begin{enumerate}
%
\item If $\E_{(s',a')\sim \rho} \abs{\hat
m^{z_R}(s,a,s',a')-m^{\pi_{z_R}}(s,a,s',a')}\leq
\eps$ for any $(s,a)$ in $S\times A$, then
$\norm{\hat V^{\pi_{z_R}}-V^\star}_\infty\leq 3\eps
\norm{r}_\infty/(1-\gamma)$.

\item If $r$ is Lipschitz and $\KRnorm{\hat
M^{z_R}(s,a,\cdot)-M^{\pi_{z_R}}(s,a,\cdot)}\leq \eps$ for any
$(s,a)\in S\times A$, then $\norm{\hat V^{\pi_{z_R}}-V^\star}_\infty\leq
3\eps \max(\norm{r}_\infty,\lipnorm{r})/(1-\gamma)$.

\item More generally, let $\norm{\cdot}_A$ be a norm on functions and
$\norm{\cdot}_B$ a norm on measures, such that $\int f\d \mu \leq
\norm{f}_A\norm{\mu}_B$ for any function $f$ and measure $\mu$. Then for any
reward function $r$ such that $\norm{r}_A<\infty$,
\begin{equation}
\norm{\hat V^{\pi_{z_R}}-V^\star}_\infty\leq
\frac{3\norm{r}_A}{1-\gamma}\sup_{s,a} \norm{\hat
M^{z_R}(s,a,\cdot)-M^{\pi_{z_R}}(s,a,\cdot)}_B.
\end{equation}
Moreover, the optimal $Q$-function is close to $\transp{F(s,a,z_R)}z_R$:
\begin{equation}
\sup_{s,a}\abs{\transp{F(s,a,z_R)}z_R-Q^\star(s,a)}\leq
\frac{2\norm{r}_A}{1-\gamma}\sup_{s,a} \norm{\hat
M^{z_R}(s,a,\cdot)-M^{\pi_{z_R}}(s,a,\cdot)}_B.
\end{equation}
\end{enumerate}
\end{thm}

\paragraph{Pointwise optimality gaps.}
We now turn to a more precise estimation of the optimality gap at each
state, expressed
directly as a matrix product involving the FB error $m^\pi-\transp{F}B$.

Here we assume that the state space is  finite: this is not essential but
simplifies notation since everything can be represented as matrices and
vectors. For each stochastic policy $\pi\from S\to \mathrm{Prob}(A)$, we denote $P_\pi$ the associated
stochastic matrix on state-actions:
\begin{equation}
(P_\pi)_{(sa)(s'a')}\deq P(s'|s,a)\pi(s')[a'].
\end{equation}
We also view rewards and $Q$-functions as vectors indexed by
state-actions.

\begin{thm}
\label{thm:pointwiseapprox}
Assume that the state space is finite. Let $F\from S\times A\times
\R^d\to \R^d$ and $B\from S\times A\to \R^d$ be any functions. 
Define $\pi_z$ as in Definition~\ref{def:fb}.
Let $\rho>0$
be any probability distribution over state-actions, and let $m^\pi$ be the successor
density \eqref{eq:defm} of policy $\pi$ with respect to $\rho$.

For each $z\in \R^d$, define the
FB
error $E(z)$ as a matrix over state-actions:
\begin{equation}
E(z)_{(sa)(s'a')}\deq m^{\pi_z}(s,a,s',a')-\transp{F(s,a,z)}B(s',a').
\end{equation}

Let $r$ be any reward function and
let $Q^\star$ and $\pi^\star$ be its optimal $Q$-function and policy. Let
$z_R=\E _\rho[r.B]$ as in
Theorem~\ref{thm:main}, and let $Q^{\pi_{z_R}}$ be the $Q$-function of
policy $\pi_{z_R}$.

Then we have the componentwise inequality between state-action vectors
\begin{equation}
0\leq Q^\star-Q^{\pi_{z_R}} \leq \left(\sum_{t\geq 0} \gamma^{t+1} P_{\pi^\star}^t
\right)\left(P_{\pi^\star}-P_{\pi_{z_R}}\right)E(z_R)\diag(\rho)\,r.
\end{equation}

In particular, if $E(z_R)=0$ then $\pi_{z_R}$ is optimal for reward $r$.
\end{thm}

The matrix $\sum_{t\geq 0} \gamma^{t+1} P_{\pi^\star}^t$ represents states
visited along the optimal trajectory starting at the initial state.
Multiplying by $(P_{\pi^\star}-P_{\pi_{z_R}})$ visits states one step away
from this trajectory. Therefore, what matters are the values of the FB
error matrix
$E(z_R)_{(sa)(s'a')}$ at state-actions $(s,a)$ one step away from the
optimal trajectories.

In unsupervised RL, we have no control over the first part $\sum_{t\geq 0}
\gamma^{t+1} P_{\pi^\star}^t$, which depends only on the optimal
trajectories of the unknown future tasks $r$. Therefore, in general it makes
sense to just minimize $E(z)$ in some matrix norm, for as many values of
$z$ as possible. (The choice of
matrix norm influences which rewards will be best optimized, as
illustrated by Theorem~\ref{thm:approx}.)

\subsubsection{An approximate $z_R$ yields an approximately optimal policy}

We now turn to the second source of approximation: computing $z_R$ in the
reward estimation phase. This is a problem only if the reward is not
specified explicitly.

We deal in turn with the effect of using a model of the reward function,
and the effect of estimating $z_R=\E_{(s,a)\sim \rho} r(s,a)B(s,a)$ via
sampling.

Which of these options is better depends a lot on the situation.  If
training of $F$ and $B$ is perfect, by construction the policies are
optimal for each $z_R$: thus, estimating $z_R$ from rewards sampled at
$N$ states $(s_i,a_i)\sim \rho$ will produce the optimal policy for
exactly that empirical reward, namely, a nonzero reward at each $(s_i,
a_i)$ but zero everywhere else, thus overfitting the reward
function. Reducing the dimension $d$ reduces this effect, since rewards
are projected on the span of the features in $B$: $B$ plays both the
roles of a transition model and a reward regularizer. This appears as a
$\sqrt{d/N}$ factor in Theorem~\ref{thm:approxzR} below.

Thus, if both the number of samples to train $F$ and $B$ and the number
of reward samples are small, using a smaller $d$ will regularize both the
model of the environment and the model of the reward. However, if the
number of samples to train $F$ and $B$ is large, yielding an excellent
model of the environment, but the number of reward samples is small, then
learning a model of the reward function will be a better option than
direct empirical estimation of $z_R$.

\bigskip

The first result below states that reward misidentification comes on top of the
 approximation error of the $\transp{F}B$ model. This is relevant, for instance, if a reward model
$\hat r$ is estimated by an external model using some reward values.

\begin{proposition}[Influence of estimating $z_R$ by an approximate
reward]
\label{prop:approxr}
Assume $\rho$ is a probability distribution.
Let $\hat r\from S\times A\to \R$ be any reward function.
Let $\hat
z_R=\E_{(s,a)\sim \rho} [\hat r(s,a)B(s,a)]$.

Let $\eps_{FB}$ be the error attained
by the $\transp{F}B$ model in
Theorem~\ref{thm:approx} for reward $\hat r$; namely, assume that
$\norm{\hat
V^{\pi_{\hat z_R}}-\hat V^\star}_\infty\leq \eps_{FB}$ with $\hat V^\star$ the optimal
value function for $\hat r$.

Then the policy $\pi_{\hat z_R}(s)=\argmax_a \transp{F(s,a,\hat z_R)}\hat
z_R$ defined by the model $\hat r$ is
$\displaystyle\left(\frac{2 \norm{r-\hat
r}_\infty}{(1-\gamma)}+\eps_{FB}\right)$-optimal for reward $r$.
\end{proposition}

This still assumes that the expectation $\hat z_R=\E_{(s,a)\sim \rho}
[\hat r(s,a)B(s,a)]$ is computed exactly for the model $\hat r$. If $\hat
r$ is given by an explicit model, this expectation can in principle be computed on the whole replay
buffer used to train $F$ and $B$, so variance would be low. Nevertheless, we
provide an additional statement which covers the influence of the
variance of the estimator of $z_R$, whether this estimator uses an
external model $\hat r$ or a direct empirical average reward observations
$r(s,a)B(s,a)$.

\begin{definition}
\label{def:skew}
The \emph{skewness} $\zeta(B)$ of $B$ is defined as follows.
Assume $B$ is bounded.
Let $B_1,\ldots,B_d\from S\times A\to \R$ be the functions of $(s,a)$
defined by each component of $B$. Let $\langle B\rangle$ be the linear
span of the $(B_i)_{1\leq i\leq d}$ as
functions on $S\times A$. Set
\begin{equation}
\zeta(B)\deq \sup_{f\in \langle B \rangle,\,f\neq 0}
\frac{\norm{f}_\infty}{\norm{f}_{L^2(\rho)}}.
\end{equation}
\end{definition}

\begin{thm}[Influence of estimating $z_R$ by empirical averages]
\label{thm:approxzR}
Assume that $\rho$ is a probability distribution.
Assume that $z_R=\E_{(s,a)\sim \rho} [r(s,a)B(s,a)]$ is estimated via
\begin{equation}
\hat z_R\deq \frac1N \sum_{i=1}^N \hat r_i B(s_i,a_i)
\end{equation}
using $N$ independent samples $(s_i,a_i)\sim \rho$, where the $r_i$ are 
random variables such that $\E [\hat r_i|s_i,a_i]=r(s_i,a_i)$, $\Var
[\hat r_i|s_i,a_i]\leq v$ for some $v\in\R$,
and the $\hat r_i$ are mutually independent given $(s_i,a_i)_{i=1,\ldots,N}$.

Let $V^\star$ be the optimal value function for reward $r$, and let $\hat
V$ be the value function of the estimated policy $\pi_{\hat z_R}$ for reward $r$.

Then, for any $\delta>0$, with probability at least $1-\delta$,
\begin{equation}
\norm{\hat V-V^\star}_\infty \leq
\eps_{FB}+
\frac{2}{1-\gamma} \sqrt{
\frac{\zeta(B)\,d}{N\delta} \left(v+\norm{r(s,a)-\E_{\rho}
r}^2_{L^2(\rho)}\right)
}
\end{equation}
which is therefore the bound on the optimality gap of $\pi_{\hat z_R}$
for $r$. Here $\eps_{FB}$ is the error due to the $\transp{F}B$ model
approximation, defined as in Proposition~\ref{prop:approxr}.
\end{thm}

The proofs are not direct, because $F$ is not continuous with respect to
$z$. Contrary to $Q$-values, successor states are not continuous in the
reward: if an action has reward $1$ and the reward for another action
changes
from $1-\eps$ to $1+\eps$, the return values change by at most $2\eps$,
but the actions and states visited by the optimal policy change a lot. So
it is not possible to reason by continuity on each of the terms involved.

\subsection{$F$ and $B$ as Successor and Predecessor Features of Each
Other, Optimality for Rewards in the Span of $B$}
\label{sec:succpred}

We now give two statements. The first encodes the idea that $F$ encodes
the future of a state while $B$ encodes the past of a state.

The second proves that the $FB$ policies are optimal for any reward that
lies in the linear span of the features learned in $B$; for rewards out of
this span, it is best if the features in $B$ are spatially smooth.

The intuition that $F$ and $B$ encode the future and past of states is
formalized as follows:
if $F$ and $B$ minimize their unsupervised loss, then $F$ is equal to the
successor features from the dual features of $B$, and $B$ is equal to the
\emph{predecessor} features from the dual features of $F$
(Theorem~\ref{thm:forwardbackwardinterpretation}).

This statement holds for a fixed $z$ and the corresponding policy
$\pi_z$. So, for the rest of this section, $z$ is fixed. In this section,
we also assume
that $\rho$ is a probability distribution.

By ``dual'' features we mean the following. Define the $d\times d$ covariance
matrices
\begin{equation}
\label{eq:covariances}
\Cov F\deq \E_{(s,a)\sim \rho} [F(s,a,z)\transp{F(s,a,z)}],
\qquad
\Cov B\deq \E_{(s,a)\sim \rho} [B(s,a)\transp{B(s,a)}].
\end{equation}
Then $(\Cov F)^{-1/2} F(s,a,z)$ is $L^2(\rho)$-orthonormal and likewise
for $B$. The ``dual'' features are $(\Cov F)^{-1} F(s,a,z)$ and $(\Cov
B)^{-1}B(s,a)$, without the square root: these are the least square
solvers for $F$ and $B$ respectively, and these are the ones that appear
below.

The unsupervised forward-backward loss 
for a fixed $z$ is
\begin{align}
\label{eq:loss2}
\ell(F,B)&\deq \int \abs{\transp{F(s,a,z)}B(s',a')-\sum_{t\geq 0} \gamma^t
\frac{P_t(\d s',\d a'|s,a,\pi_z)}{\rho(\d s',\d a')}}^2 \rho(\d s,\d
a)\rho(\d s',\d a')
\\&=\norm{\transp{F(\cdot,z)}B(\cdot)-m^{\pi_z}(\cdot,\cdot)}^2_{L^2(\rho)\otimes
L^2(\rho)}.
\end{align}
Thus, minimizers in dimension $d$ correspond to an SVD of
the successor state density in $L^2(\rho)$, truncated to the largest $d$ singular values.

\begin{thm}
\label{thm:forwardbackwardinterpretation}
Consider a smooth parametric model for $F$ and $B$, and assume this model
is overparameterized. \footnote{
Intuitively, a parametric function $f$ is overparameterized if every
possible small change of $f$ can be realized by
a small change of the parameter.
Formally, we say that a parametric family of functions
$\theta \in \Theta\mapsto f_\theta \in L^2(X,\R^d)$ smoothly parameterized by
$\theta$, on some space $X$, is
\emph{overparameterized} if, for any $\theta$, the differential
$\partial_\theta f_\theta$ is surjective from $\Theta$ to $L^2(X,\R^d)$.
For finite $X$, this implies that the
dimension of $\theta$ is larger than $\# X$. For infinite $X$, this
implies that $\dim(\theta)$ is infinite, such as parameterizing functions
on $[0;1]$ by their Fourier expansion.}
Also assume that the data distribution $\rho$ has
positive density everywhere.

Let $z\in Z$. Assume that for this $z$,
$F$ and $B$ lie in $L^2(\rho)$ and achieve a local extremum of
$\ell(F,B)$ within this
parametric model. Namely, the derivative $\partial_\theta \ell(F,B)$ of
the loss with respect to the parameters $\theta$ of $F$ is $0$, and
likewise for $B$.

Then $F$ is equal to $(\Cov B)^{-1}$ times the successor features of $B$:
for any $(s,a)\in S\times A$,
\begin{equation}
(\Cov B) F(s,a,z)=\sum_{t\geq 0} \gamma^t \int_{(s',a')} P_t(\d s',\d a'|s,a,\pi_z) \,B(s',a')
\end{equation}
and $B$ is equal to $(\Cov F)^{-1}$ times the predecessor features of $F$:
\begin{equation}
(\Cov F) B(s',a')=\sum_{t\geq 0} \gamma^t \int_{(s,a)} \frac{P_t(\d s',\d
a'|s,a,\pi_z)}{\rho(\d s',\d a')} \,F(s,a,z) \,\rho(\d s,\d a)
\end{equation}
$\rho$-almost everywhere.
Here the covariances have been defined in \eqref{eq:covariances}, and
$P_t(\cdot |s,a,\pi)$ denotes the law of $(s_t,a_t)$ under
trajectories starting at $(s,a)$ and following policy $\pi$.

The same result holds when working with features $\phi(s',a')$, just by
applying it to $B\circ \phi$.
\end{thm}

Note that in the FB framework, we may normalize either $F$ or $B$
(Remark~\ref{rem:normB}), but not both.

As a consequence of Theorem~\ref{thm:forwardbackwardinterpretation}, we characterize  below which kind of rewards we can capture if we fix $B$ and train only $F$ to minimize the unsupervised loss given $B$.  Namely, we show that for any reward $r$, the resulting policy is optimal for the $L^2(\rho)$-orthogonal projection of $r$ onto the span of $B$.

\begin{thm}[Optimizing $F$ for a given $B$; influence of the span of $B$]
\label{thm:fixedbackward}
Let $B$ a fixed function in $L^2(\rho,\R^d)$. Define the \emph{span} of $B$ as
the set of functions $\transp{w}B\in L^2(\rho,\R)$ when $w$ ranges in
$\R^d$.

Consider a smooth parametric model for $F$, and assume this model is
overparameterized.  Assume that $F$ lies in $L^2(\rho)$ and achieves a
local extremum of $\ell(F,B)$ within this parametric model for each $z
\in Z$. 

Assume that the data distribution $\rho$ has positive density everywhere.

Then for any bounded reward function $r\from  S \times A \rightarrow \R$
that lies in the span of $B$,
$\pi_{z_R}$ is an optimal policy for the reward $r$.

More generally, for any bounded reward function $r\from  S \times A
\rightarrow \R$, the policy $\pi_{z_R}$ is an optimal policy for the
reward $r_B$ defined as the  $L^2(\rho)$-orthogonal projection of $r$
onto the span of $B$.

Moreover, 
\begin{equation}
\label{eq:simplifiedlaplacianbound}
\norm{V^{\pi_{z_R}} - V^\star}_\infty \leq \frac{2}{1-\gamma} \norm{r - r_B}_\infty
\end{equation}
with $V^\star$ the optimal value function for $r$. More precisely,
\begin{equation}
\label{eq:laplacianbound}
\norm{V^{\pi_{z_R}} - V^\star}_\infty \leq \norm{(\Id-\gamma
P_{\pi^\star})^{-1}(r -
r_B)}_\infty
+\norm{(\Id-\gamma
P_{\pi_{z_R}})^{-1}(r -
r_B)}_\infty
\end{equation}
with $\pi^\star$ the optimal policy for reward $r$, and notation $P_\pi$
as in Theorem~\ref{thm:pointwiseapprox}.
\end{thm}

The last bounds \eqref{eq:laplacianbound} implies
the previous one \eqref{eq:simplifiedlaplacianbound}, as  $(\Id-\gamma P_{\pi})^{-1}$ is bounded by
$\frac{1}{1-\gamma}$ in $L^\infty$ norm for any policy $\pi$.

\NDY{Maybe this theorem should be closer to section on optimality/approx solutions?}

\paragraph{Discussion: optimal $B$ and priors on rewards.}
This bound is interesting if $r-r_B$ is small, namely, if $B$ captures
most of the components of the reward functions we are interested in.
But even for rewards not spanned by $B$,
the bound is smaller if $r-r_B$ avoids the largest
eigendirections of $(\Id-P_\pi)^{-1}$ for various policies $\pi$, namely,
if $B$ captures these largest eigendirections. These
eigendirections are those of $P_\pi$: so the bound will be small 
if $B$ contains the largest eigendirections of $P_\pi$ for various
policies $\pi$, corresponding to spatially continuity of functions under
transitions in the environment ($P_\pi f$
close to $f$).

If we are interested in spatially smooth rewards $r$, then
$r-r_B$ is small if $B$
if captures smooth functions first. 
%
But even for rewards not spanned by $B$, and for non-spatially smooth rewards (e.g., goal-oriented problems
with reward $\1_{s=\mathrm{goal}}$),
the bound \eqref{eq:laplacianbound} shows that
$B$ should first capture spatially smooth eigenvectors of many policies
$\pi$.  

Is this a natural consequence of FB training?  Up to some approximations,
yes.  For a single $z$ and policy $\pi_z$, the loss \eqref{eq:loss2} used
to train $B$ is optimal when $B$ captures the largest \emph{singular}
directions of $(\Id-P_{\pi_z})^{-1}$, which is slightly different. The
optimal policy $(\Id-P_{\pi^\star})^{-1}$ is not represented in the
criterion, and it is not clear to what extent spatial continuity with
respect to $P_{\pi_z}$ or $P_{\pi^\star}$ differ.  Moreover, in the full
algorithm, $B$ is shared between several $z$ and several policies. So we
have no rigorous result here.  Still, the intuition shows that FB
training goes in the correct direction.

\subsection{Estimating $z_R$ from a Different State Distribution at Test
Time}
\label{sec:rhotest}

\newcommand{\rhotest}{\rho_\mathrm{test}}

The algorithm of
Section~\ref{sec:algo} computes $F$ and $B$ with respect a reference
measure $\rho$ equal to the distribution
$\rho$ of state-actions in the training set.
Theorem~\ref{thm:main} requires $z_R=\E_{(s,a)\sim \rho} [r(s,a)B(s,a)]$
to be estimated using rewards observed from the same state-action
distribution $\rho$ as the one used to train $F$ and
$B$. 

So in general, estimating $z_R$ requires
either being able to run the exploration policy again once reward
samples are available, or being able to explicitly estimate $r(s,a)$ on states stored
in the training set.

However, at test time, we would generally like to use the policy learned,
rather than the exploration policy. This will result in a distribution of
states-actions $\rhotest(s,a)$ different from the training
distribution $\rho(s,a)$.

If the training set remains accessible, a generic solution is to train a
model $\hat r(s,a)$ of the reward function, from rewards $r(s,a)$
observed at test time under any distribution $(s,a)\sim \rhotest$. Then one
can estimate $z_R$ by averaging this model $\hat r$ over state-actions sampled
from the training set:
\begin{equation}
\label{eq:hatzR}
\hat z_R\deq \E_{(s,a)\sim \rho} [\hat
r(s,a)B(s,a)].
\end{equation}

However, the training set may not be available anymore at test time. But
as it
turns out, if we use for $\hat r$ a \emph{linear} model based on the
features learned in $B$, then we do not need to store the training set:
it is enough to estimate the matrix $\Cov_\rho(B)$, which can be
pre-computed during training. This is summarized in the following.

\begin{proposition}
\label{prop:linearz}
Let $\hat r$ be the linear model of rewards computed at test time by linear
regression of the reward $r$ over the components $B_1,\ldots,B_d$ of $B$, with
state-actions taken from a test distribution $\rhotest$. Explicitly,
\begin{equation}
\label{eq:linearhatr}
\hat r(s,a)\deq \transp{B(s,a)}w, \qquad w\deq \left(
\Cov_{\rhotest} B
\right)^{-1} \E_{(s,a)\sim \rhotest} [r(s,a)B(s,a)].
\end{equation}
Here we assume
that $\Cov_{\rhotest} B=\E_{(s,a)\sim \rhotest}
[B(s,a)\transp{B(s,a)}]$ is invertible.

Then the estimate $\hat z_R$ computed by using this
model $\hat r$ in \eqref{eq:hatzR} is
\begin{equation}
\label{eq:linearhatzR}
\hat z_R=(\Cov_\rho B)\left(\Cov_{\rhotest} B
\right)^{-1} \E_{(s,a)\sim \rhotest} [r(s,a)B(s,a)].
\end{equation}

Moreover, 
if $\rho=\rhotest$, or if $r$ is linear over
$B$, then $\hat z_R=z_R$ .
\end{proposition}

For this estimate, $\Cov_{\rhotest} B$ and $\E_{(s,a)\sim \rhotest}
[r(s,a)B(s,a)]$ can be computed at test time, while the matrix $\Cov_\rho
B$ must be computed at training time. This way, the training set can be
discarded.

If the estimate \eqref{eq:linearhatr} is used, the learned policies correspond to using
universal successor features approximators \cite{borsa2018universal} on top of the features learned
by $B$. 
Indeed,
universal successor features with features $\phi$ use the policies
$\argmax_a \transp{\psi(s,a,w)}w$ with $w=\left(\Cov_{\rhotest} \phi
\right)^{-1} \E_{(s,a)\sim \rhotest} [r(s,a)\phi(s,a)]$ the regression
vector of $r$ over the features $\phi$, and with $\psi=\Succ(\phi)$ the
successor features of $\phi$. Here we use the policies $\pi_{\hat z_R}=\argmax_a
\transp{F(s,a,\hat z_R)}\hat z_R$. Let us set $\phi\deq B$ as the base
features in successor features.
Then the above shows that $\hat
z_R=(\Cov_\rho B)w$ when using the linear model of rewards.
Moreover, we proved in Theorem~\ref{thm:forwardbackwardinterpretation}
that the optimum for $F$ is $F=(\Cov_\rho B)^{-1} \Succ(B)$. Therefore,
the policies coincide in this situation.

Thus, universal successor features based on $B$ appear as a
particular case if a linear model of rewards is used at test time,
although in general any reward model may be used in \eqref{eq:hatzR}.

\subsection{A Note on the Measure $M^\pi$ and its Density $m^\pi$}
\label{sec:diracs}

In finite spaces, the definition of the successor state density $m^\pi$ via 
\begin{equation}
M^\pi(s,a,\d s',\d
a')=m^\pi(s,a,s',a')\rho(\d s',\d a')
\end{equation}
with respect to the data
distribution $\rho$ poses no problem, as long as the data distribution is
positive everywhere.

In continuous spaces, this can be understood as the (Radon--Nikodym)
density of $M^\pi$ with respect to $\rho$, assuming $M^\pi$ has no
singular part with respect to $\rho$. However, this is \emph{never} the case:
in the definition \eqref{eq:defM} of the
successor state measure $M^\pi$, the term $t=0$ produces a Dirac measure
$\delta_{s,a}$. So $M^\pi$ has a singular component due to $t=0$, and
$m^\pi$ is better thought of as a distribution.

When $m^\pi$ is a distribution, a continuous parametric model
$m^\pi_\theta$ learned by \eqref{eq:mtd} can approximate $m^\pi$ in the weak
topology only: $m^\pi_\theta \rho$ approximates $M^\pi$ for the weak
convergence of measures. Thus, for the forward-backward representation,
$\transp{F(s,a,z)}B(s',a')\rho(\d
s',\d a')$ weakly approximates $M^{\pi_z}(s,a,\d s',\d a')$.

We have not found this to be a problem either in theory or practice. In
particular, 
Theorem~\ref{thm:approx} covers weak approximations.

\bigskip

Alternatively, one may just define successor states starting at $t=1$ in
\eqref{eq:defM}. This only works well if rewards $r(s,a)$ depend on the state
$s$ but not the action $a$ (e.g., in goal-oriented settings).
If starting the definition at $t=1$, $m^\pi$ is an ordinary function provided the
transition kernels $P(\d s'|s,a)$ of the environment are non-singular,
$\rho$ has positive density, and $\pi(\d a|s)$ is non-singular as well.
Starting at $t=1$ induces the following changes in the theorems:
\begin{itemize}
\item In the learning algorithm \eqref{eq:mtd} for successor states, the
term $\partial_\theta m^\pi_\theta(s,a,s,a)$ becomes $\gamma \,
\partial_\theta m^\pi_\theta(s,a,s',a')$.
\item The expression for the $Q$-function in Theorem~\ref{thm:main}
becomes $Q^\star(s,a)=r(s,a)+\transp{F(s,a,z_R)}z_R$, and likewise in
Theorem~\ref{thm:main2}. The $r(s,a)$ term covers the immediate reward at
a state, since we have excluded $t=0$ from the definition of successor
states.
\item In general the expression for optimal policies becomes
\begin{equation}
\pi_z(s)\deq \argmax_a\{ r(s,a)+\transp{F(s,a,z)}z\}
\end{equation}
which cannot be computed from $z$ and $F$ alone in the unsupervised
training phase. The algorithm only makes sense for rewards that depend on
$s$ but not on
$a$ (e.g., in goal-oriented settings): then the policy $\pi_z$ is equal to $\pi_z(s)\deq
\argmax_a\transp{F(s,a,z)}z$ again.
\end{itemize}

\newpage
\section{Proofs}
\label{appendix: proof}
The first proposition is a direct consequence of the definition
\eqref{eq:Mgoal} of successor states with a goal space $G$.

\begin{proposition}
\label{prop:Qsucc}
Let $\phi\from S\times A\to G$ be a map to some goal space $G$.

Let $\pi$ be some policy, and let $M^\pi$ be the successor state measure
\eqref{eq:Mgoal}
of $\pi$ in goal space $G$. Let $m^\pi$ be the density of $M^\pi$ with respect to some
positive
measure $\rho$ on $G$.

Let $r\from G\to \R$ be some function on $G$, and define the reward
function $R(s,a)\deq r(\phi(s,a))$ on $S\times A$.

Then
the $Q$-function $Q^\pi$ of
policy $\pi$ for reward $R$ is
\begin{align}
Q^\pi(s,a)&=\int r(g)\,M^\pi(s,a,\d g)
\\&=\int_{g\in G } r(g)\,m^\pi(s,a,g)\,\rho(\d g).
\end{align}
\end{proposition}

\begin{proof}[Proof of Proposition~\ref{prop:Qsucc}]
For each time $t\geq 0$, let $P^\pi_t(s_0,a_0,\d g)$ be the probability
distribution of $g=\phi(s_t,a_t)$ over trajectories of the policy $\pi$
starting at $(s_0,a_0)$ in the MDP. Thus, by the definition
\eqref{eq:Mgoal},
\begin{equation}
M^\pi(s,a,\d g)=\sum_{t\geq 0} \gamma^t P_t^\pi(s,a,\d g).
\end{equation}

The $Q$-function of $\pi$ for the reward $R$ is by definition (the
sums and integrals are finite since $R$ is bounded)
\begin{align}
Q^\pi(s,a)&=\sum_{t\geq 0} \gamma^t \E[R(s_t,a_t)\mid s_0=s,\,a_0=a,\,\pi]
\\&=\sum_{t\geq 0} \gamma^t \E[r(\phi(s_t,a_t))\mid s_0=s,\,a_0=a,\,\pi]
\\&=\sum_{t\geq 0} \gamma^t \int_g r(g) P_t^\pi(s,a,\d g)
\\&=\int_g r(g) M^\pi(s,a,\d g)
\\&=\int_g r(g) m^\pi(s,a,g)\rho(\d g)
\end{align}
by definition of the density $m^\pi$.
\end{proof}

\begin{proof}[Proof of Theorems~\ref{thm:main} and \ref{thm:main2}]
Theorem~\ref{thm:main} is a particular case of Theorem~\ref{thm:main2}
($\phi=\Id$ and $\bar m=0$), so we only prove the latter.

Let $R(s,a)=r(\phi(s,a))$ be a reward function as in the theorem.

For any policy $\pi$, let $M^\pi$ be its successor measure defined by
\eqref{eq:Mgoal}, and let
$m^{\pi}$ denote its density with respect to $\rho$. 

The $Q$-function of $\pi$ for the reward $R$ is, by
Proposition~\ref{prop:Qsucc},
\begin{align}
Q^\pi(s,a)
=\int_g r(g) m^\pi(s,a,g)\rho(\d g)
\end{align}

The
definition of an extended FB representation states that
for any $z\in Z$, $m^{\pi_{z}}(s,a,g)$ is equal
to $\transp{F(s,a,z)}B(g)+\bar m(s,z,g)$.

Therefore, for any $z\in Z$ we have
\begin{align}
Q^{\pi_z}(s,a)&=\int_g r(g) \left(
\transp{F(s,a,z)}B(g)+\bar m(s,z,g)
\right)\rho(\d g)
\\&=\transp{F(s,a,z)} \int_g r(g)B(g)\rho(\d g)+\int_g
r(g)\bar m(s,z,g)\rho(\d g)
\\&=\transp{F(s,a,z)}z_R+\bar V^z(s)
\end{align}
by definition of $z_R$ and $\bar V$. This proves the claim
\eqref{eq:Qfunc_general} about
$Q$-functions.

By definition, the policy $\pi_z$ selects the action $a$ that maximizes
$\transp{F(s,a,z)}z$.
Take $z=z_R$. Then
\begin{align}
\pi_{z_R}&=\argmax_a \transp{F(s,a,z_R)}z_R
\\&=\argmax_a \left\{
\transp{F(s,a,z_R)}z_R+\bar V^z(s)
\right\}
\end{align}
since the last term does not depend on $a$.

This quantity is equal to $Q^{\pi_{z_R}}(s,a)$.
Therefore,
\begin{equation}
\label{eq:pizR}
\pi_{z_R}=\argmax_a Q^{\pi_{z_R}}(s,a)
\end{equation}
and by the above, $Q^{\pi_{z_R}}(s,a)$ is indeed equal to the $Q$-function of
policy
$\pi_{z_R}$ for the reward $R$. Therefore, $\pi_{z_R}$ and
$Q^{\pi_{z_R}}$ constitute an optimal Bellman pair for reward $R$. Since
$Q^{\pi_{z_R}}(s,a)$ is the $Q$-function of $\pi_{z_R}$, it satisfies the
Bellman equation
\begin{align}
Q^{\pi_{z_R}}(s,a)&=R(s,a)+\gamma \E_{s'|(s,a)}
Q^{\pi_{z_R}}(s',\pi_{z_R}(s'))
\\&=R(s,a)+\gamma \E_{s'|(s,a)} \max_{a'} Q^{\pi_{z_R}}(s',a')
\end{align}
by \eqref{eq:pizR}. This is the optimal Bellman equation for
$R$, and
$\pi_{z_R}$ is the optimal policy for $R$.

We still have to prove the last statement of
Theorem~\ref{thm:main2}. Since $\pi_{z_R}$ is an optimal policy for $R$,
for any other policy $\pi_z$ and state-action $(s,a)$ we have
\begin{equation}
Q^{\pi_{z_R}}(s,a)\geq Q^{\pi_z}(s,a).
\end{equation}
Using the formulas above for $Q^\pi$,
with $\bar m=0$, this rewrites as
\begin{equation}
\transp{F(s,a,z_R)}z_R\geq \transp{F(s,a,z)}z_R
\end{equation}
as needed. Thus $\transp{F(s,a,z_R)}z_R\geq \sup_{z\in Z}
\transp{F(s,a,z)}z_R$, and equality occurs by taking $z=z_R$.
This ends the proof of Theorem~\ref{thm:main2}.
\end{proof}

\begin{proof}[Proof of Proposition~\ref{prop:existence}]
Assume $d=\#S\times \#A$; extra dimensions can just be ignored by setting
the extra components of
$F$ and $B$ to $0$.

With $d=\#S \times \#A$, we can index the components of $Z$ by pairs
$(s,a)$.

First, let us set $B(s,a)\deq \1_{s,a}$.


Let $r\from S\times A\to \R$ be any reward function. Let $z_R\in
\R^{\#S\times \#A}$ be
defined as in Theorem~\ref{thm:main}, namely, 
\begin{equation}
z_R=\sum_{(s,a)}
r(s,a)B(s,a)\rho(s,a).
\end{equation}
With our choice of $B$, the components of $z_R$ are
$(z_R)_{s,a}=r(s,a)\rho(s,a)$. Since $\rho>0$, the correspondence $r\leftrightarrow z_R$
is bijective.

Let us now define $F$.  Take $z\in Z$. Since $r\leftrightarrow z_R$ is
bijective, this $z$ is equal to $z_R$ for some reward function $r$.
Let $\pi_z$ be an optimal
policy for this reward $r$ in the MDP. Let $M^{\pi_r}$ be the successor state measure
of policy $\pi_z$, namely:
\begin{align}
M^{\pi_z}(s,a,s',a')&=
\sum_{t\geq 0} \gamma^t  \Pr\left((s_t,a_t)=(s',a')\mid
(s_0,a_0)=(s,a),\,\pi_z\right).
\end{align}

Now define $F(s,a,z)$ by setting its $(s',a')$ component to
$M^{\pi_z}(s,a,s',a')/\rho(s',a')$ for each $(s',a')$:
\begin{equation}
F(s,a,z)_{s',a'}\deq M^{\pi_z}(s,a,s',a')/\rho(s',a').
\end{equation}

Then we have
\begin{multline}
\transp{F(s,a,z)}B(s',a')=\sum_{s'',a''}
F(s,a,z)_{s'',a''}\,B(s',a')_{s'',a''}\\=F(s,a,z)_{s',a'}=M^{\pi_z}(s,a,s',a')/\rho(s',a')
\end{multline}
because by our choice of $B$, $B(s',a')_{s''a''}=\1_{s'=s'',\,a'=a''}$.

Thus, $\transp{F(s,a,z)}B(s',a')$ is the density of the successor
state measure $M^{\pi_z}$ of policy $\pi_z$ with respect to $\rho$, as
needed.

We still have to check that $\pi_z$ satisfies $\pi_z(s)=\argmax
\transp{F(s,a,z)}z$ (since this is not how it was defined). Since $\pi_z$ was defined as an optimal policy for
the reward $r$ associated with $z$, it satisfies
\begin{equation}
\pi_z(s)=\argmax_a Q^{\pi_z}(s,a)
\end{equation}
with $Q^{\pi_z}(s,a)$ the $Q$-function of policy $\pi_z$ for the reward
$r$. This $Q$-function is equal to the cumulated expected reward
\begin{align}
Q^{\pi_z}(s,a)&=
\sum_{t\geq 0} 
\gamma^t 
\E \left[
r(s_t,a_t) \mid s_0=s,\,a_0=a,\,\pi_z
\right]
\\&=\sum_{t\geq 0} \gamma^t
\sum_{s',a'} r(s',a') \Pr\left((s_t,a_t)=(s',a')\mid
s_0=s,\,a_0=a,\,\pi_z
\right)
\\&= \sum_{s',a'} r(s',a') \sum_{t\geq 0} \gamma^t
\Pr\left((s_t,a_t)=(s',a')\mid
s_0=s,\,a_0=a,\,\pi_z
\right)
\\&= \sum_{s',a'} r(s',a') M^{\pi_z}(s,a,s',a')
\\&= \sum_{s',a'} r(s',a') F(s,a,z)_{s',a'} \,\rho(s',a')
\\&= \transp{F(s,a,z)}\left(\sum_{s'a'} r(s',a')\rho(s',a')\1_{s'a'}\right)
\\&= \transp{F(s,a,z)}z
\end{align}
since $z$ is equal to $\sum_{(s',a')} r(s',a')B(s',a')\rho(s',a')$. This
proves that $\pi_z(s)=\argmax_a Q^{\pi_z(s,a)}=\argmax_a
\transp{F(s,a,z)}z$. So this choice of $F$ and $B$ satisfies all the
properties claimed.
\end{proof}

We will rely on the following two basic results in $Q$-learning.

\begin{proposition}[$r\mapsto Q^\star$ is Lipschitz in sup-norm]
\label{prop:rtoQ}
Let $r_1,\,r_2 \from S \times A \to \R$ be two bounded reward functions.
Let $Q_1^\star$ and $Q_2^\star$ be the corresponding optimal $Q$-functions,
and likewise for the $V$-functions.
Then
\begin{equation}
\sup_{S\times A} \abs{Q_1^\star-Q_2^\star}\leq \frac{1}{1-\gamma} \sup_{S\times A}
\abs{r_1-r_2} \text{ and } \sup_{S} \abs{V_1^\star-V_2^\star}\leq \frac{1}{1-\gamma} \sup_{S\times A}
\abs{r_1-r_2}.
\end{equation}
Moreover for any policy $\pi$, we have  

\begin{equation}
\sup_{S\times A} \abs{Q_1^\pi-Q_2^\pi}\leq \frac{1}{1-\gamma} \sup_{S\times A}
\abs{r_1-r_2} \text{ and }  
\sup_{S} \abs{V^\pi_1-V^{\pi}_2}\leq \frac{1}{1-\gamma} \sup_{S\times A}
\abs{r_1-r_2}.
\end{equation}
\end{proposition}

\begin{proof}
Assume $\sup_{S\times A}
\abs{r_1-r_2}\leq \eps$ for some $\eps\geq 0$.

For any policy $\pi$, let $Q_1^\pi$ be its $Q$-function for reward $r_1$,
and likewise for $r_2$.
Let $\pi_1$ and $\pi_2$ be optimal policies for $r_1$ and $r_2$,
respectively. Then for any $(s,a)\in S\times A$,
\begin{align}
Q_1^\star(s,a)&=Q_1^{\pi_1}(s,a)
\\&\geq Q_1^{\pi_2}(s,a)
\\&=
\sum_{t\geq 0} \gamma^t
\E\left[r_1(s_t,a_t) \mid \pi_2,\,(s_0,a_0)=(s,a)\right]
\\&\geq
\sum_{t\geq 0} \gamma^t
\E\left[r_2(s_t,a_t) -\eps \mid \pi_2,\,(s_0,a_0)=(s,a)\right]
\\&=
\sum_{t\geq 0} \gamma^t
\E\left[r_2(s_t,a_t) \mid \pi_2,\,(s_0,a_0)=(s,a)\right]
-\frac{\eps}{1-\gamma}
\\&=Q^\star_2(s,a)-\frac{\eps}{1-\gamma}
\end{align}
and likewise in the other direction, which ends the proof for
$Q$-functions. The case of $V$-functions follows by restricting to the
optimal actions at each state $s$.

Now, let $\pi$ a policy. We have
\begin{align}
\abs{ Q^\pi_1(s, a) - Q^\pi_2(s, a)} & = \abs{\sum_{t\geq 0} \gamma^t
\E\left[r_1(s_t,a_t) \mid \pi,\,(s_0,a_0)=(s,a)\right] - \sum_{t\geq 0} \gamma^t
\E\left[r_2(s_t,a_t) \mid \pi,\,(s_0,a_0)=(s,a)\right] } \\
& \leq \sum_{t\geq 0} \gamma^t  \E\left[ \abs{ r_2(s_t,a_t) - r_1(s_t, a_t) } \mid \pi,\,(s_0,a_0)=(s,a)\right]  \\
& \leq \frac{1}{1-\gamma} \sup_{S\times A}
\abs{r_1-r_2}.
\end{align}
The case of $V$-functions follows by taking the expectation over actions according to $\pi$.
\end{proof}

\begin{proposition}
\label{prop:pifgap}
Let $f\from S\times A\to \R$ be any function, and define a policy $\pi_f$
by $\pi_f(s)\deq \argmax_a f(s,a)$. Let $r\from S\times A\to \R$ be some
bounded reward function. Let $Q^\star$ be its optimal $Q$-function, and let $Q^{\pi_f}$ be the $Q$-function of $\pi_f$ for
reward $r$.

Then
\begin{equation}
\sup_{S\times A} \abs{f-Q^\star}\leq \frac{2}{1-\gamma}
\sup_{S\times A} \abs{f-Q^{\pi_f}} 
\end{equation}
and
\begin{equation}
\sup_{S\times A} \abs{Q^{\pi_f}-Q^\star}\leq \frac{3}{1-\gamma}
\sup_{S\times A} \abs{f-Q^{\pi_f}} 
\end{equation}
\end{proposition}

\NDY{Maybe comment on: The classical result is with $\abs{f-Tf}$ on the right for the
first claim, and $\abs{f-Q^\star}$ on the right for the second claim. 
If we follow the classical argument and prove the second claim via the first,  we get $1/(1-\gamma)^2$... Maybe this
is somewhere in Bertsekas, but it was faster to rewrite a proof...}

\begin{proof}
Define $\eps(s,a)\deq Q^{\pi_f}(s,a)-f(s,a)$.

The $Q$-function $Q^{\pi_f}$ satisfies the Bellman equation
\begin{equation}
Q^{\pi_f}(s,a)=r(s,a)+\gamma \E_{s'|(s,a)} Q^{\pi_f}(s',\pi_f(s'))
\end{equation}
for any $(s,a)\in S\times A$. Substituting $Q^{\pi_f}=f+\eps$, this rewrites as
\begin{align}
f(s,a)&=r(s,a)-\eps(s,a)+\gamma \E_{s'|(s,a)}
\left[f(s',\pi_f(s'))+\eps(s',\pi_f(s'))\right]
\\&=r(s,a)-\eps'(s,a)+\gamma \E_{s'|(s,a)} f(s',\pi_f(s'))
\\&=r(s,a)-\eps'(s,a)+\gamma \E_{s'|(s,a)} \max_{a'} f(s',a')
\label{eq:bellmanf}
\end{align}
by definition of $\pi_f$,
where we have set
\begin{equation}
\eps'(s,a)\deq \eps(s,a)-\gamma \E_{s'|(s,a)} \eps(s',\pi_f(s')).
\end{equation}

\eqref{eq:bellmanf} is the optimal Bellman equation for the reward $r-\eps'$.
Therefore, $f$ is the optimal $Q$-function for the reward $r-\eps'$.
Since $Q^\star$ is the optimal $Q$-function for reward $r$, by
Proposition~\ref{prop:rtoQ}, we have
\begin{equation}
\sup_{S\times A}\abs{f-Q^\star}\leq \frac{1}{1-\gamma} \sup_{S\times A} \abs{\eps'}
\end{equation}

By construction of $\eps'$, $\sup_{S\times A} \abs{\eps'}\leq
2\sup_{S\times A}\abs{\eps}=2\sup_{S\times A}
\abs{f-Q^{\pi_f}}$. This proves the first claim.

The second claim follows by the triangle inequality
$\abs{Q^{\pi_f}-Q^\star}\leq \abs{Q^{\pi_f}-f}+\abs{f-Q^\star}$ and
$\frac{2}{1-\gamma}+1\leq \frac{3}{1-\gamma}$.
\end{proof}

\begin{proof}[Proof of Theorem~\ref{thm:approx}]
By construction of the Kantorovich--Rubinstein norm, the second claim of
Theorem~\ref{thm:approx} is a particular case of the third claim, with $\norm{f}_A\deq
\max(\norm{f}_\infty,\lipnorm{f})$ and $\norm{\mu}_B\deq \KRnorm{\mu}$.

Likewise, since $m$ is the density of $M$ with respect to $\rho$, the first claim is an instance of the third, by taking
$\norm{f}_A\deq \norm{f}_\infty$ and $\norm{\mu}_B\deq \norm{\frac{\d
\mu}{\d \rho}}_{L^1(\rho)}$.
Therefore, we only prove the third claim.

Let $z\in Z$ and let $r\from S\times A\to \R$ be any reward function. By
Proposition~\ref{prop:Qsucc} with $G=S\times A$ and $\phi=\Id$, the $Q$-function of policy $Q^{\pi_z}$ for
this reward is
\begin{equation}
Q^{\pi_z}(s,a)=\int r(s',a') \,M^{\pi_z}(s,a,\d s',\d a').
\end{equation}

Let $\eps_z(s,a,\d s',\d a')$ be the difference of measures between the model
$\transp{F}B\rho$ and $M^{\pi_z}$:
\begin{align}
\eps_z(s,a,\d s',\d a')&\deq 
M^{\pi_z}(s,a,\d s',\d a')
-
\hat M^z(s,a,\d s',\d a')
\\&=
M^{\pi_z}(s,a,\d s',\d a')-
\transp{F(s,a,z)}B(s',a')\rho(\d s',\d
a').
\end{align}
We want to control the optimality gap in terms of
$\sup_{s,a}\norm{\eps_z(s,a,\cdot)}_B$.

By definition of $\eps_z$, 
\begin{align}
Q^{\pi_z}(s,a)&=\int r(s',a')\transp{F(s,a,z)}B(s',a')\rho(\d s',\d
a')
+\int r(s',a') \,\eps_z(s,a,\d s',\d a')
\\&=\transp{F(s,a,z)}z_R+\int r(s',a') \,\eps_z(s,a,\d s',\d a')
\end{align}
since $z_R=\int r(s',a')B(s',a')\rho(\d s',\d
a')$. Therefore,
\begin{align}
\abs{Q^{\pi_z}(s,a)-\transp{F(s,a,z)}z_R}
&= \abs{\int r(s',a')
\,\eps_z(s,a,\d s',\d a')}
\\&\leq \norm{r}_A \norm{\eps_z(s,a,\cdot)}_B
\end{align}
for any reward $r$ and any $z\in Z$ (not necessarily $z=z_R$).

Let $Q^\star$ be the optimal $Q$-function for reward $r$.
Define $f(s,a)\deq \transp{F(s,a,z_R)}z_R$. By definition,
the policy $\pi_{z_R}$ is equal to $\argmax_a f(s,a)$. Therefore, by
Proposition~\ref{prop:pifgap},
\begin{equation}
\sup_{S\times A} \abs{Q^{\pi_{z_R}}-Q^{\star}}\leq \frac{3}{1-\gamma}
\sup_{S\times A} \abs{f-Q^{\pi_{z_R}}}.
\end{equation}
and
\begin{equation}
\sup_{S\times A} \abs{f-Q^{\star}}\leq \frac{2}{1-\gamma}
\sup_{S\times A} \abs{f-Q^{\pi_{z_R}}}.
\end{equation}
But by the above, 
\begin{align}
\sup_{S\times A} \abs{f-Q^{\pi_{z_R}}} 
&=\sup_{S\times A} \abs{\transp{F(s,a,z_R)}z_R-Q^{\pi_{z_R}}(s,a)} 
\\&\leq \norm{r}_A \sup_{S\times A}
\norm{\eps_{z_R}(s,a,\cdot)}_B.
\end{align}

Therefore, for any reward function $r$, 
\begin{equation}
\sup_{S\times A} \abs{Q^{\pi_{z_R}}-Q^\star} \leq
\frac{3\norm{r}_A}{1-\gamma} \sup_{S\times A}
\norm{\eps_{z_R}(s,a,\cdot)}_B.
\end{equation}
This inequality transfers to the value functions, hence the result.
In addition, using again $f(s,a)=\transp{F(s,a,z_R)}z_R$, we obtain
\begin{equation}
\sup_{S\times A} \abs{\transp{F(s,a,z_R)}z_R-Q^\star(s,a)} \leq
\frac{2\norm{r}_A}{1-\gamma} \sup_{S\times A}
\norm{\eps_{z_R}(s,a,\cdot)}_B.
\end{equation}
\end{proof}

\begin{proposition}[Pointwise optimality gap]
\label{prop:pointwisegap}
Assume the state space is finite, and view rewards and $Q$-functions as
vectors over state-actions.

Let $r_1$ and $r_2$ be two reward functions, and let $\pi_1$ and $\pi_2$
be optimal policies for $r_1$ and $r_2$ respectively. Let $P_1$ and $P_2$
be the stochastic transition matrices over state-actions induced by
$\pi_1$ and $\pi_2$.

Then the optimality gap of policy $\pi_2$ on reward $r_1$ is at most
\begin{equation}
0\leq Q^{\pi_1}_{r_1}-Q^{\pi_2}_{r_1} \leq \left(
(\Id-\gamma P_1)^{-1}-(\Id-\gamma P_2)^{-1}
\right)(r_1-r_2)
\end{equation}
where the equality holds componentwise viewing the $Q$-functions as
vectors over state-actions.
\end{proposition}

\begin{proof}
This is a classical result. The inequality $0\leq
Q^{\pi_1}_{r_1}-Q^{\pi_2}_{r_1}$ is trivial since $\pi_1$ is optimal for
$r_1$ and the optimal policy is optimal at every state-action.
Denote $M_1\deq (\Id-\gamma
P_1)^{-1}$ and likewise for $M_2$. Then for any reward function $r$ one
has $Q^{\pi_1}_r=M_1r$ and likewise for $M_2$. Therefore
\begin{align}
Q^{\pi_1}_{r_1}-Q^{\pi_2}_{r_1}
&=M_1 r_1-M_2 r_1
\\&= M_1 r_1 -M_1 r_2 + M_1 r_2-M_2 r_1
\\& \leq M_1 r_1 -M_1 r_2 +M_2 r_2 -M_2 r_1 \qquad\text{since $\pi_2$ is
optimal for $r_2$}
\\&=(M_1-M_2)(r_1-r_2).
\end{align}
\end{proof}

\begin{proof}[Proof of Theorem~\ref{thm:pointwiseapprox}]
Let $f\from S\times A\to \R$ be any function. Define the policy
$\pi_f(s)\deq \argmax_a f(s,a)$. Define $r'\deq (\Id-\gamma P_{\pi_f})f$.
The equality $f=r'+\gamma P_{\pi_f}f$ can be rewritten as
\begin{equation}
f(s,a)=r'(s,a)+\gamma \E_{s'\sim P(\d s'|s,a)} f(s',\pi_f(s')).
\end{equation}
But by definition of $\pi_f$, $\pi_f(s')=\argmax_{a'} f(s',a')$.
Therefore
\begin{equation}
f(s,a)=r'(s,a)+\gamma \E_{s'\sim P(\d s'|s,a)} \max_{a'} f(s',a')
\end{equation}
namely, $f$ is the optimal $Q$-function for reward $r'$, with $\pi_f$ the
corresponding optimal policy.

Let $r$ be any reward function and let
$Q^{\pi_f}$ be the $Q$-function of $\pi_f$ for $r$. By definition, it
satisfies the Bellman equation $Q^{\pi_f}=r+\gamma P_{\pi_f} Q^{\pi_f}$,
namely,
$r=(\Id-\gamma P_{\pi_f})Q^{\pi_f}$.
Therefore,
\begin{equation}
r- r' = (\Id-\gamma P_{\pi_f})(Q^{\pi_f}-f).
\end{equation}

For any policy $\pi$, denote $M_\pi\deq (\Id-\gamma P_\pi)^{-1}$. ($M_\pi$ is the successor measure $M^\pi$ seen
as a matrix.)
The $Q$-function of $\pi$ for reward $r$ is $M_\pi r$.

Since $\pi^\star$ is optimal for $r$, and $\pi_f$ is optimal for $r'$,
Proposition~\ref{prop:pointwisegap} yields
\begin{align}
0\leq Q^\star-Q^{\pi_f} &\leq
(M_{\pi^\star}-M_{\pi_f})(r-r')
\\
&= (M_{\pi^\star}-M_{\pi_f})(\Id-\gamma P_{\pi_f})(Q^{\pi_f}-f).
\end{align}
Since $M_{\pi_f}$ is the inverse of $\Id-\gamma P_{\pi_f}$ and likewise
for $\pi^\star$, we have
\begin{align}
(M_{\pi^\star}-M_{\pi_f})(\Id-\gamma P_{\pi_f})
&=M_{\pi^\star}(\Id-\gamma P_{\pi_f})-\Id
\\
&=M_{\pi^\star}(\Id-\gamma P_{\pi_f}-(\Id-\gamma P_{\pi^\star}))
\\&=\gamma M_{\pi^\star}(P_{\pi^\star}-P_{\pi_f})
\end{align}
and therefore
\begin{equation}
0\leq Q^\star-Q^{\pi_f} \leq \gamma M_{\pi^\star} (P_{\pi^\star}-P_{\pi_f})
(Q^{\pi_f}-f).
\end{equation}

Now set
\begin{equation}
f(s,a)\deq \transp{F(s,a,z_R)}z_R
\end{equation}
or in matrix notation, $f=\transp{F(z_R)}z_R$.
Then $\pi_f=\pi_{z_R}$ by definition. Thus
$Q^{\pi_f}=M_{\pi_{z_R}}r=m^{\pi_{z_R}}\diag(\rho)r$ in matrix
notation. Moreover, $z_R=\E_\rho [B(s,a)r(s,a)]=B\diag(\rho)r$ in matrix
notation. So $f=\transp{F(z_R)}B\diag(\rho)r$.
Therefore, 
\begin{equation}
Q^{\pi_f}-f=(m^{\pi_{z_R}}-\transp{F(z_R)}B)\diag(\rho)r
\end{equation}
and
\begin{equation}
Q^\star-Q^{\pi_f} \leq \gamma M^{\star}
(P_{\pi^\star}-P_{\pi_f})\left(m^{\pi_{z_R}}-\transp{F(z_R)}B\right)\diag(\rho)r
\end{equation}
and using $M_{\pi^\star}=\sum_{t\geq 0} \gamma^t P^t_{\pi_\star}$ provides the
required inequality.

As a remark, the same proof works on continuous state spaces, by viewing
all matrices as linear operators over functions on $S\times A$.
\end{proof}

\begin{proof}[Proof of Proposition~\ref{prop:approxr}]
This is just a triangle inequality.

\begin{align}
\| V_r^{\pi_{\hat z_R}} - V^\star_r \|_\infty 
& \leq
\| V_{r}^{\pi_{\hat z_R}} - V_{\hat{r}}^{\pi_{\hat z_R}} \|_\infty
+
\| V_{\hat{r}}^{\pi_{\hat z_R}} - V^\star_{\hat{r}} \|_\infty 
+ \| V_{\hat{r}}^{\star} - V^\star_{r} \|_\infty  \\
& \leq \frac{\| r - \hat{r}\|_\infty}{1-\gamma} + \epsilon_{FB} +
\frac{\| r - \hat{r}\|_\infty}{1-\gamma}.
\end{align}
The last inequality follows from two facts: by assumption, the difference between the value function of $\pi_{\hat
z_R}$ and the optimal value function $V^\star_{\hat r}$ is at most
$\eps_{FB}$, then by Proposition~\ref{prop:rtoQ}, the difference between
$V^\star_{\hat r}$ and $V^\star_r$ as well as the difference between
$V_{r}^{\pi_{\hat z_R}}$ and $V_{\hat{r}}^{\pi_{\hat z_R}}$ are bounded by
$\frac{1}{1-\gamma}\sup_{S\times A} \abs{\hat r-r}$. 
\end{proof}

\begin{proof}[Proof of Theorem~\ref{thm:approxzR}]
We proceed by building a reward function $\hat r$ corresponding to $\hat
z_R$. Then we will bound $\hat r-r$ and apply
Proposition~\ref{prop:approxr}.

First, by Remark~\ref{rem:normB}, up to reducing $d$, we can assume that $B$ is
$L^2(\rho)$-orthonormal.

For any function $\phi\from (s,a)\to \R$, define $z_\phi\deq
\E_{(s,a)\sim \rho} [\phi(s,a)B(s,a)]$. 
For each $z\in Z$, define $\phi_z$ via $\phi_z(s,a)\deq
\transp{B(s,a)}z$. Then, if $B$ is $L^2(\rho)$-orthonormal, we have
$z_{\phi_z}=z$. (Indeed,
$z_{\phi_z}=\E[(\transp{B(s,a)}z)B(s,a)]=\E[B(s,a)(\transp{B(s,a)}z)]=\left(
\E [B(s,a)\transp{B(s,a)}]\right)z$.)

Define the function
\begin{equation}
\hat r\deq r+\phi_{\hat z_R-z_R}
\end{equation}
using the functions $\phi_z$ defined above. By construction, $z_{\hat
r}=z_R+z_{\phi_{\hat z_R-z_R}}=\hat z_R$.
Therefore, the policy
$\pi_{\hat z_R}$ associated to $\hat z_R$ is the policy associated to the
reward $\hat r$.

We will now apply Proposition~\ref{prop:approxr} to $r$ and $\hat r$.
For this, we need to bound $\norm{\phi_{\hat z_R-z_R}}_\infty$.

Let $B_1,B_2,\ldots,B_d$ be the components of $B$ as functions on
$S\times A$.
For any $z\in Z$, we have
\begin{equation}
\norm{\phi_z}^2_{L^2(\rho)}=\norm{\sum_i z_i B_i}^2_{L^2(\rho)}=\sum_i
z_i^2=\norm{z}^2
\end{equation}
since the $B_i$ are $L^2(\rho)$-orthonormal. Moreover, by construction,
$\phi_z$ lies in the linear span $\langle B\rangle$ of the functions
$(B_i)$. Therefore
\begin{equation}
\norm{\phi_z}_\infty \leq \zeta(B)
\norm{\phi_z}_{L^2(\rho)}=\zeta(B)\norm{z}
\end{equation}
by the definition of $\zeta(B)$ (Definition~\ref{def:skew}).

Therefore,
\begin{equation}
\norm{\phi_{\hat z_R-z_R}}_\infty\leq \zeta(B) \norm{\hat z_R-z_R}.
\end{equation}

Let us now bound
$\hat z_R-z_R$:
\begin{align}
\E\left[\norm{\hat z_R-z_R}^2\right]&=
\E\left[\E\left[\norm{\hat z_R-z_R}^2\mid (s_i,a_i)\right]\right]
\\&=
\E\left[\E\left[\norm{\hat z_R-\E[\hat z_R \mid
(s_i,a_i)]}^2+\norm{\E[\hat z_R \mid (s_i,a_i)]-z_R}^2\mid (s_i,a_i)\right]\right]
\\&=
\E\left[\E\left[\norm{\frac1N \sum_i (\hat r_i-r(s_i,a_i))B(s_i,a_i)}^2
+\norm{\frac1N \sum_i r(s_i,a_i)B(s_i,a_i)-z_R}^2\mid
(s_i,a_i)\right]\right]
\end{align}

The first term satisfies
\begin{align}
\E\left[\norm{\frac1N \sum_i (\hat r_i-r(s_i,a_i))B(s_i,a_i)}^2 \mid
(s_i,a_i)\right]
&=\frac{1}{N^2} \sum_i \E\left[
(\hat r_i-r(s_i,a_i))^2 \norm{B(s_i,a_i)}^2
\right]
\\&\leq \frac1{N^2} \sum_i v \norm{B(s_i,a_i)}^2
\end{align}
because the $\hat r_i$ are independent conditionally to $(s_i,a_i)$, and
because $B$ is deterministic. The expectation of this over $(s_i,a_i)$ is
\begin{equation}
\E\left[
\frac1{N^2} \sum_i v \norm{B(s_i,a_i)}^2\right]=\frac{v}{N} \E_{(s,a)\sim \rho}
\norm{B(s,a)}^2=\frac{v}{N} \norm{B}^2_{L^2(\rho)}
\end{equation}
which is thus a bound on the first term.

The second term satisfies
\begin{equation}
\E\left[
\norm{\frac1N \sum_i r(s_i,a_i)B(s_i,a_i)-z_R}^2\right]
=\frac{1}{N} \norm{r(s,a)B(s,a)-\E_\rho [r(s,a)B(s,a)]}^2_{L^2(\rho)}
\end{equation}
since the $(s_i,a_i)$ are independent with distribution $\rho$. By the
Cauchy--Schwarz inequality (applied to each component of $B$), this is at
most
\begin{equation}
\frac1N \norm{r(s,a)-\E_{\rho} r}^2_{L^2(\rho)} \norm{B}^2_{L^2(\rho)}.
\end{equation}

Therefore,
\begin{equation}
\E\left[\norm{\hat z_R-z_R}^2\right]\leq \left(v+\norm{r(s,a)-\E_{\rho}
r}^2_{L^2(\rho)}\right)\frac{\norm{B}^2_{L^2(\rho)}}{N}.
\end{equation}

Since $B$ is orthonormal in $L^2(\rho)$, we have
$\norm{B}^2_{L^2(\rho)}=d$. Putting everything together, we find
\begin{equation}
\E\left[
\norm{\hat r-r}_\infty^2
\right]
\leq \frac{\zeta(B)\,d}{N} \left(v+\norm{r(s,a)-\E_{\rho}
r}^2_{L^2(\rho)}\right).
\end{equation}

Therefore, by the Markov inequality, for any $\delta>0$, with probability $1-\delta$,
\begin{equation}
\norm{\hat r-r}_\infty\leq \sqrt{
\frac{\zeta(B)\,d}{N\delta} \left(v+\norm{r(s,a)-\E_{\rho}
r}^2_{L^2(\rho)}\right)
}
\end{equation}
hence the conclusion by Proposition~\ref{prop:approxr}.
\end{proof}

\begin{proof}[Proof of Theorem~\ref{thm:forwardbackwardinterpretation}]
Let
\begin{equation}
m(s,a,s',a')\deq \sum_{t\geq 0} \gamma^t
\frac{P_t(\d s',\d a'|s,a,\pi_z)}{\rho(\d s',\d a')}
\end{equation}
so that
\begin{equation}
\ell(F,B)=\int \abs{\transp{F(s,a,z)}B(s',a')-m(s,a,s',a')}^2 \rho(\d
s,\d a)\rho(\d s',\d a').
\end{equation}

Let us first take the derivative with respect to the parameters of $F$.
This is $0$ by assumption, so we find
\begin{align}
0&=\int \partial_\theta \transp{F(s,a,z)}B(s',a')
\left(\transp{F(s,a,z)}B(s',a')-m(s,a,s',a')\right)\rho(\d
s,\d a)\rho(\d s',\d a')
\\&=\int \partial_\theta \transp{F(s,a,z)} G(s,a) \rho(\d s,\d a)
\end{align}
where
\begin{equation}
G(s,a)\deq \int B(s',a')
\left(\transp{F(s,a,z)}B(s',a')-m(s,a,s',a')\right)\rho(\d s',\d a')
\end{equation}

Since the model is overparameterized, we can realize any $L^2$ function
$f(s,a)$ as the derivative $\partial_\theta F(s,a,z)$ for some direction
$\theta$. Therefore, 
the equation $0=\int \partial_\theta \transp{F(s,a,z)}G(s,a)\rho(\d s,\d
a)$ implies that
$G(s,a)$ is $L^2(\rho)$-orthogonal to any function
$f(s,a)$
in $L^2(\rho)$. Therefore, $G(s,a)$ vanishes $\rho$-almost everywhere,
namely
\begin{equation}
\int B(s',a')
\transp{F(s,a,z)}B(s',a')\rho(\d s',\d a')
=
\int B(s',a') m(s,a,s',a')\rho(\d s',\d a')
\end{equation}

Now, since $\transp{F(s,a,z)}B(s',a')$ is a real number,
$\transp{F(s,a,z)}B(s',a')=\transp{B(s',a')}F(s,a,z)$. Therefore, the
right-hand-side above rewrites as
\begin{equation}
\int B(s',a') \transp{B(s',a')}F(s,a,z) \rho(\d s',\d a')=(\Cov
B)F(s,a,z)
\end{equation}
so that
\begin{equation}
(\Cov B)F(s,a,z)=\int B(s',a') m(s,a,s',a')\rho(\d s',\d a').
\end{equation}

Unfolding the definition of $m$ yields the statement for $F$. The proof
for $B$ is similar.
\end{proof}

\begin{proof}[Proof of Theorem~\ref{thm:fixedbackward}]

According to the proof of Theorem~\ref{thm:forwardbackwardinterpretation}, if $F$ achieves a local extremum of $\ell(F, B)$ given a fixed $B$, we have

\begin{equation}
(\Cov B)F(s,a,z)=\int B(s',a') m^{\pi_z}(s,a,s',a')\rho(\d s',\d a')
\end{equation}
where $m^{\pi_z}(s,a,s',a') = \sum_{t\geq 0} \gamma^t
\frac{P_t(\d s',\d a'|s,a,\pi_z)}{\rho(\d s',\d a')}$ is the successor state density induced by the policy $\pi_z$.

Let $r\from S \times A \rightarrow \R$ a bounded reward function that
lies in the span of $B$ i.e there exists $\omega \in \R^d$ such that
$r(s, a) = B(s, a)^\top \omega$ for any state-action pair $(s, a)$. This
implies that $(\Cov B) \omega =  \E_{(s,a)\sim \rho} [
B(s,a)\transp{B(s,a)}\omega]=\E_{(s, a) \sim \rho}[ B(s, a) r(s, a) ] = z_R$,  by definition of $z_R$.

Therefore,
\begin{align}
F(s, a, z_R)^\top z_R & = F(s, a, z_R)^\top (\Cov B) \omega  \\
& = \left( (\Cov B) F(s, a, z_R) \right)^\top \omega \\
& = \left(  \int B(s',a')^\top   m^{\pi_{z_R}}(s,a,s',a')\rho(\d s',\d a') \right) \omega \\
& = \int B(s',a')^\top   \omega~m^{\pi_{z_R}}(s,a,s',a') \rho(\d s',\d a') \\
 & = \int r(s', a') m^{\pi_{z_R}}(s,a,s',a')\rho(\d s',\d a') \\
 & = Q^{\pi_{z_R}} (s, a)
\end{align}
Therefore, $\pi_{z_R}$ is the greedy policy with respect to its own Q-value. We conclude that $\pi_{z_R}$ is optimal for $r$.

Now, let $r\from  S \times A \rightarrow \R$ be an arbitrary bounded
reward function, and let $r_B$ be the $L^2(\rho)$-projection of $r$ onto the span of $B$. Both $r$ and $r_B$ share the same $z_R = \E_{(s, a) \sim \rho}[r(s, a) B(s, a)] = \E_{(s, a) \sim \rho}[r_B(s, a) B(s, a)]$.  According to the first part of our proof, $\pi_{z_R}$ is optimal for $r_B$ since $r_B$ lies in the span of $B$. 
Denote by the subscript $r$ in $V_r$ the reward function that a value
function corresponds to. By Proposition~\ref{prop:pointwisegap},
\begin{align}
0\leq V_r^\star - V_r^{\pi_{z_R}} & \leq 
(\Id-\gamma
P_{\pi^\star})^{-1}(r -
r_B)-
(\Id-\gamma
P_{\pi_{z_R}})^{-1}(r -
r_B)
\end{align}
hence, taking norms,
\begin{equation}
\norm{V^{\pi_{z_R}} - V^\star}_\infty \leq \norm{(\Id-\gamma
P_{\pi^\star})^{-1}(r -
r_B)}_\infty
+\norm{(\Id-\gamma
P_{\pi_{z_R}})^{-1}(r -
r_B)}_\infty
\end{equation}
as needed.

The bound with $\frac{2}{1-\gamma}$ follows by noting that
$(\Id-\gamma P_{\pi})^{-1}$ is bounded by
$\frac{1}{1-\gamma}$ in $L^\infty$ norm for any policy $\pi$.
\end{proof}

\begin{proof}[Proof of Proposition~\ref{prop:linearz}]
From the definition \eqref{eq:hatzR} of $\hat z_R$ and the definition
\eqref{eq:linearhatr} of $\hat r$, we find
\begin{align}
\hat z_R &=\E_\rho [B(s,a)\hat r(s,a)]
\\&=\E_\rho [B(s,a) \transp{B(s,a)}w]
\\&=(\Cov_\rho B)w
\end{align}
hence the result given the expression \eqref{eq:linearhatr} for $w$.

If $\rho=\rhotest$, then the covariances cancel out: we find $\hat
z_R=(\Cov_\rho B)w=(\Cov_\rho B)(\Cov_{\rhotest}
B)^{-1}\E_{\rhotest}[r(s,a)B(s,a)]=\E_{\rhotest}[r(s,a)B(s,a)]=\E_{\rho}[r(s,a)B(s,a)]=z_R$.

If $r$ is linear in $B$, then the linear regression model does not
depend on the data distribution $\rhotest$ used: if $r(s,a)=\transp{B(s,a)}w_0$ then
$w=w_0$ for any $\rhotest$, as long as
$\Cov_{\rhotest} B$ is invertible.
In that case, both $z_R$ and $\hat z_R$ are equal to $(\Cov_\rho B)w_0$.
\end{proof}

\newpage
\section{Experimental Setup}
\label{sec:setup}

In this section we provide additional information about our experiments. 

\subsection{Environments}
\label{sec:env}
\begin{itemize}
\item \textbf{Discrete maze:} is the $11\times11$ classical tabular gridworld with foor rooms. States are represented by one-hot unit vectors, $S = \{0, 1\}^{121}$.  There are five available actions , $A = \{$\texttt{left}, \texttt{right}, \texttt{up}, \texttt{down}, \texttt{do nothing} $\}$. The dynamics are deterministic and the walls are impassable. 
\item \textbf{Continuous maze:}  is a two dimensional environment with impassable walls. States are represented by their Cartesian coordinates $(x,y) \in  S=[0,1]^2$. There are five available actions, $A = \{$\texttt{left}, \texttt{right}, \texttt{up}, \texttt{down}, \texttt{do nothing} $\}$. The execution of one of the actions moves the agent $0.1$ units in the desired direction, and normal random noise with zero mean and standard
deviation $0.01$ is added to the position of the agent (that is, a move
along the x axis would be $x' = x \pm 0.1 +  \mathcal{N} (0, 0.01)$, where
$  \mathcal{N} (0, 0.01)$ is a normal variable with mean 0 and standard
deviation $0.01$). If after a move the agent ends up outside of $[0,
1]^2$, the agent's position is clipped (e.g if $x < 0$ then we set
$x=0$). If a move make the agent cross an interior wall, this move is undone. For all algorithms, we convert a state $s=(x, y)$ into feature vector $\phi(s) \in \mathbb{R}^{441}$ by computing  the activations of a regular $21\times21$ grid of radial basis functions at the point $(s, y)$. Especially, we use Gaussian functions: $\phi(s) =\left (  \exp( - \frac{(x - x_i)^2 + (y - y_i)^2}{\sigma}), \ldots,  \exp( - \frac{(x - x_{441})^2 + (y - y_{441})^2}{ 2 \sigma^2}) \right)$  where $(x_i, y_i)$ is the center of the $i^{th}$ Gaussian and $\sigma=0.05$.
\item  \textbf{FeatchReach:} is a variant of the simulated robotic arm environment from~\cite{plappert2018multi} using discrete actions instead of continuous actions.  States are 10-dimensional vectors consisting of positions and velocities of robot joints. We discretise the original 3-dimensional action space into 6 possible actions using action stepsize of 1 (The same way as in \textcolor{blue}{https://github.com/paulorauber/hpg}, the implementation of hindsight policy gradient~\cite{rauber2018hindsight}). The goal space is 3-dimensional space representing of the position of the object to reach.

\item \textbf{Ms.\ Pacman:}  is a variant of the Atari 2600 game Ms.\ Pacman~\cite{bellemare2013arcade}, where an episode ends when the agent is captured by a monster~\cite{rauber2018hindsight}. 
States are obtained by processing the raw visual input directly from the screen. Frames are preprocessed by cropping, conversion to grayscale and downsampling to $84\times84$ pixels. A state $s_t$ is the concatenation of $(x_{t-12}, x_{t-8}, x_{t-4}, x_t)$ frames, i.e. an $84\times84\times4$ tensor.  
An action repeat of 12 is used. As Ms.\ Pacman is not originally a
multi-goal domain, we define the set of goals as the set of the 148
reachable coordinate pairs
$(x,y)$ on the screen; these can be reached only by learning to avoid monsters. In contrast with~\cite{rauber2018hindsight}, who use a heuristic to find the agent's position from the screen's pixels, we use the Atari annotated RAM interface wrapper~\cite{anand2019unsupervised}.
\end{itemize}

\subsection{Architectures}
\label{sec:archi}
We use the same architecture for discrete maze, continuous maze and
FeatchReach. Both forward and backward networks are represented by a
feedforward neural network with three hidden layers, each with 256 ReLU
units. The forward network receives a concatenation of a state and a $z$
vector as input and has $|A| \times d$ as output dimension. The backward
network receives a state as input (or gripper's position for FeatchReach)
and has $d$ as output dimension. For goal-oriented DQN, the $Q$-value
network is also a feedforward neural network with three hidden layers, each with 256 ReLU units. It receives a concatenation of a state and a goal as input and has $|A|$ as output dimension.

For Ms.\ Pacman, the forward network is represented by a convolutional neural network given by a convolutional layer with 32 filters ($8\times8$, stride 4); convolutional layer with 64 filters ($4 \times 4$, stride 2); convolutional layer with 64 filters ($3 \times 3$, stride 1); and three fully-connected layers, each with 256 units. We use ReLU as activation function. The $z$ vector is concatenated with the output of the third convolutional layer. The output dimension of the final linear layer is $|A| \times d$. 
The backward network acts only on agent's position, a 2-dimensional input. It is represented by
a feedforward neural network with three hidden layers, each with 256 ReLU
units. The output dimension is $d$. For goal-oriented DQN, the $Q$-value network is represented by a convolutional neural network with the same architecture as the one of the forward network. The goal's position is concatenated with the output of the third convolutional layer. The output dimension of the final linear layer is $|A|$.  

\subsection{Implementation Details}
\label{sec:implem}

For all environments, we run the algorithms for 800 epochs. Each epoch
consists of 25 cycles where we interleave between gathering some amount
of transitions, to add to the replay buffer $\mathcal{D}$ (old transitions are thrown when we reach the maximum of its size), and performing 40 steps of
stochastic gradient descent on the model parameters. To collect
transitions, we generate episodes using some behavior policy. For both
mazes, we use a uniform policy while for FetchReach and Ms.\ Pacman, we
use an $\epsilon$-greedy policy ($\eps=0.2$) with respect to the current approximation $F(s, a, z)^\top z$ for a sampled $z$. At evaluation time, $\eps$-greedy policies are also used, with a smaller
$\eps=0.02$ for all environments except from discrete maze where we use Boltzmann policy with temperature $\tau =1$. We train each model for three different random seeds.

For generality, we will keep using the notation $B(s, a)$ while in our
experiments $B$ acts only on $\phi(s, a)$, a part of the state-action
space. For discrete and continuous mazes, $\phi(s, a) = s$, for
FetchReach, $\phi(s, a)$ the position of arm's gripper and for Ms.\
Pacman, $\phi(s, a)$ is the $2$-dimensional position $(x,y)$ of the agent on the screen. 

We denote by $\theta$ and $\omega$ the parameters of forward and backward networks respectively and  $\theta^{-}$ and $\omega^{-}$ the parameters of their corresponding target networks. Both  $\theta^{-}$ and $\omega^{-}$ are updated after each cycle using Polyak averaging; i.e $\theta^{-} \leftarrow \alpha \theta^{-} + (1-\alpha) \theta$ and $\omega^{-} \leftarrow \alpha \omega^{-} + (1-\alpha) \omega$ where $\alpha=0.95$ is the Polyak coefficient.

During training, we sample $z$ from a rescaled Gaussian that we denote
$\nu$. With a pure Gaussian in large dimension, the norm of $z$ would be very concentrated
around a single
value. Instead, we first sample a $d$-dimensional standard Gaussian
variable $x \sim \mathcal{N}(0, \Id) \in \mathbb{R}^d $  and a scalar
centered Cauchy variable $u \in \mathbb{R}$ of scale $0.5$, then we set
$z = \sqrt{d}\, u\, \frac{x}{\norm{x}}$.
We use a Cauchy distribution to ensure that the norm of $z$ spans the
non-negative real numbers space while having a heavy tail. We also scale by $\sqrt{d}$ to ensure  that each component of $z$
has an order of magnitude of $1$.

Before being fed to $F$, $z$ is
preprocessed by $z\gets \frac{z}{ \sqrt{ 1 + \norm{z}_2^2 / d}}$; this way, $z$ ranges over a bounded set in $\R^d$, and this takes advantage of optimal policies being equal for a reward $R$ and for $\lambda R$ with $\lambda>0$.

To update network parameters, we compute an empirical loss by sampling 3 mini-batches, each of size $b=128$,  of transitions $\{(s_i, a_i, s_{i+1})\}_{i \in I} \subset \mathcal{D}$, of target state-action pairs  $\{(s'_i, a'_i)\}_{i \in I} \subset \mathcal{D}$ and of $\{z_i \}_{i \in I} \sim \nu$:
\begin{align}
 \mathscr{L} (\theta, \omega) & = \frac{1}{2 b^2} \sum_{i, j \in I^2} \left( F_\theta(s_i, a_i,z_i)^\top B_\omega(s'_j, a'_j)  - \gamma \sum_{a \in A} \pi_{z_i} (a \mid s_{i+1}) \cdot F_{\theta^{-}}(s_{i+1}, a ,z_i)^\top B_{\omega^{-}}(s'_j, a'_j)  \right)^2 \nonumber \\
 & \quad  - \frac{1}{b} \sum_{i \in I} F_\theta(s_i, a_i,z_i)^\top  B_\omega(s_i, a_i)  
\end{align}
where  we use the Boltzmann policy $\pi_{z_i}( \cdot \mid s_{i+1} ) = \texttt{softmax} (F_{\theta^{-}}(s_{i+1},
 \cdot, z_i)^\top z_i / \tau )$ with fixed temperature $\tau=200$ to avoid the instability and discontinuity caused by the argmax operator.
 
Since there is
unidentifiability between $F$ and $B$ (Appendix, Remark~\ref{rem:normB}),
we include a gradient to make $B$ closer to orthonormal, $\E_{(s,a)\sim \rho} B(s,a)\transp{B(s,a)}\approx \Id$: 
\begin{multline}
\frac{1}{4} \partial_\omega \norm{\E_{(s,a)\sim \rho}
B_\omega(s,a)\transp{B_{\omega}(s,a)} - \Id }^2 =\\ \E_{(s,a)\sim \rho, (s', a') \sim \rho} \partial_\omega B_\omega (s, a)^\top \left( B_\omega(s, a)^\top B_\omega(s', a') \cdot B_\omega (s', a') - B(s, a) \right)
\end{multline}
To compute an unbiased estimate of the latter gradient, we use the following auxiliary empirical loss: 

\begin{align}
\mathscr{L}_{\texttt{reg}} (\omega) & = \frac{1}{b^2} \sum_{i, j \in I^2} B_{\omega}(s_i, a_i)^\top \texttt{stop-gradient}(B_{\omega}(s'_j, a'_j) ) \cdot \texttt{stop-gradient}  (B_{\omega}(s_i, a_i)^\top B_{\omega}(s'_j, a'_j)) \nonumber \\
& \quad - 
  \frac{1}{b} \sum_{i \in I} B_{\omega}(s_i, a_i)^\top \texttt{stop-gradient}(B_{\omega}(s_i, a_i) )
\end{align}

Finally, we use the Adam optimizer and we update $\theta$ and $\omega$ by taking a gradient step on $\mathscr{L} (\theta, \omega)$ and $\mathscr{L} (\theta, \omega)  + \lambda \cdot \mathscr{L}_{\texttt{reg}} (\omega)$ respectively, where $\lambda$ is a regularization coefficient that we set to 1 for all experiments.

We summarize the hyperparameters used for FB algorithm and goal-oriented DQN in table~\ref{table: fb param} and~\ref{table: dqn param} respectively.

\begin{table*}[h!]
\vspace{5px} 
    \footnotesize\centering
    \begin{tabular}{l|llll}
    \toprule
    Hyperparameters & Discrete Maze & Continuous Maze & FetchReach & Ms.\ Pacman \\
    \midrule
    number of cycles per epoch & 25 & 25 & 25 & 25 \\
    number of episodes per cycles & 4 & 4 & 2 & 2 \\
    number of timesteps per episode & 50& 30& 50& 50 \\
    number of updates per cycle & 40 & 40 & 40 & 40 \\
    exploration $\eps$ & 1 & 1 & 0.2 & 0.2 \\ 
    evaluation $\eps$ & Boltzman with $\tau=1$ & 0.02 & 0.02 & 0.02 \\
    temperature $\tau$ & 200 & 200 & 200 & 200 \\
    learning rate & 0.001 & 0.0005 & 0.0005 & 0.0001 if $d=100$ else 0.0005 \\
    mini-batch size & 128 & 128 & 128 & 128 \\
    regularization coefficient $\lambda$ & 1 & 1 & 1 & 1 \\
    Polyak coefficient $\alpha$ & 0.95 & 0.95 & 0.95 & 0.95 \\
    discount factor $\gamma$ & 0.99 & 0.99 & 0.9 & 0.9 \\
    replay buffer size & $10^6$ & $10^6$ & $10^6$ & $10^6$ \\
    \bottomrule
     \end{tabular}
     \caption{\label{table: fb param} Hyperparameters of the FB algorithm}
\end{table*}

\begin{table*}[h!]
    \centering\footnotesize
    \begin{tabular}{l|llll}
    \toprule
    Hyperparameters & Discrete Maze & Continuous Maze & FetchReach & Ms.\ Pacman \\
    \midrule
    number of cycles per epoch & 25 & 25 & 25 & 25 \\
    number of episodes per cycles & 4 & 4 & 2 & 2 \\
    number of timesteps per episode & 50& 30& 50& 50 \\
    number of updates per cycle & 40 & 40 & 40 & 40 \\
    exploration $\eps$ & 0.2& 0.2 & 0.2 & 0.2 \\ 
    evaluation $\eps$ & Boltzman with $\tau=1$ & 0.02 & 0.02 & 0.02 \\
    learning rate & 0.001 & 0.0005 & 0.0005 & 0.0005 \\
    mini-batch size & 128 & 128 & 128 & 128 \\
    Polyak coefficient $\alpha$ & 0.95 & 0.95 & 0.95 & 0.95 \\
    discount factor $\gamma$ & 0.99 & 0.99 & 0.9 & 0.9 \\
    replay buffer size & $10^6$ & $10^6$ & $10^6$ & $10^6$ \\
    ratio of hindsight replay &- &- &- & 0.8 \\
    \bottomrule
     \end{tabular}
     \caption{\label{table: dqn param}  Hyperparameters of the goal-oriented DQN algorithm}
\end{table*}

\subsection{Experimental results}
\label{sec:addit_results}
In this section, we provide additional experimental results.
\subsubsection{Goal-Oriented Setup: Quantitative Comparisons}

\begin{figure}[h!]
    \centering
    \includegraphics[width=0.6\textwidth]{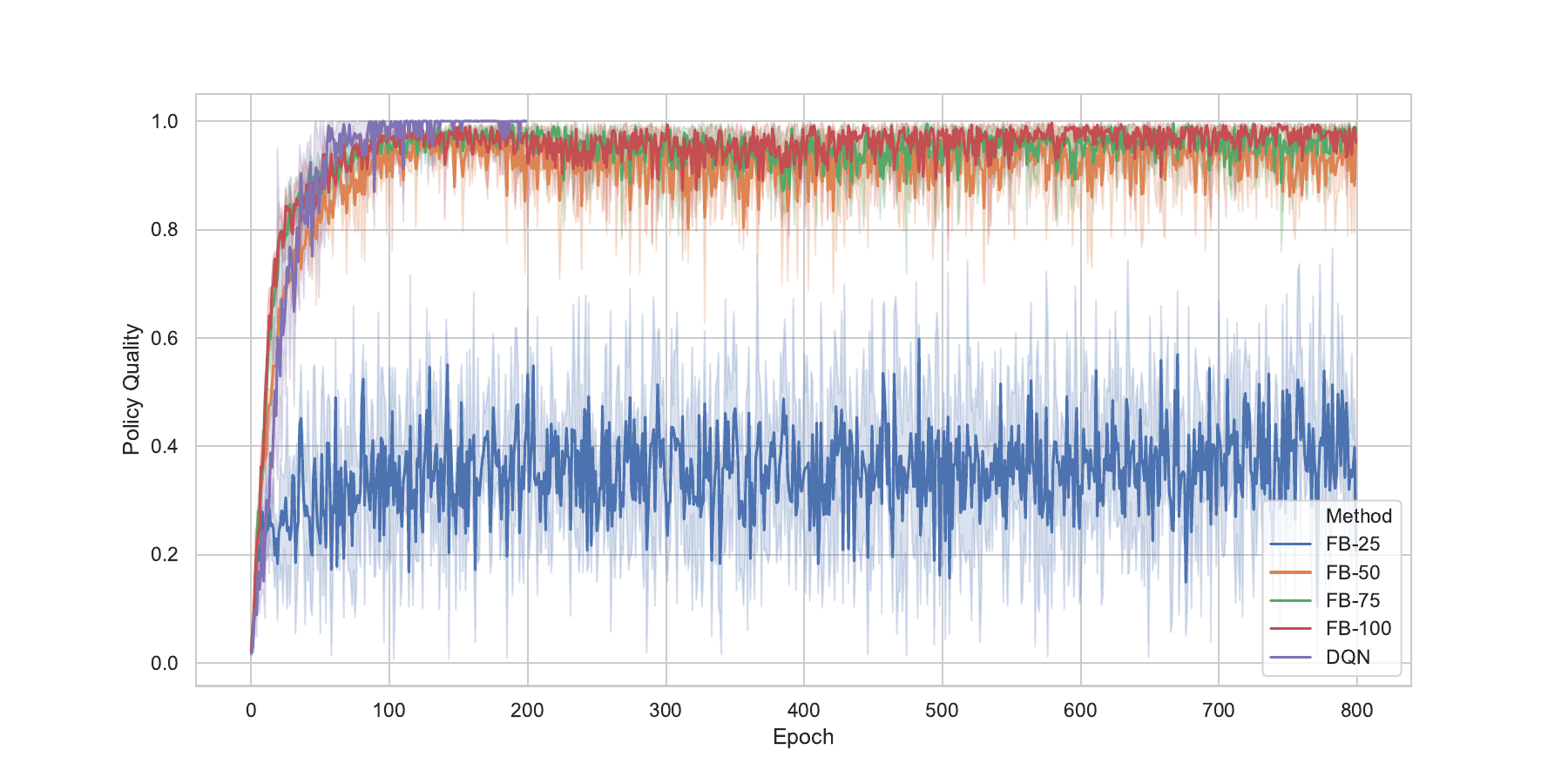} \hspace{10pt}
     \includegraphics[width=0.3\textwidth]{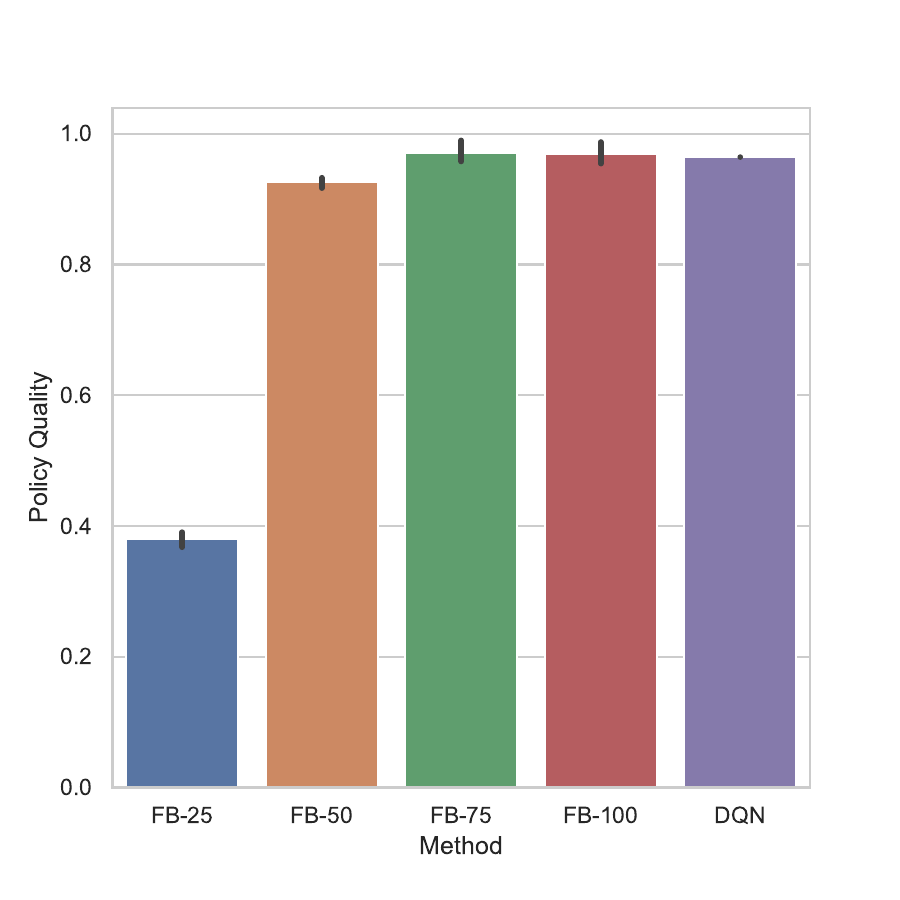} 
    \caption{\textbf{Discrete maze}: Comparative performance of FB for different dimensions and DQN. \textbf{Left}: the policy quality averaged over 20 randomly selected goals as function of the training epochs. \textbf{Right}: the policy quality averaged over the goal space after 800 training epochs.}
\label{fig:discretemazeresults}
\end{figure}

\begin{figure}[h!]
    \centering
    \includegraphics[width=0.6\textwidth]{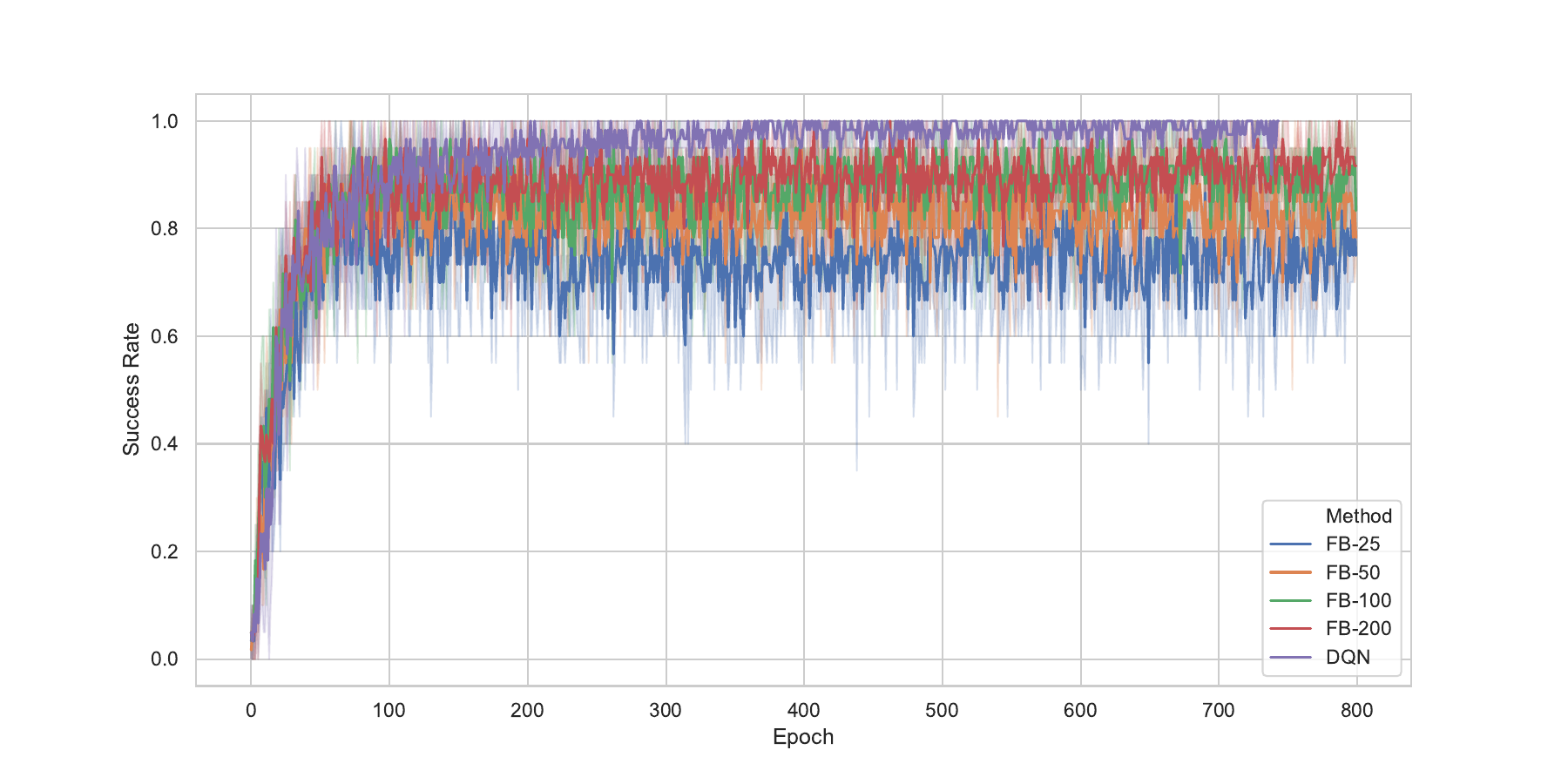} \hspace{10pt}
     \includegraphics[width=0.3\textwidth]{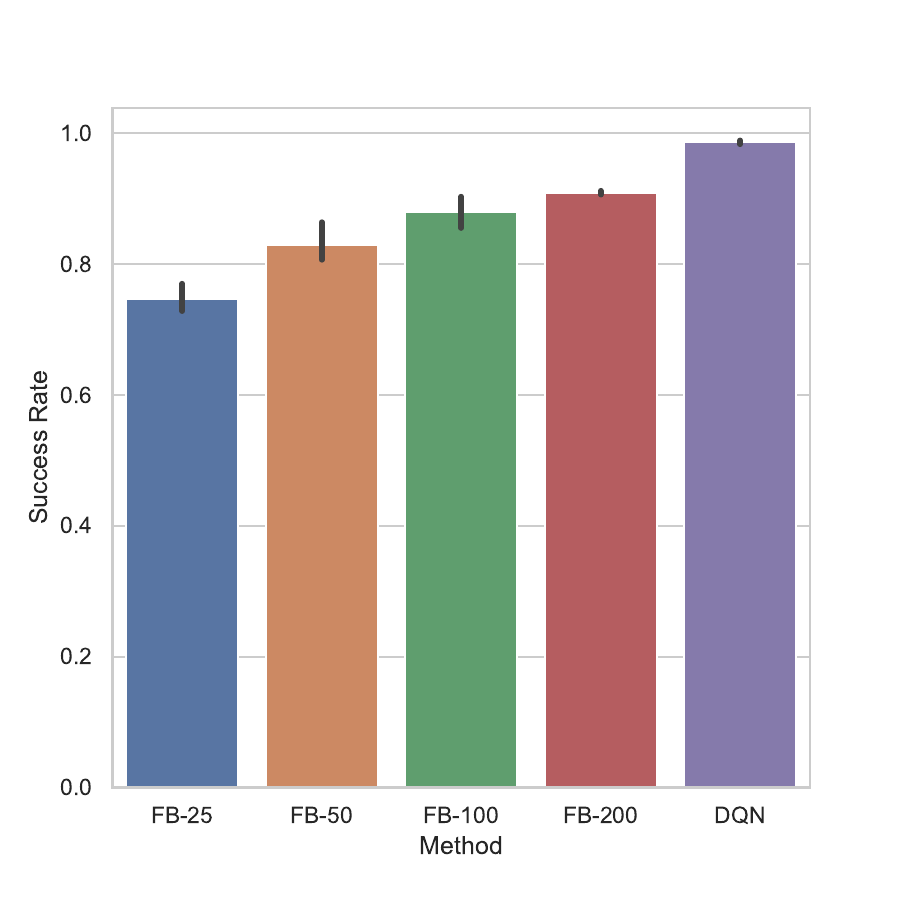} 
    \caption{\textbf{Continuous maze:} Comparative performance of FB for different dimensions and DQN. \textbf{Left}: the success rate averaged over 20 randomly selected goals as function of the training epochs. \textbf{Right}: the success rate averaged over 1000 randomly sampled goals after 800 training epochs.}
\end{figure}

\begin{figure}[h!]
    \centering
    \includegraphics[width=0.6\textwidth]{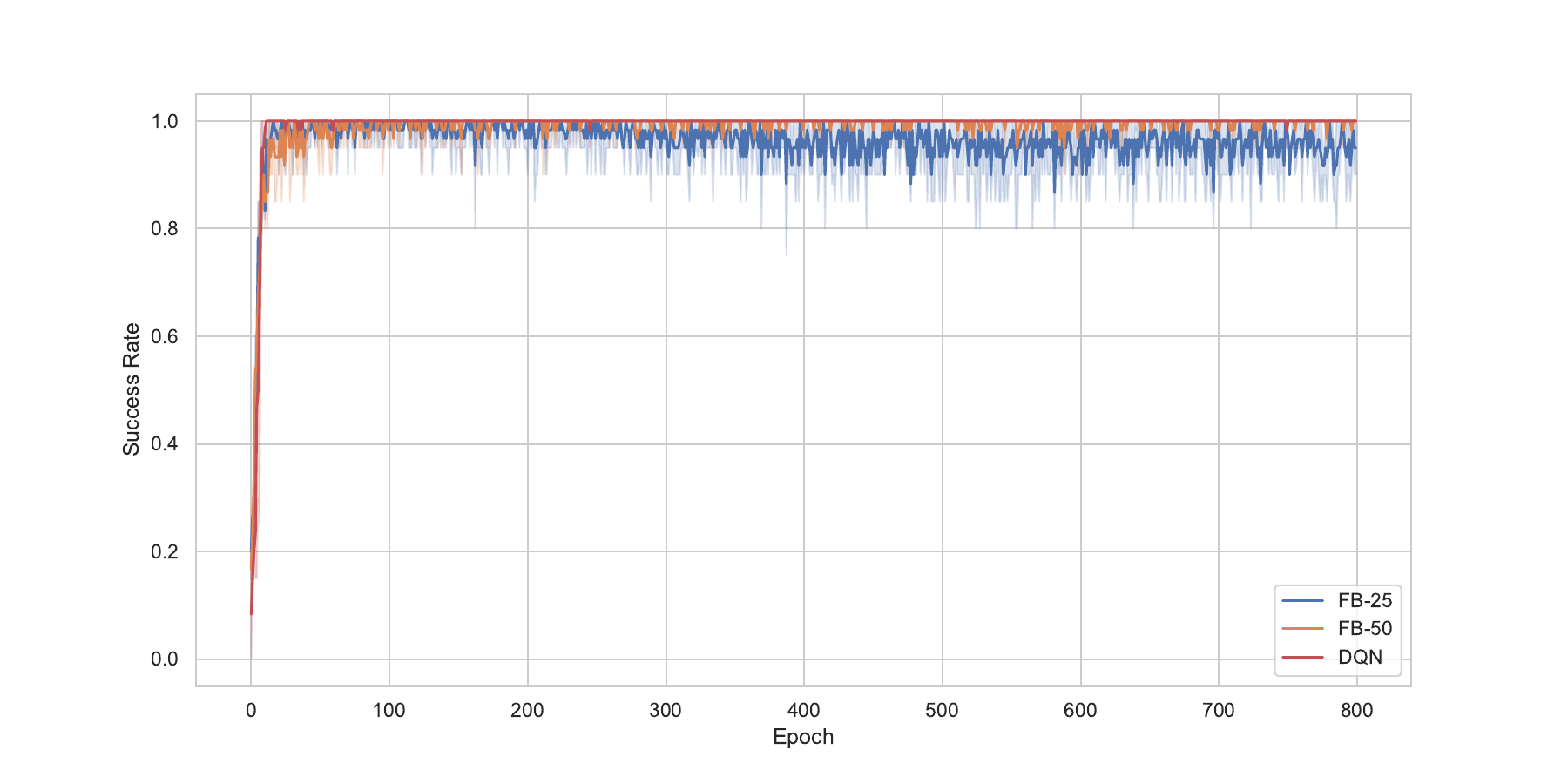} \hspace{10pt}
     \includegraphics[width=0.3\textwidth]{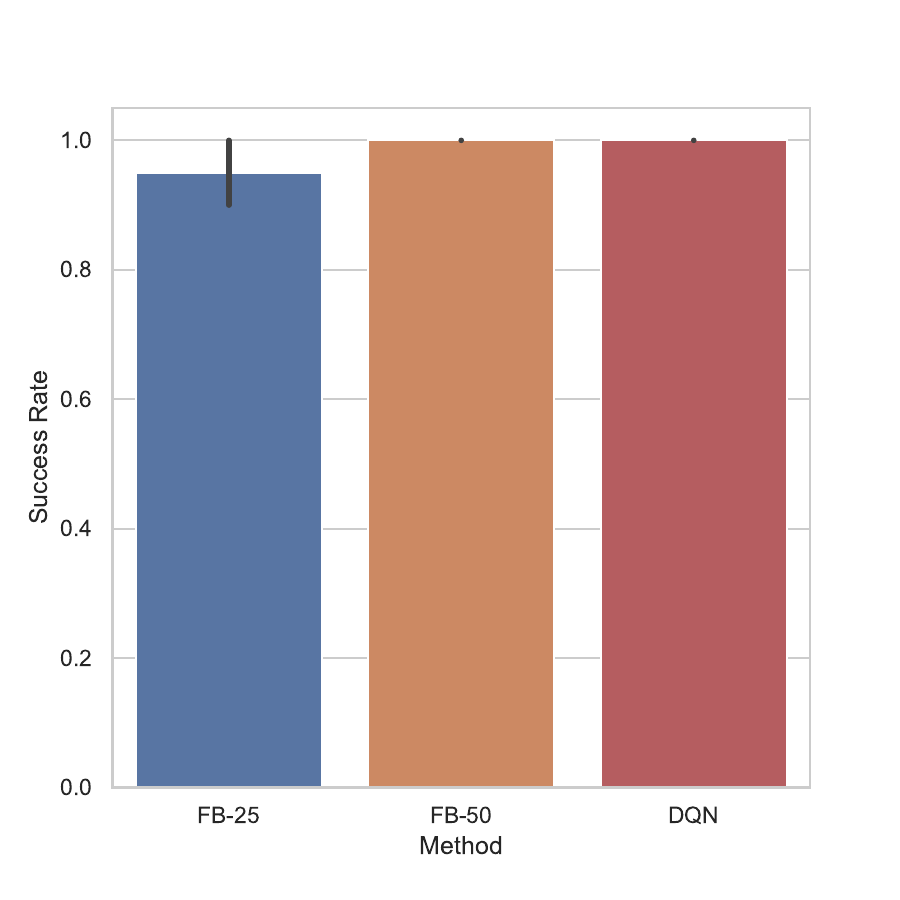}
    \caption{\textbf{FetchReach}: Comparative performance of FB for different dimensions and DQN. \textbf{Left}: the success rate averaged over 20 randomly selected goals as function of the training epochs. \textbf{Right}: the success rate averaged over 1000 randomly sampled goals after 800 training epochs.}
\end{figure}

\begin{figure}[h!]
    \centering
    \includegraphics[width=0.6\textwidth]{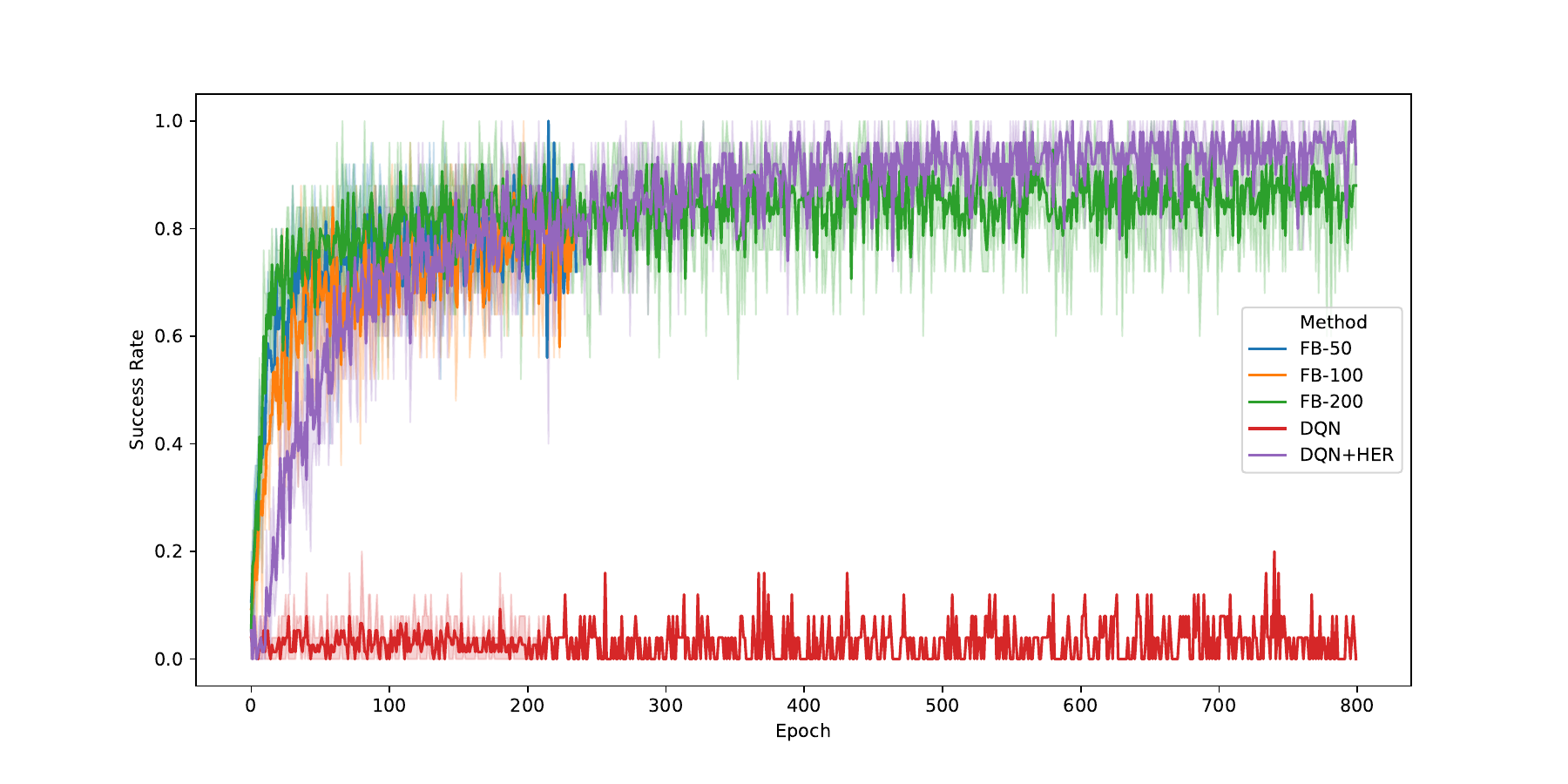} \hspace{10pt}
     \includegraphics[width=0.3\textwidth]{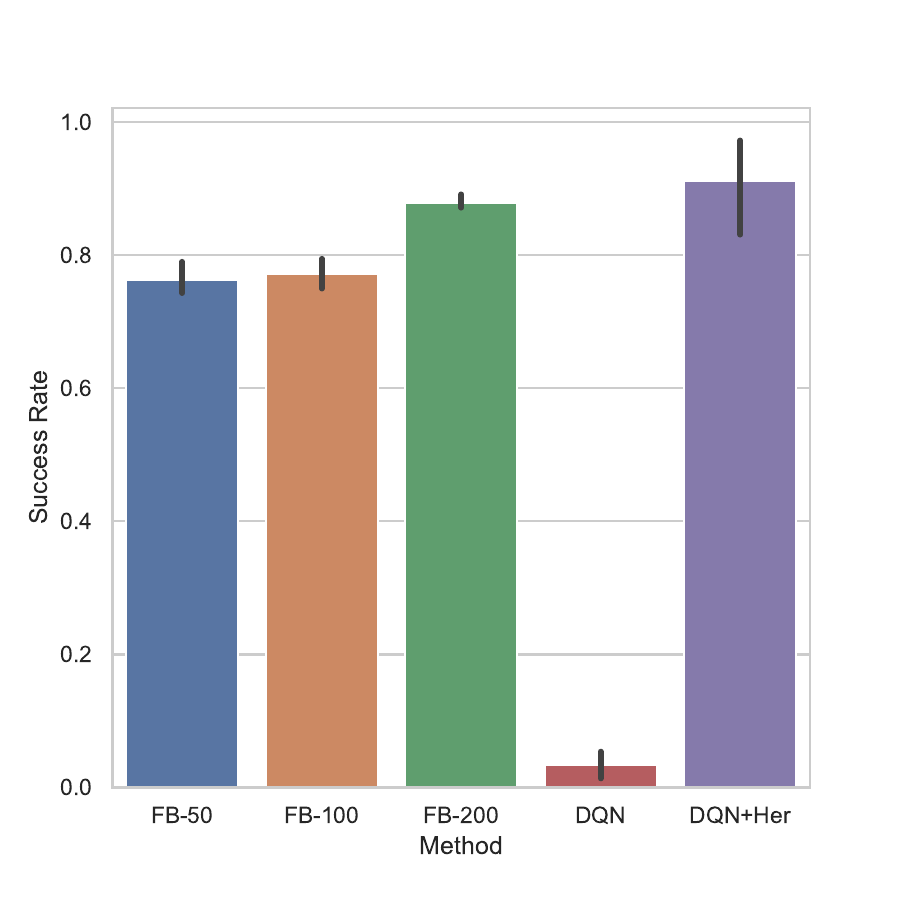}
    \caption{\textbf{Ms.\ Pacman}: Comparative performance of FB for different dimensions and DQN. \textbf{Left}: the success rate averaged over 20 randomly selected goals as function of the training epochs. \textbf{Right}: the success rate averaged over the 184 handcrafted goals after training epochs. Note that FB-50 and F-100 have been trained only for 200 epochs.}
\end{figure}

\begin{figure}[h!]
    \centering
    \includegraphics[width=0.48\textwidth]{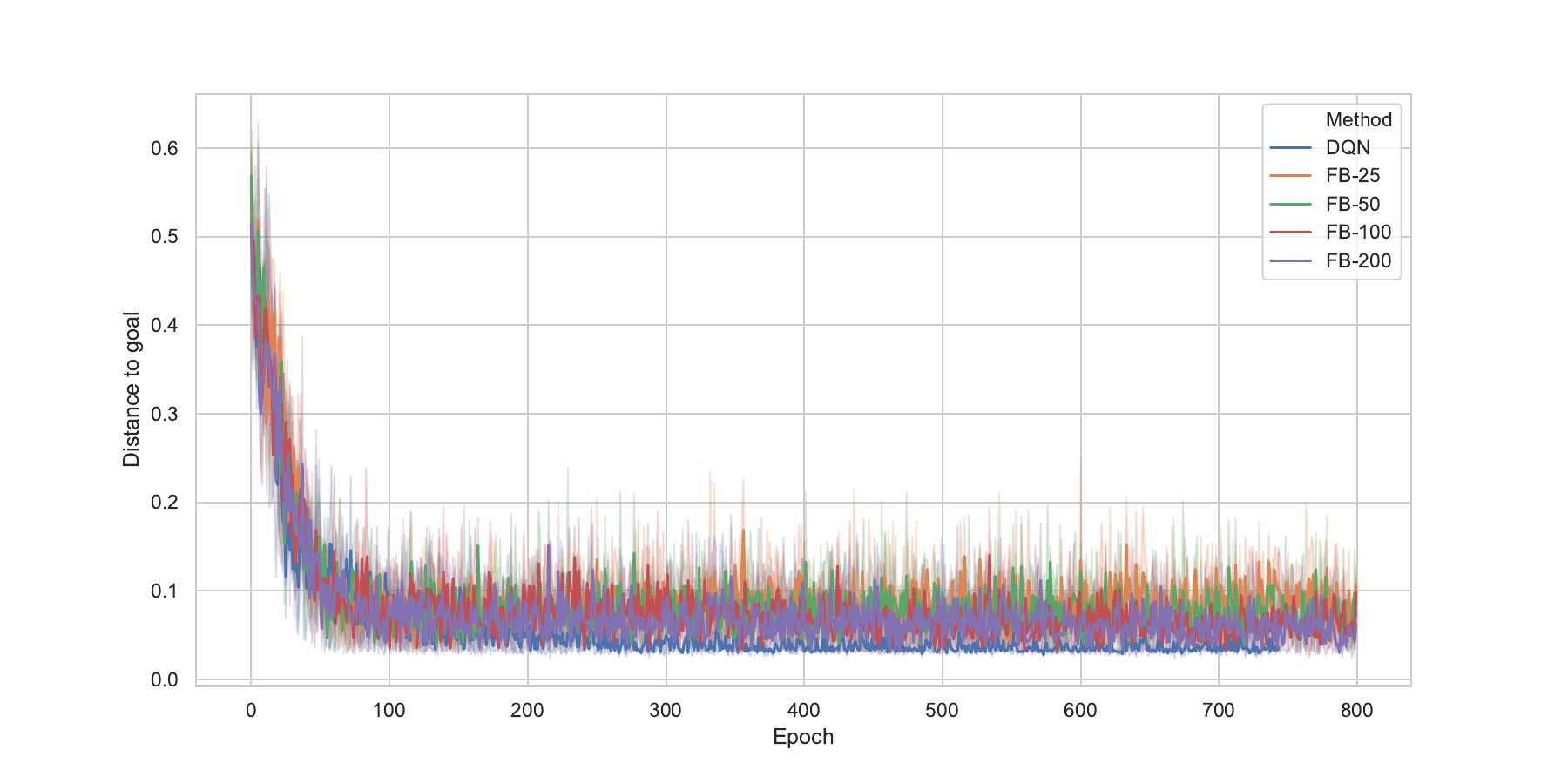} \hspace{10pt}
     \includegraphics[width=0.48\textwidth]{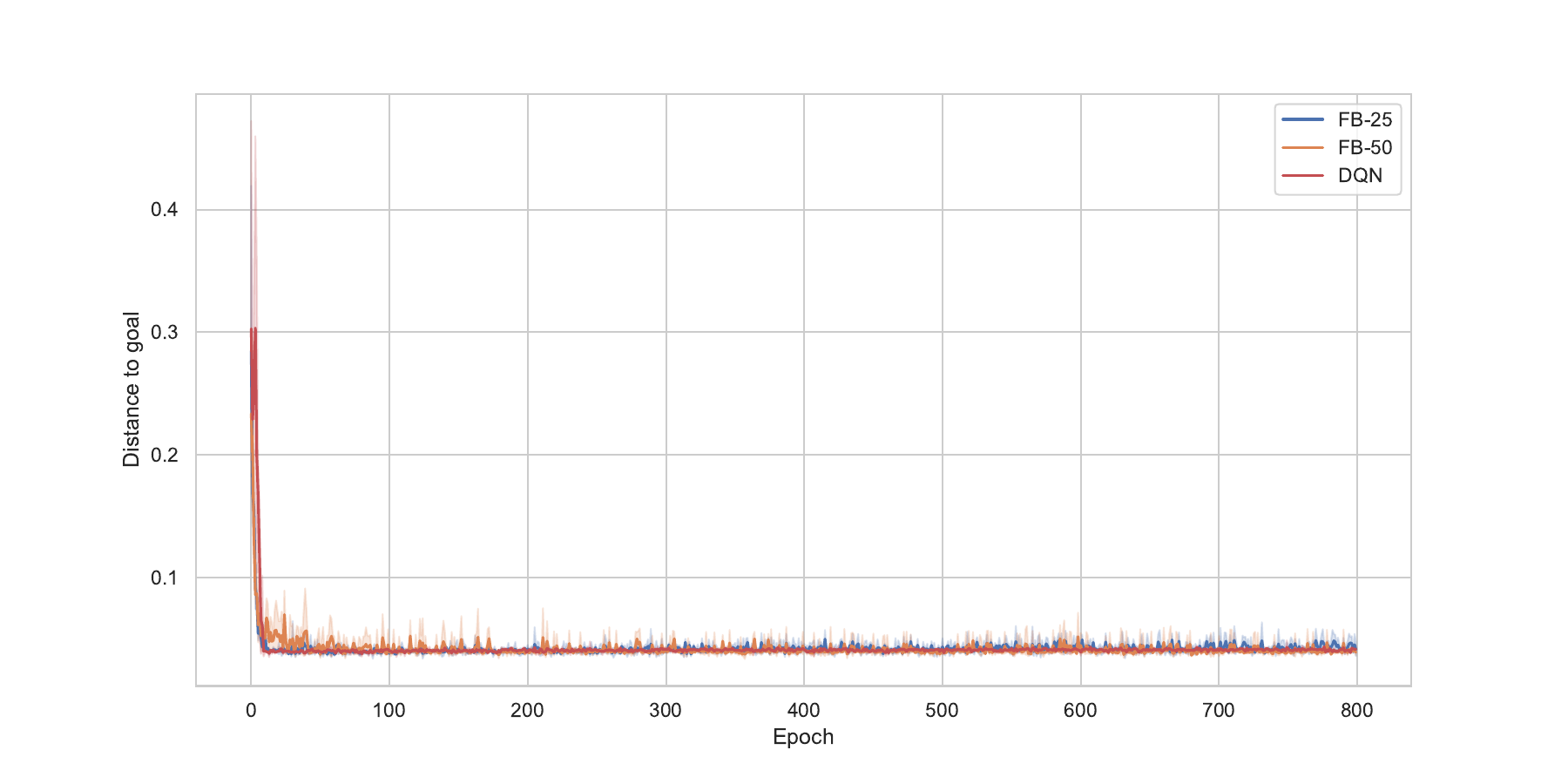}
    \caption{Distance to goal of FB for different dimensions and DQN as function of training epochs. \textbf{Left}: Continuous maze. \textbf{Right}: FetchReach.}
\end{figure}

\subsubsection{More Complex Rewards: Qualitative Results}

\begin{figure}[h!]
    \centering
    \includegraphics[width=1\textwidth]{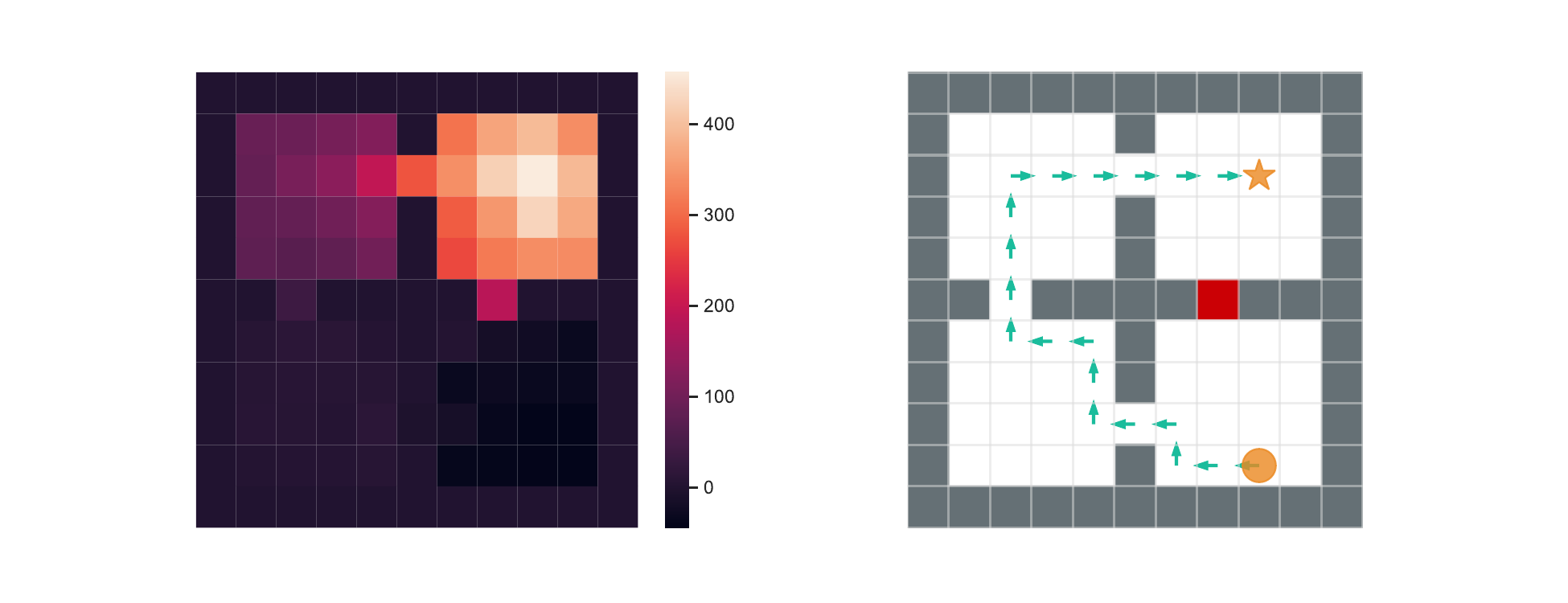} \hspace{10pt}
     \includegraphics[width=1\textwidth]{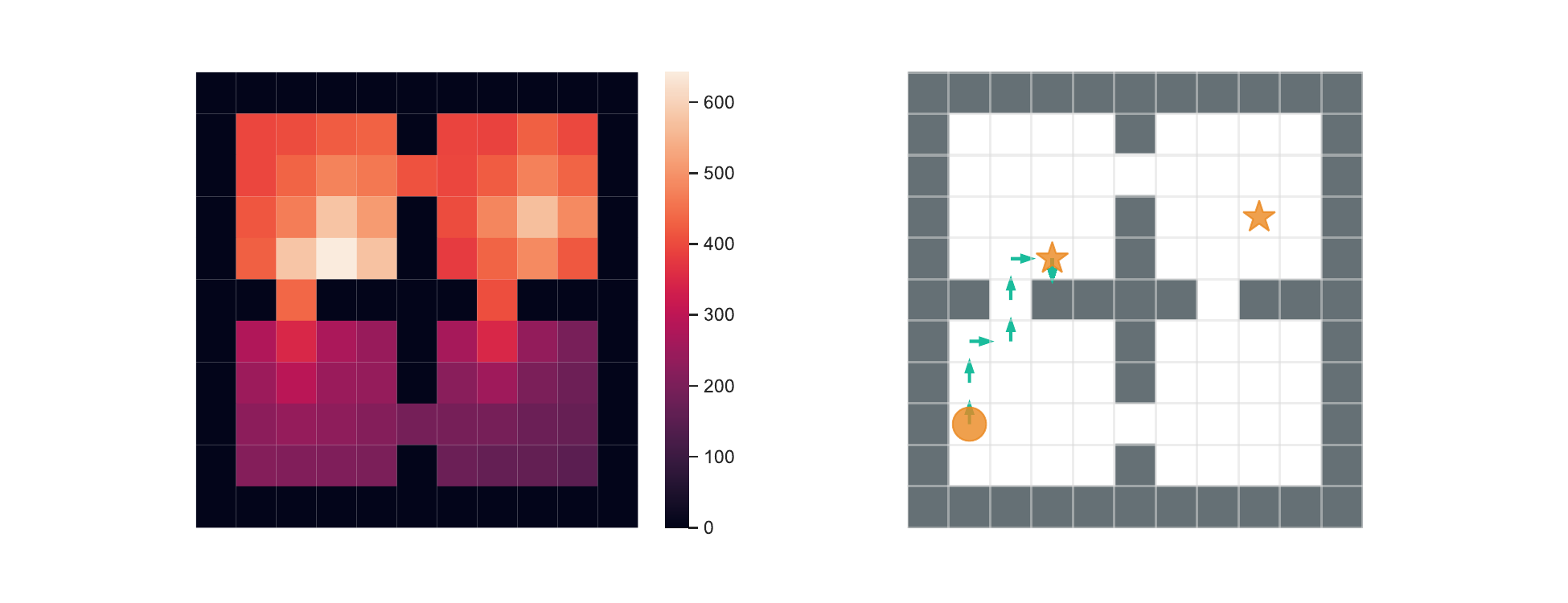}
      \includegraphics[width=1\textwidth]{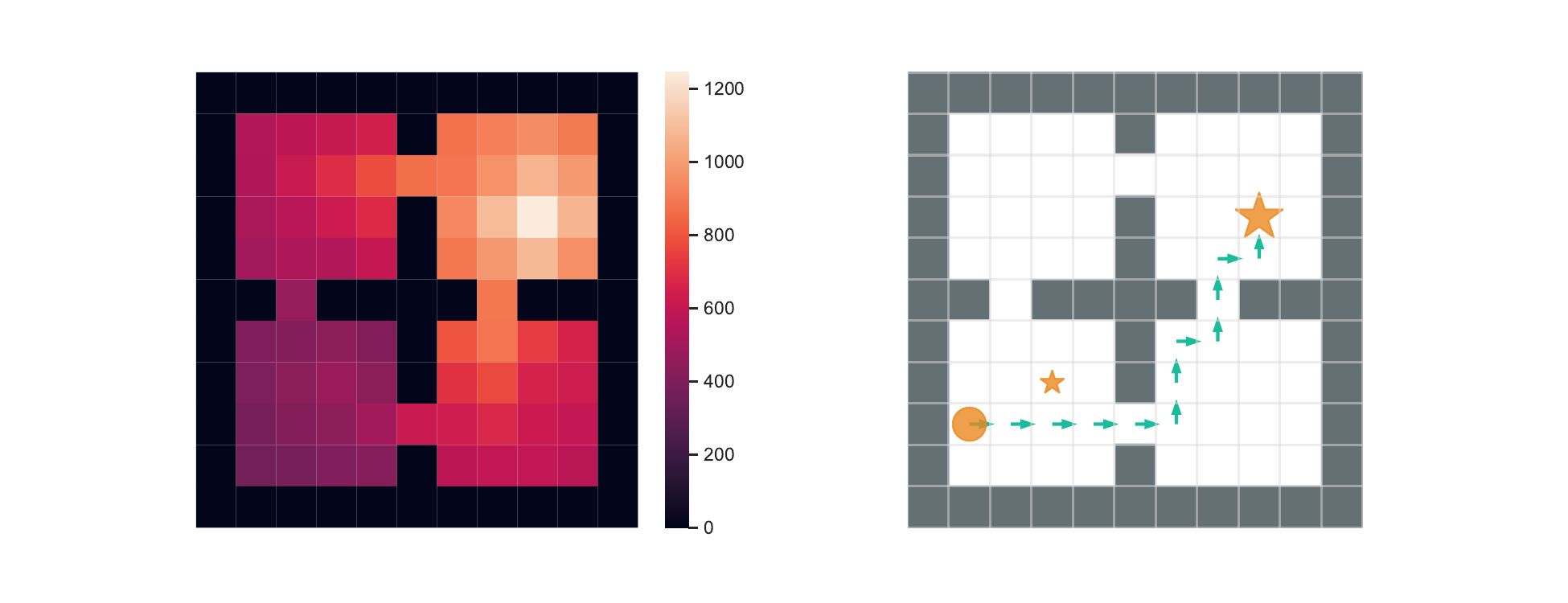}
    \caption{\textbf{Discrete Maze}: Heatmap
    plots of  $\max_{a \in A}F(s, a, z_R)^\top z_R$  (left) and trajectories of the Boltzmann policy with respect to $F(s, a, z_R)^\top z_R$ with temperature $\tau=1$ (right).
    \textbf{Top row}: for the task of reaching a target while avoiding a
    forbidden region, \textbf{Middle row}: for the task of reaching the closest goal among two equally rewarding positions, \textbf{Bottom row}: choosing between a small, close reward and a large, distant one. }
\end{figure}

\begin{figure}[h!]
    \centering
    \includegraphics[width=0.8\textwidth]{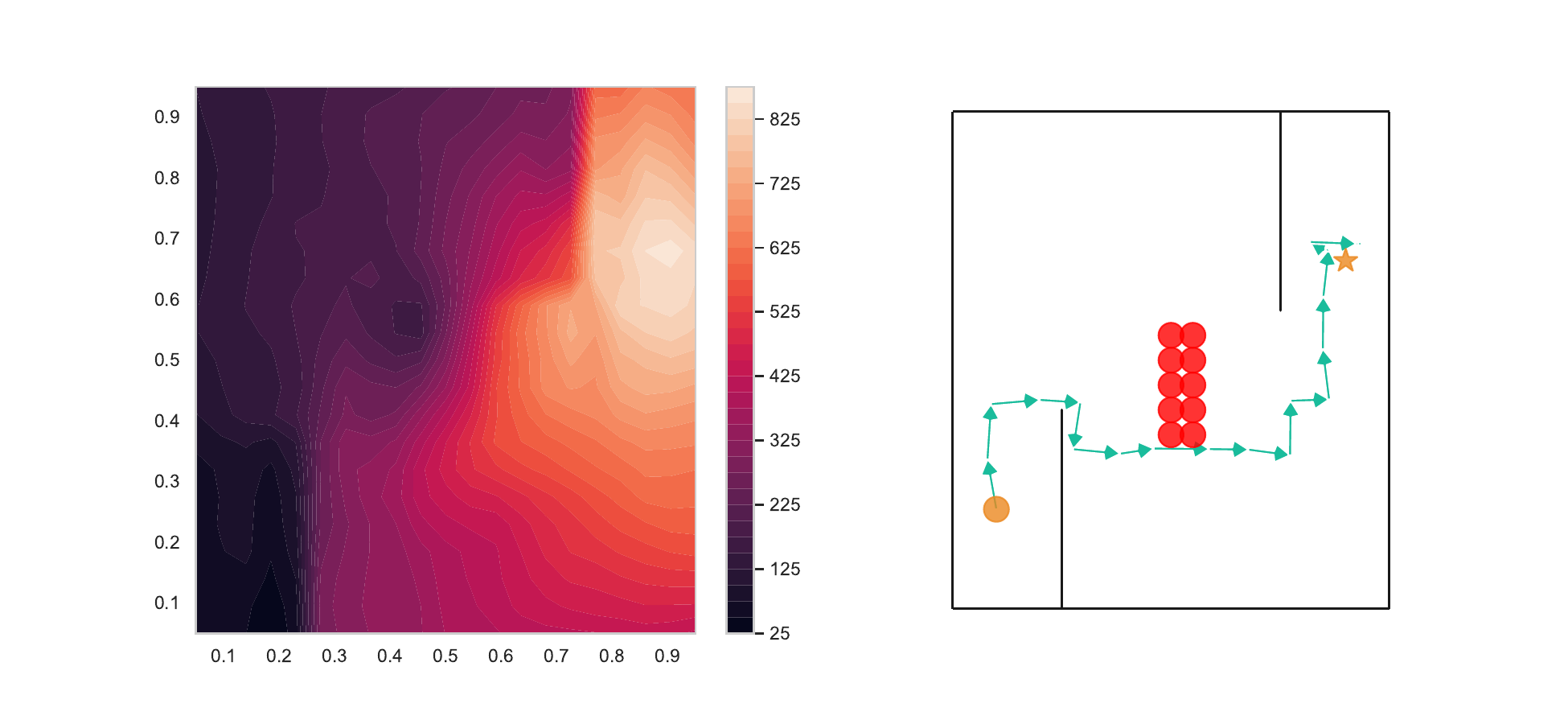}\hspace{10pt}
    \includegraphics[width=0.8\textwidth]{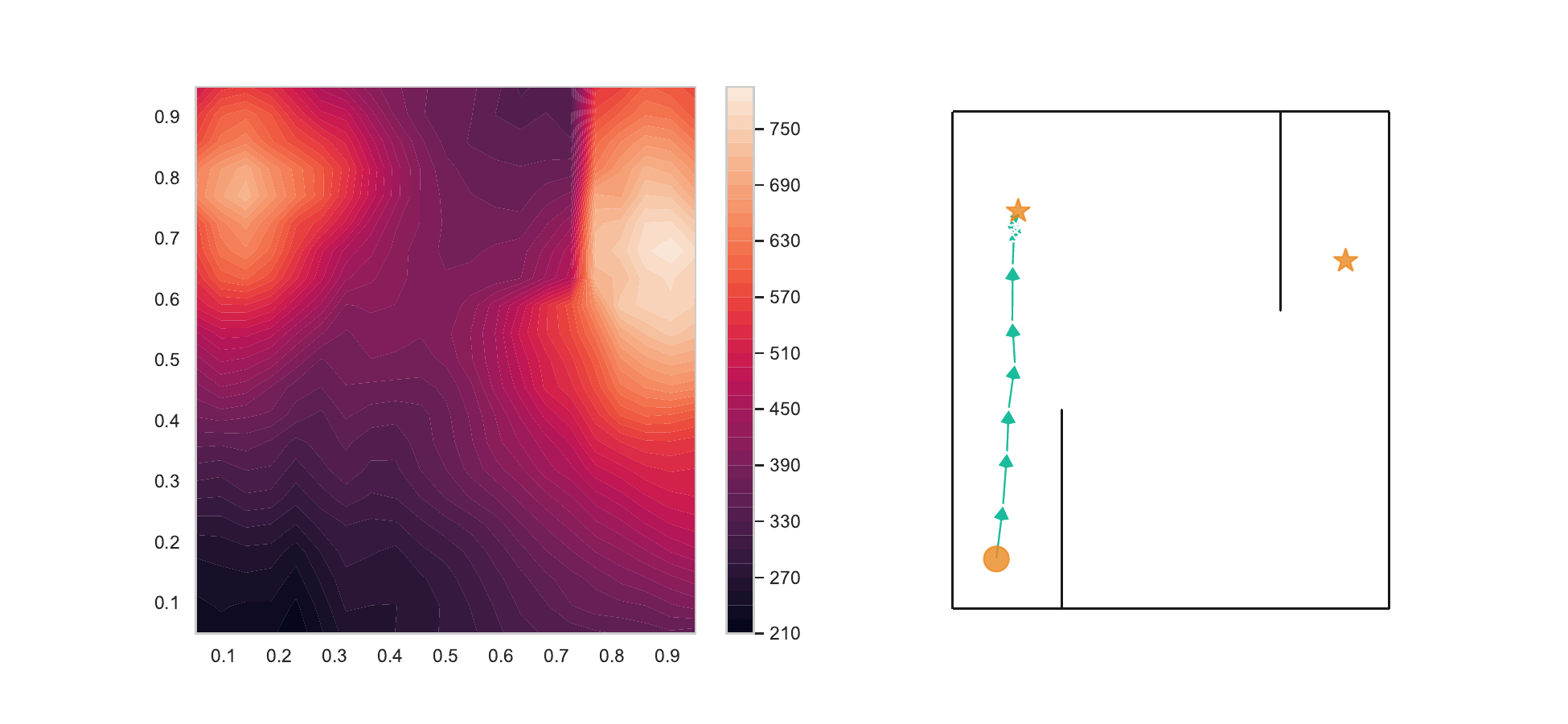} \hspace{10pt}
     \includegraphics[width=0.8\textwidth]{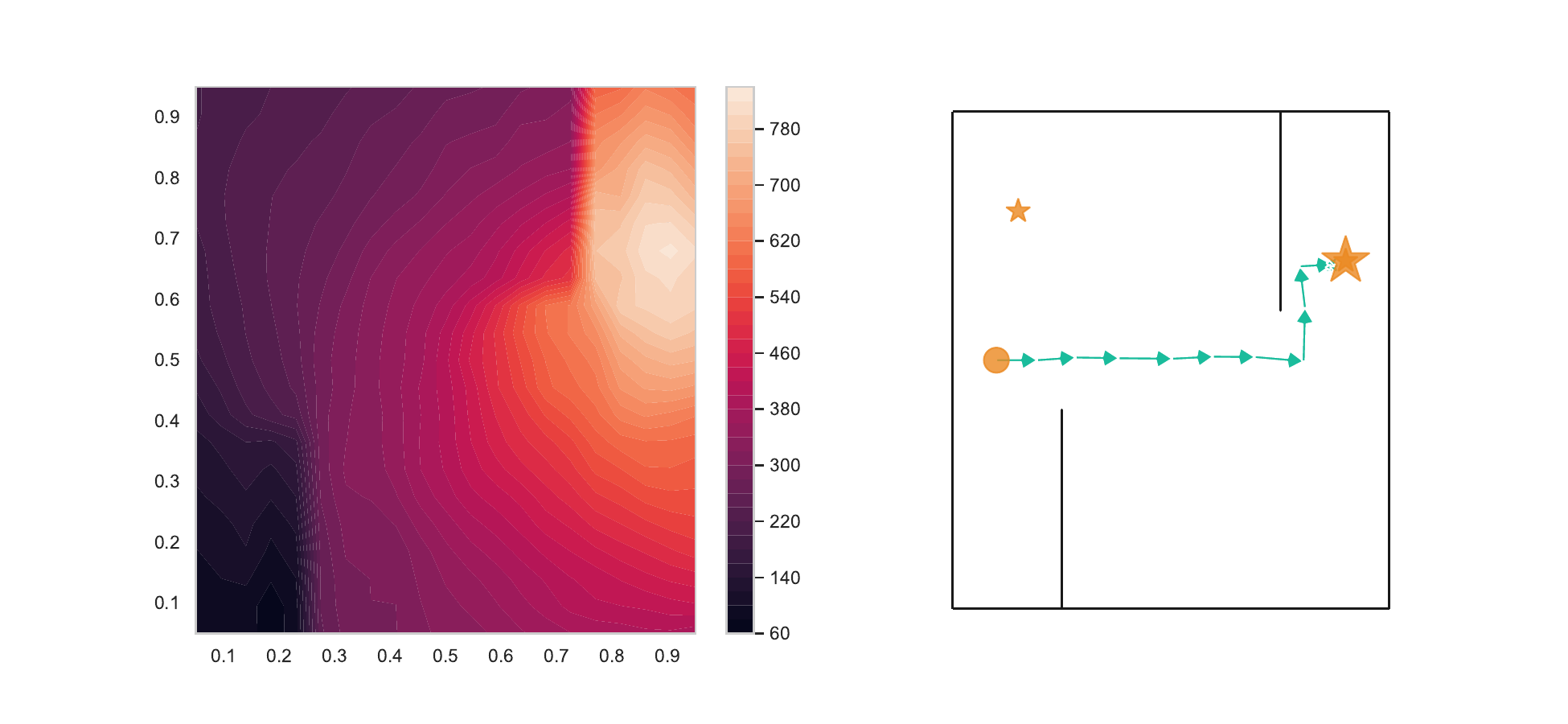} \hspace{10pt}
    \caption{\textbf{Continuous Maze}: Contour plots
    plot of  $\max_{a \in A}F(s, a, z_R)^\top z_R$  (left) and trajectories of the $\eps$ greedy policy with respect to $F(s, a, z_R)^\top z_R$ with $\eps=0.1$ (right).
    \textbf{Left}: for the task of reaching a target while avoiding a
    forbidden region, \textbf{Middle}: for the task of reaching the closest goal among two equally rewarding positions, \textbf{Right}: choosing between a small, close reward and a large, distant one..}
\end{figure}

\begin{figure}[h!]
    \centering
    \includegraphics[width=0.3\textwidth]{pacman/neg_FB-300.pdf}\hspace{10pt}
    \includegraphics[width=0.3\textwidth]{pacman/nearest_FB-300.pdf} \hspace{10pt}
     \includegraphics[width=0.3\textwidth]{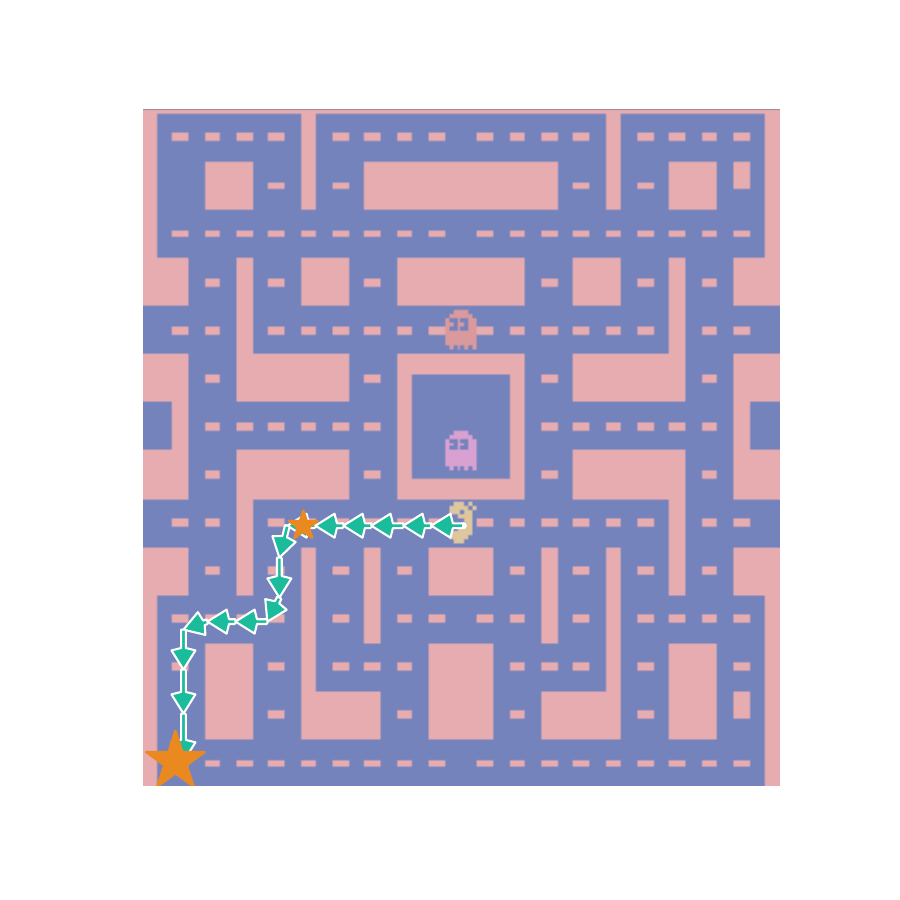} \hspace{10pt}
    \caption{ \textbf{Ms.\ Pacman}: Trajectories of the $\eps$ greedy policy with respect to $F(s, a, z_R)^\top z_R$ with $\eps=0.1$ (right).
    \textbf{Top row}: for the task of reaching a target while avoiding a
    forbidden region, \textbf{Middle row}: for the task of reaching the closest goal among two equally rewarding positions, \textbf{Bottom row}: choosing between a small, close reward and a large, distant one..}
\end{figure}

\begin{figure}[h!]
    \centering
    \includegraphics[width=0.19\textwidth,  trim = {1cm 1cm 1cm 1cm}]{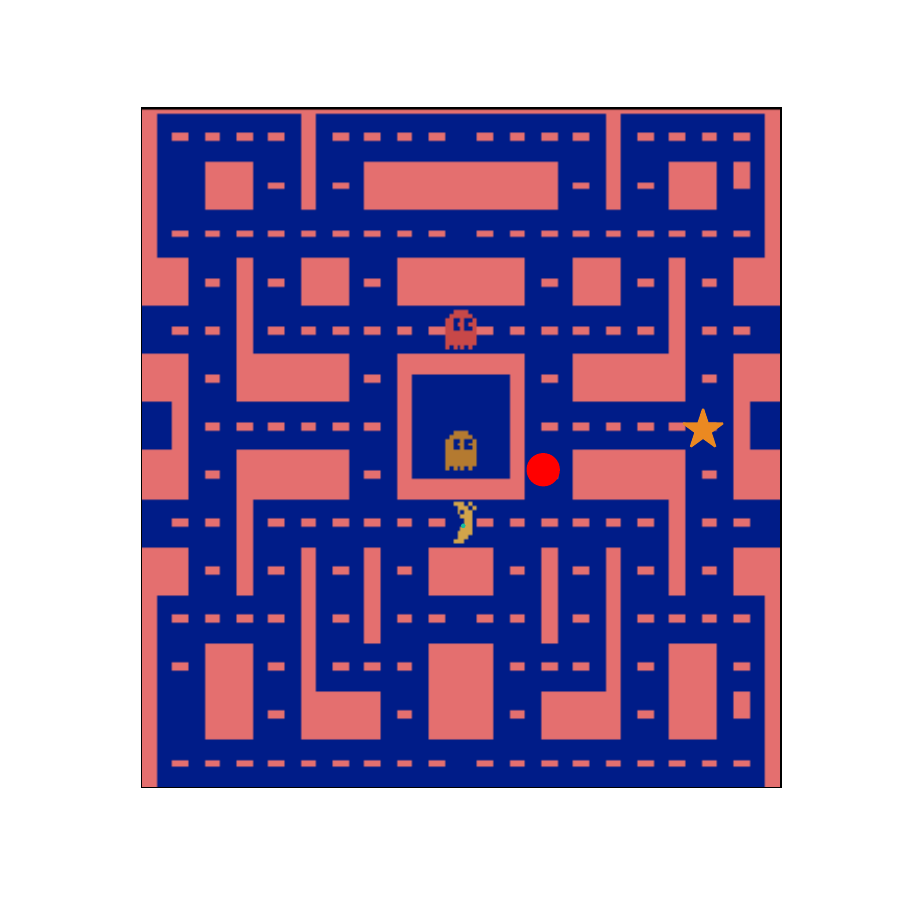}
    \includegraphics[width=0.19\textwidth,  trim = {1cm 1cm 1cm 1cm}]{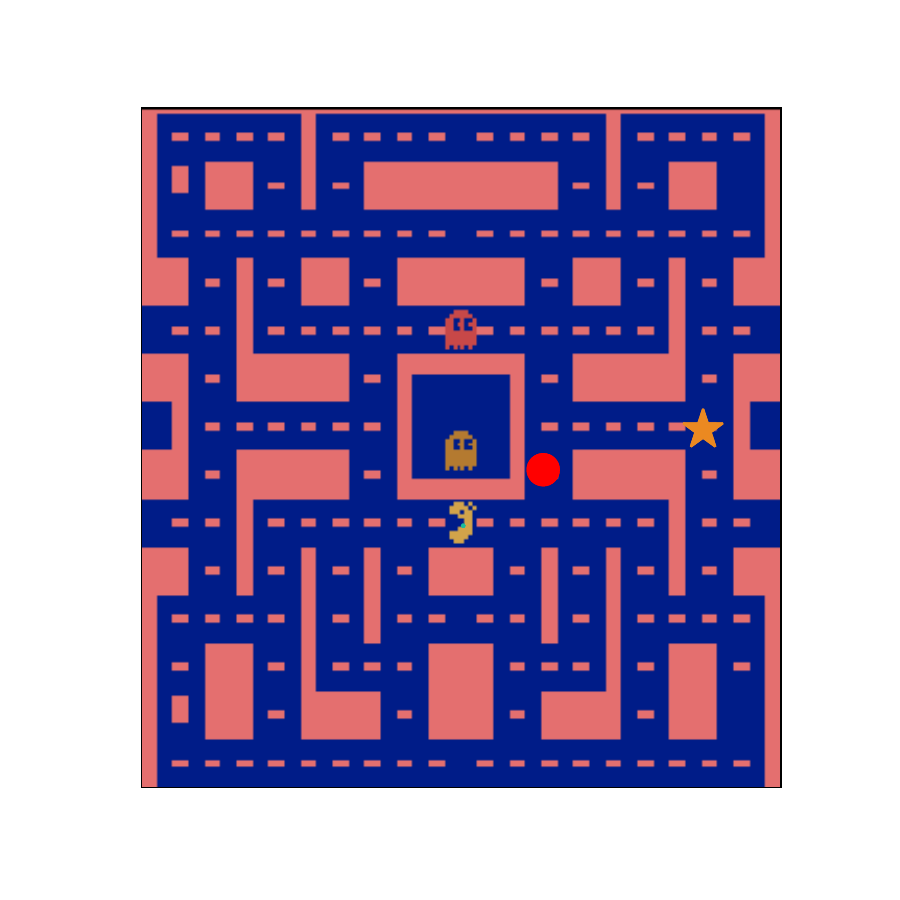} 
     \includegraphics[width=0.19\textwidth,  trim = {1cm 1cm 1cm 1cm}]{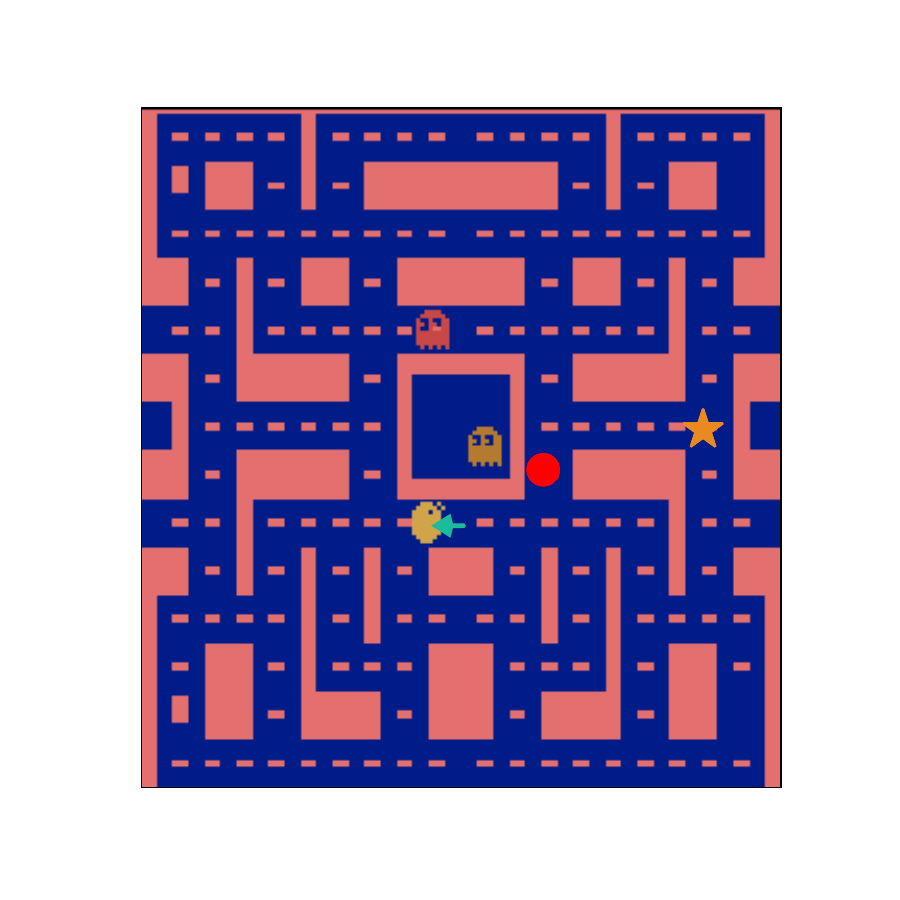}
     \includegraphics[width=0.19\textwidth,  trim = {1cm 1cm 1cm 1cm}]{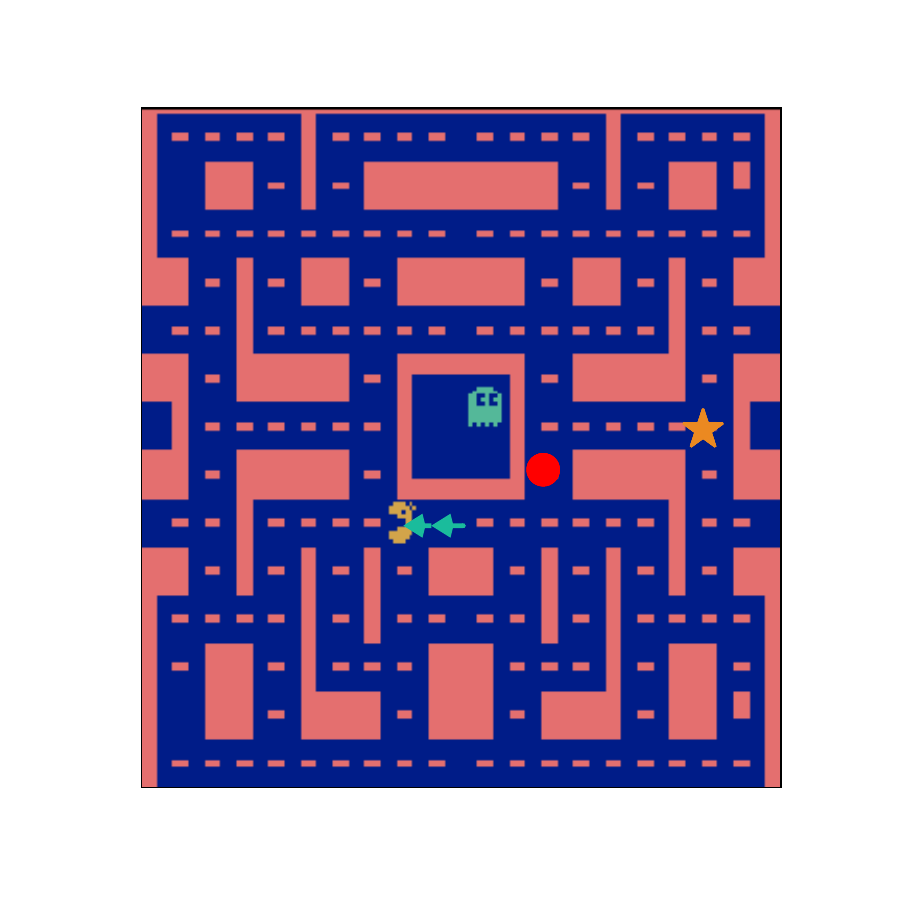} 
     \includegraphics[width=0.19\textwidth,  trim = {1cm 1cm 1cm 1cm}]{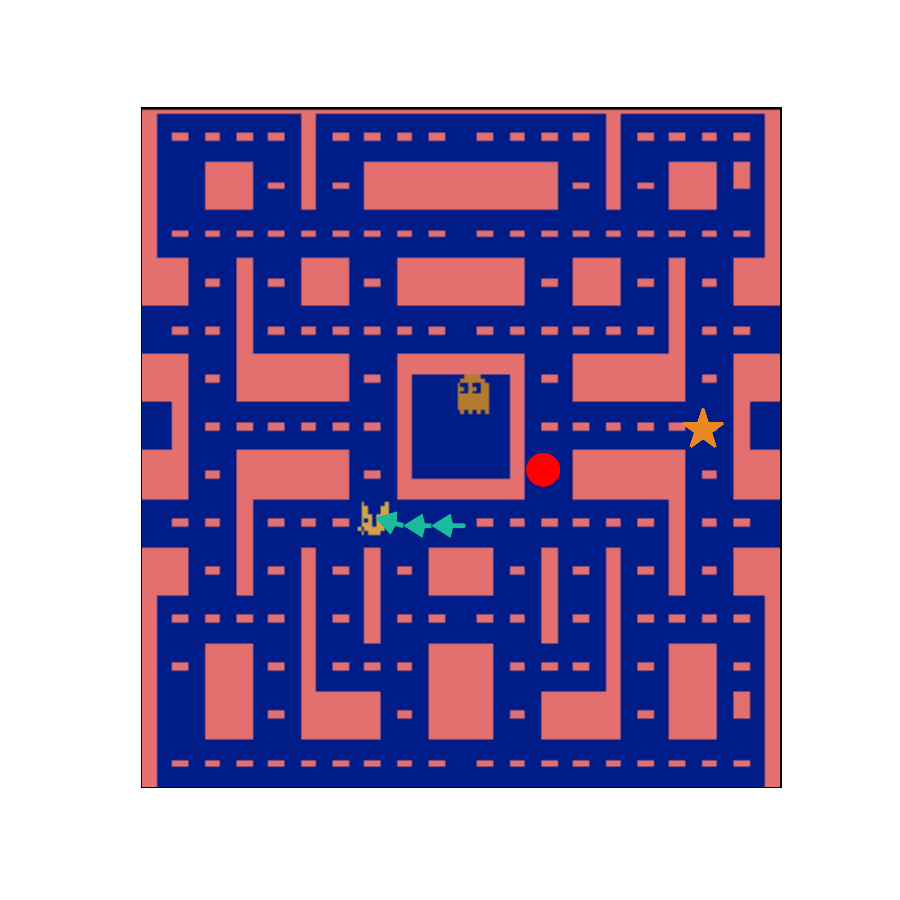} 
     \includegraphics[width=0.19\textwidth,  trim = {1cm 1cm 1cm 1cm}]{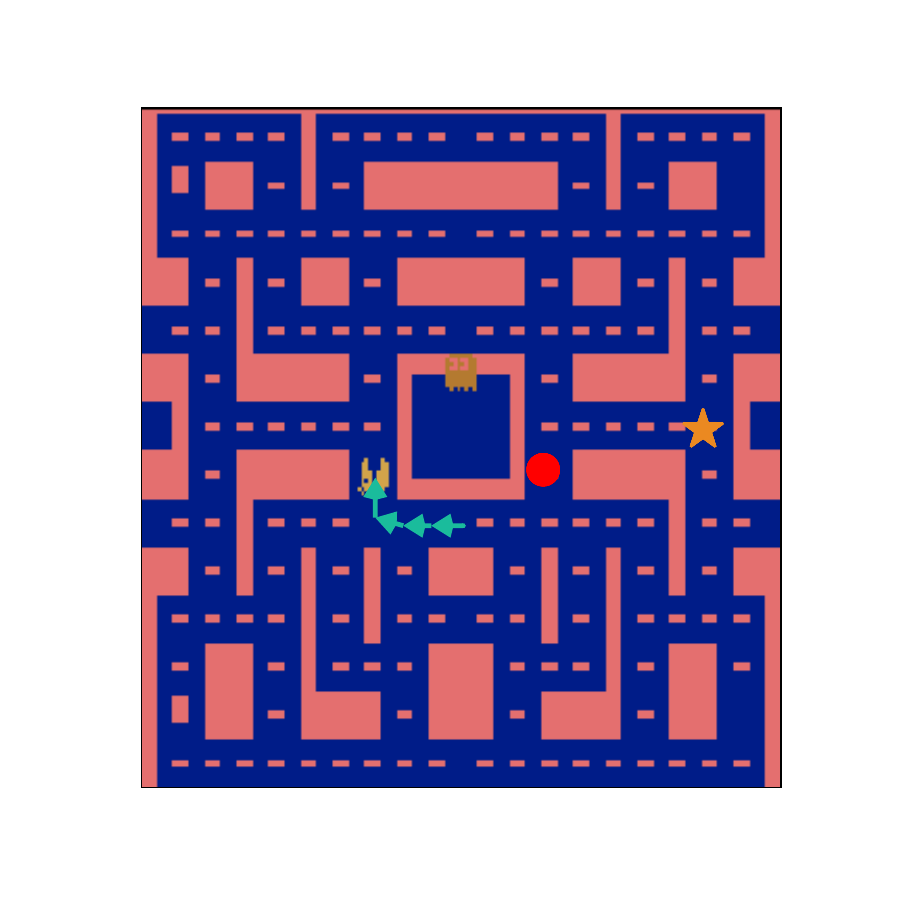}
    \includegraphics[width=0.19\textwidth,  trim = {1cm 1cm 1cm 1cm}]{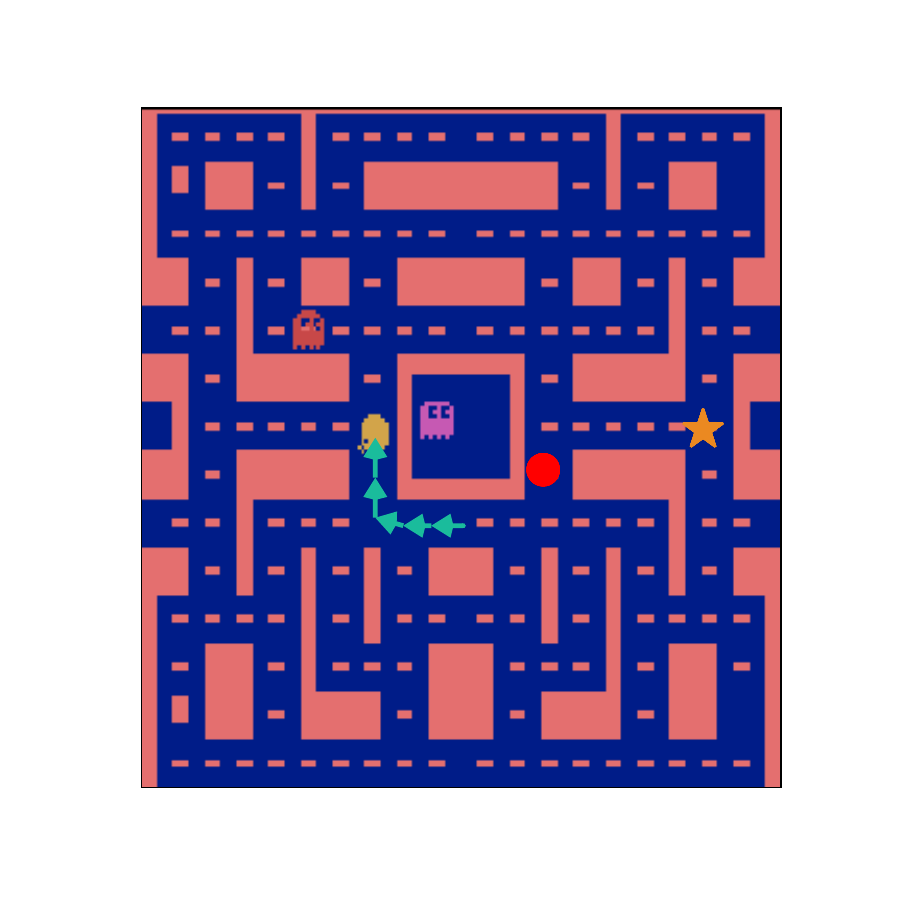} 
     \includegraphics[width=0.19\textwidth,  trim = {1cm 1cm 1cm 1cm}]{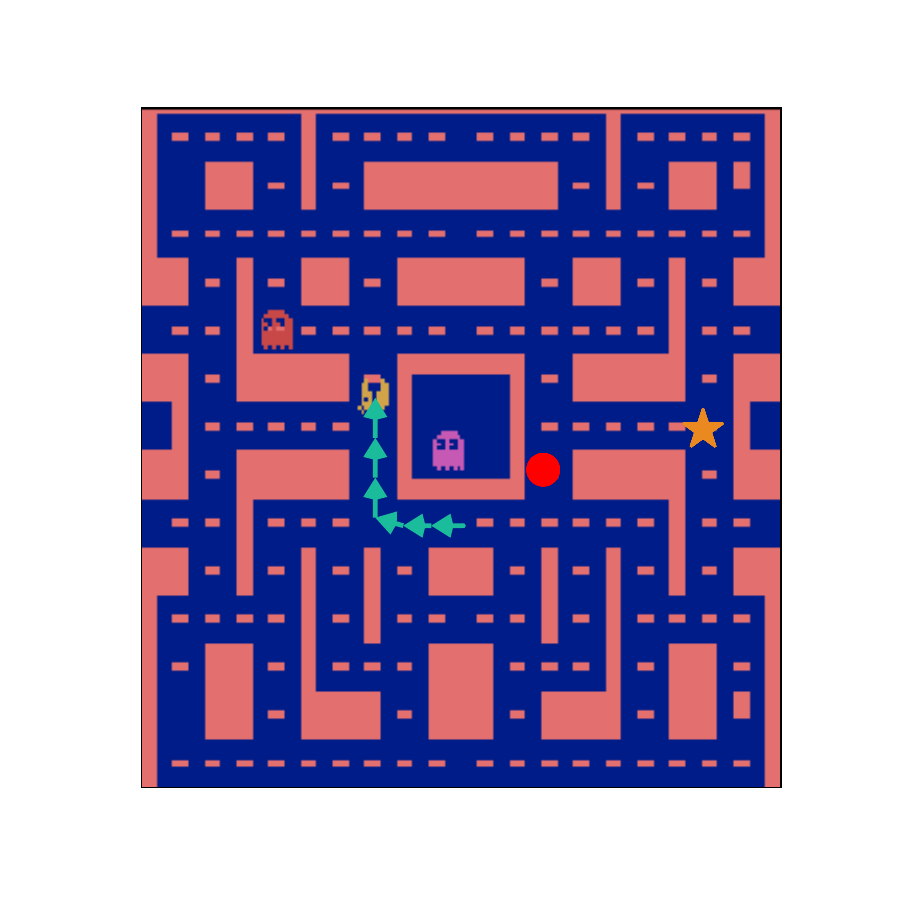}
     \includegraphics[width=0.19\textwidth,  trim = {1cm 1cm 1cm 1cm}]{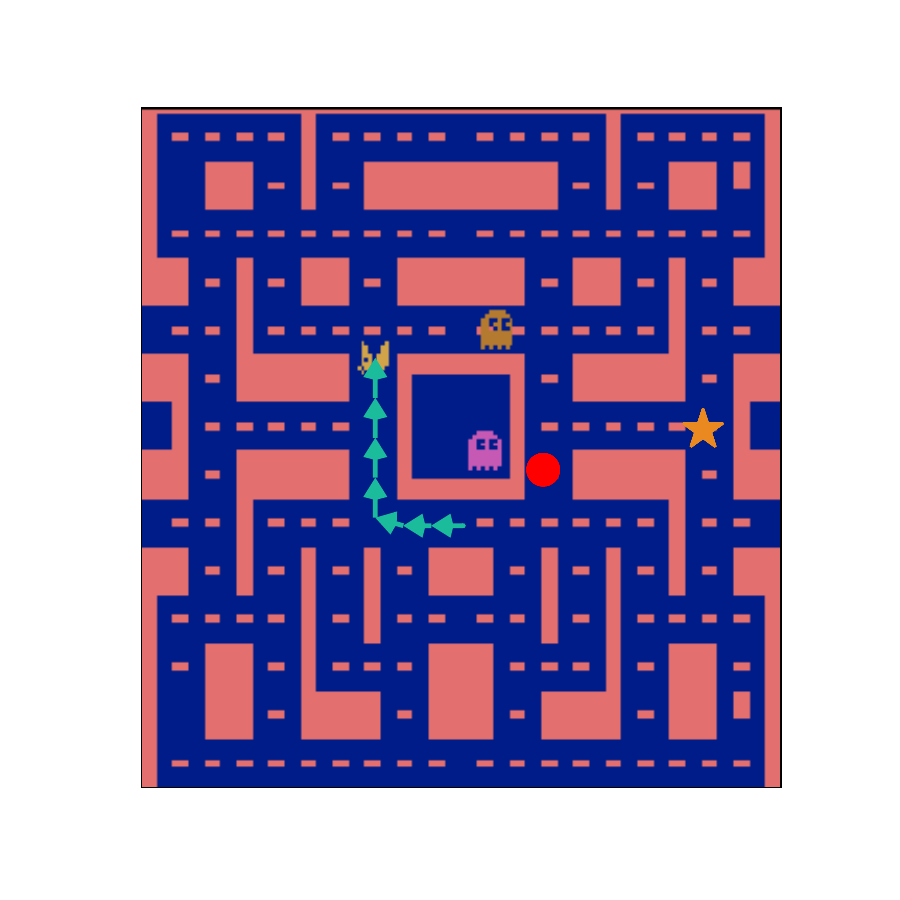} 
     \includegraphics[width=0.19\textwidth,  trim = {1cm 1cm 1cm 1cm}]{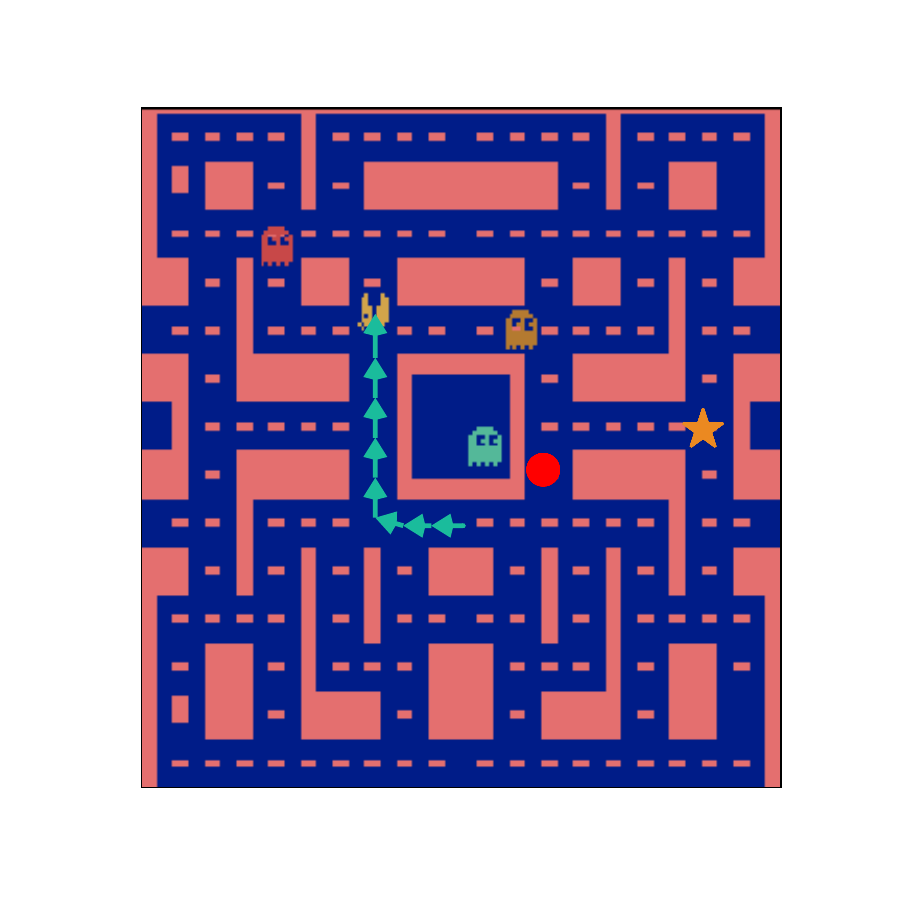} 
     \includegraphics[width=0.19\textwidth,  trim = {1cm 1cm 1cm 1cm}]{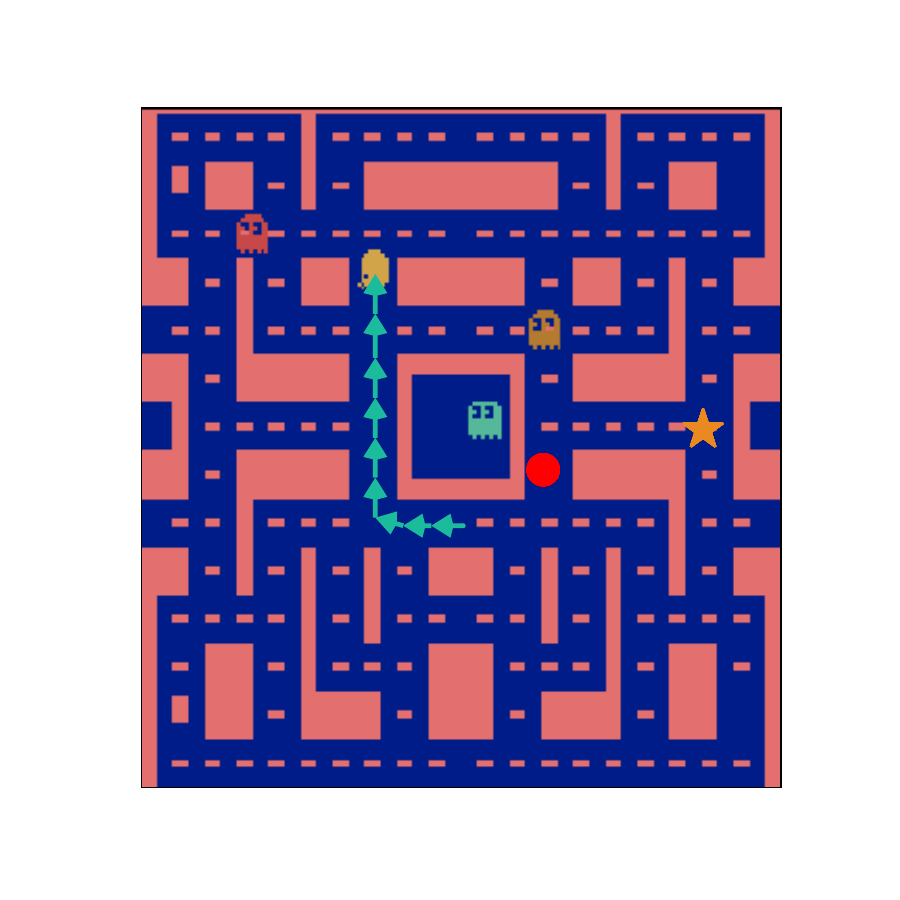}
    \includegraphics[width=0.19\textwidth,  trim = {1cm 1cm 1cm 1cm}]{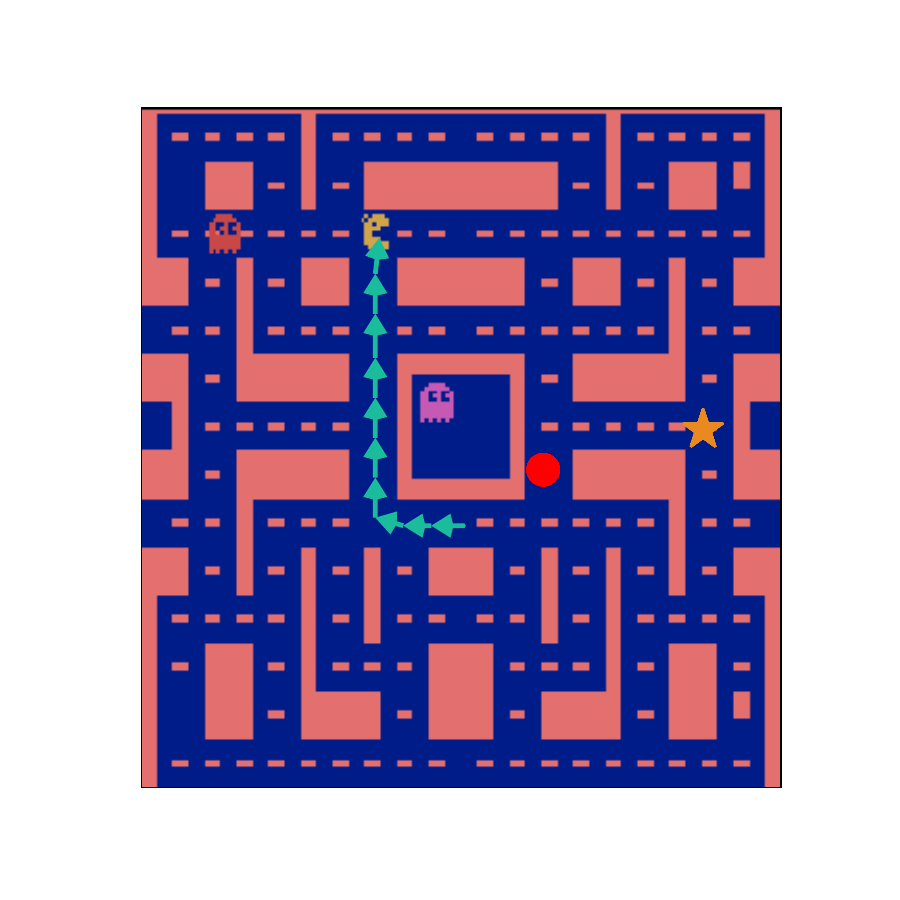} 
     \includegraphics[width=0.19\textwidth, trim = {1cm 1cm 1cm 1cm}]{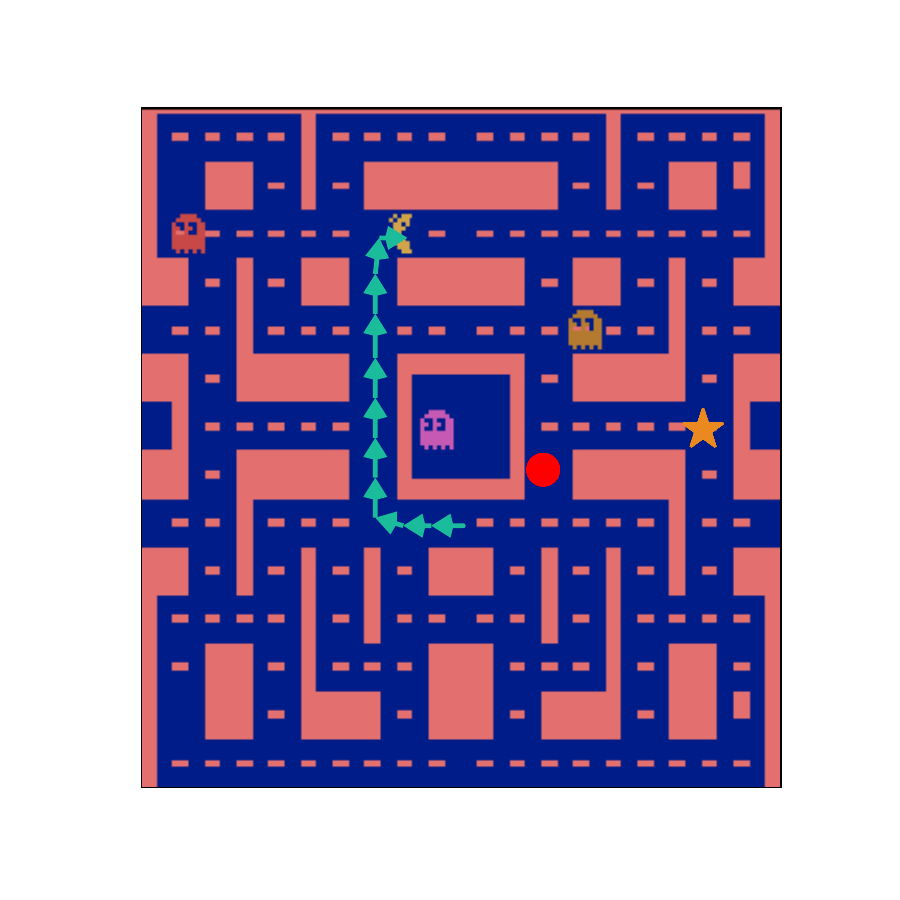}
     \includegraphics[width=0.19\textwidth,  trim = {1cm 1cm 1cm 1cm}]{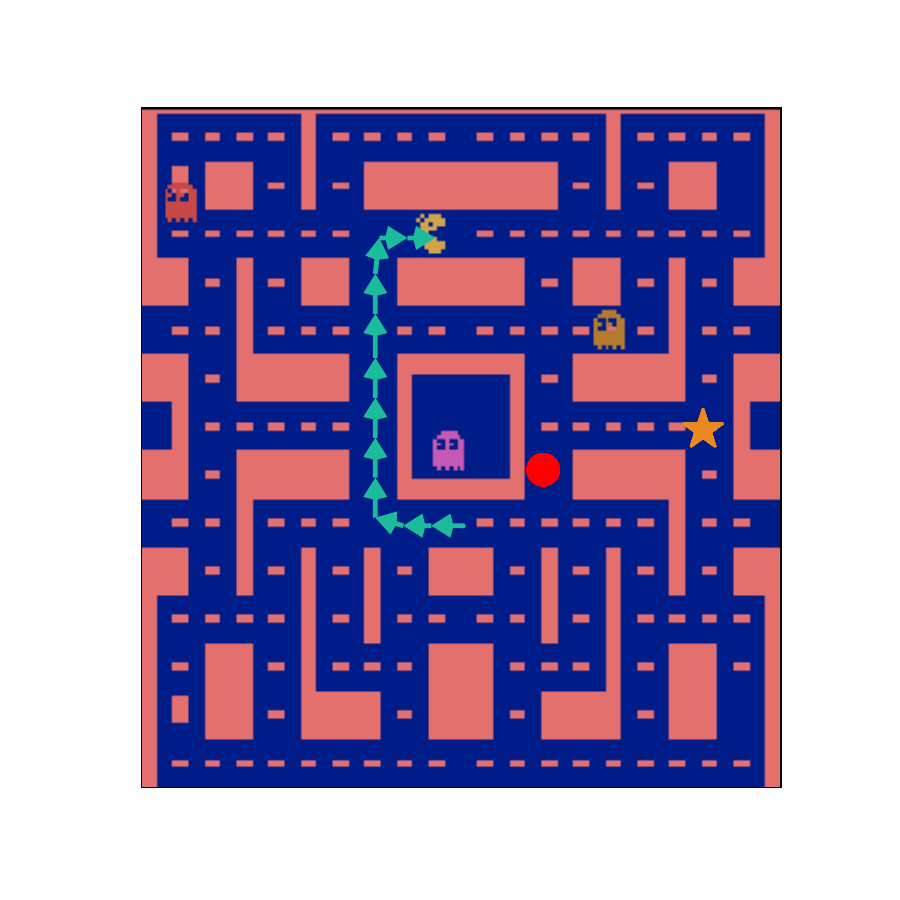} 
     \includegraphics[width=0.19\textwidth,  trim = {1cm 1cm 1cm 1cm}]{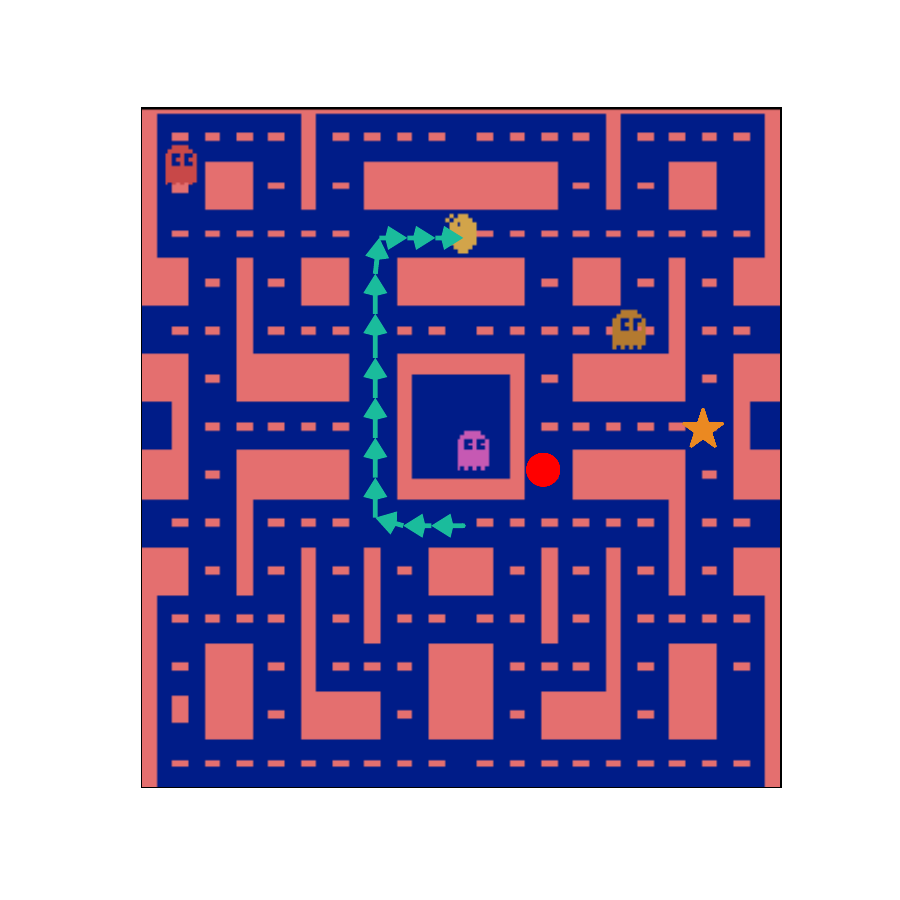} 
     \includegraphics[width=0.19\textwidth,  trim = {1cm 1cm 1cm 1cm}]{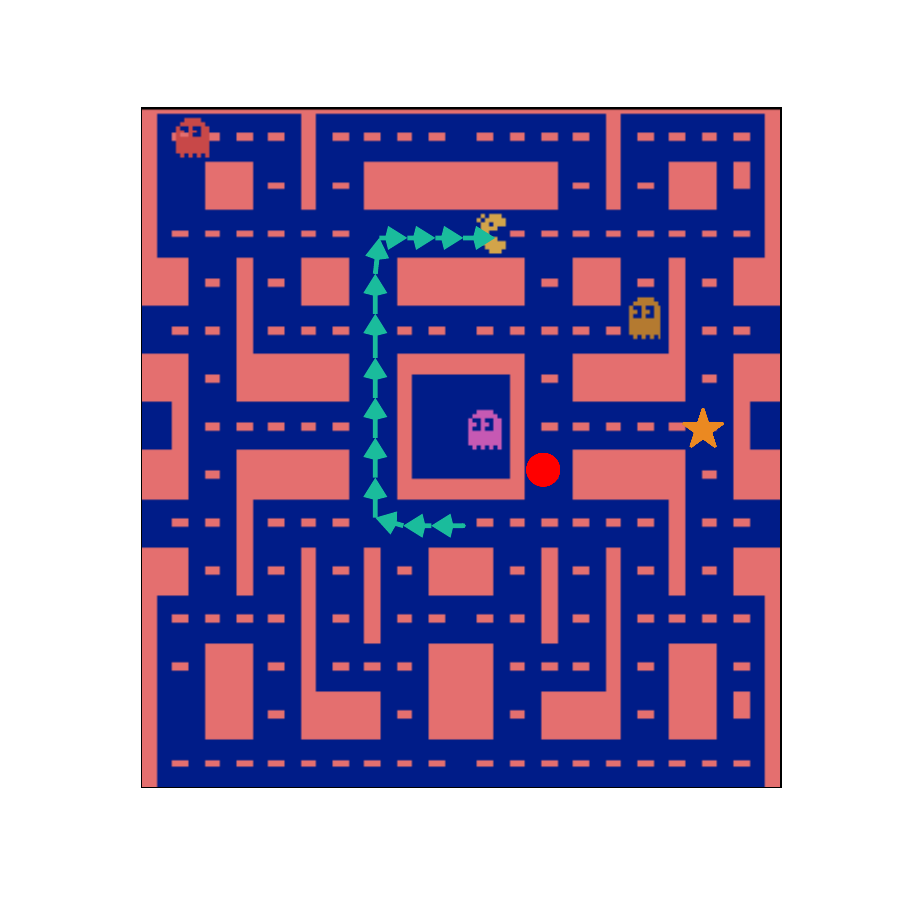}
    \includegraphics[width=0.19\textwidth,  trim = {1cm 1cm 1cm 1cm}]{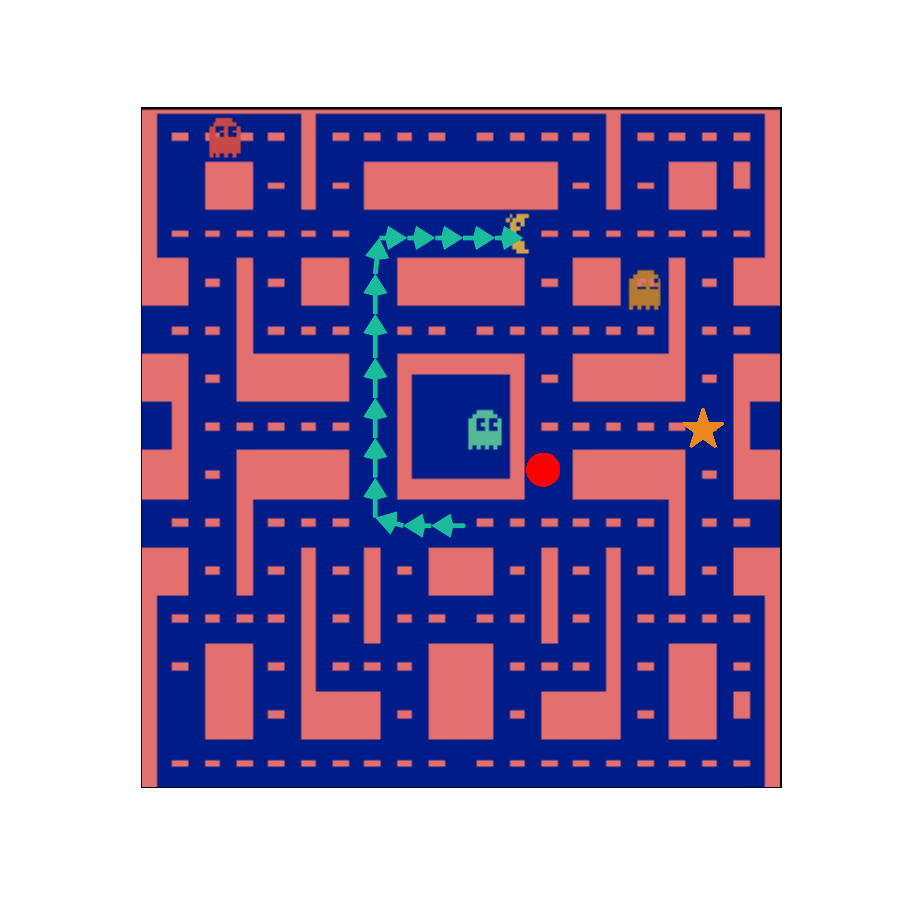} 
     \includegraphics[width=0.19\textwidth,  trim = {1cm 1cm 1cm 1cm}]{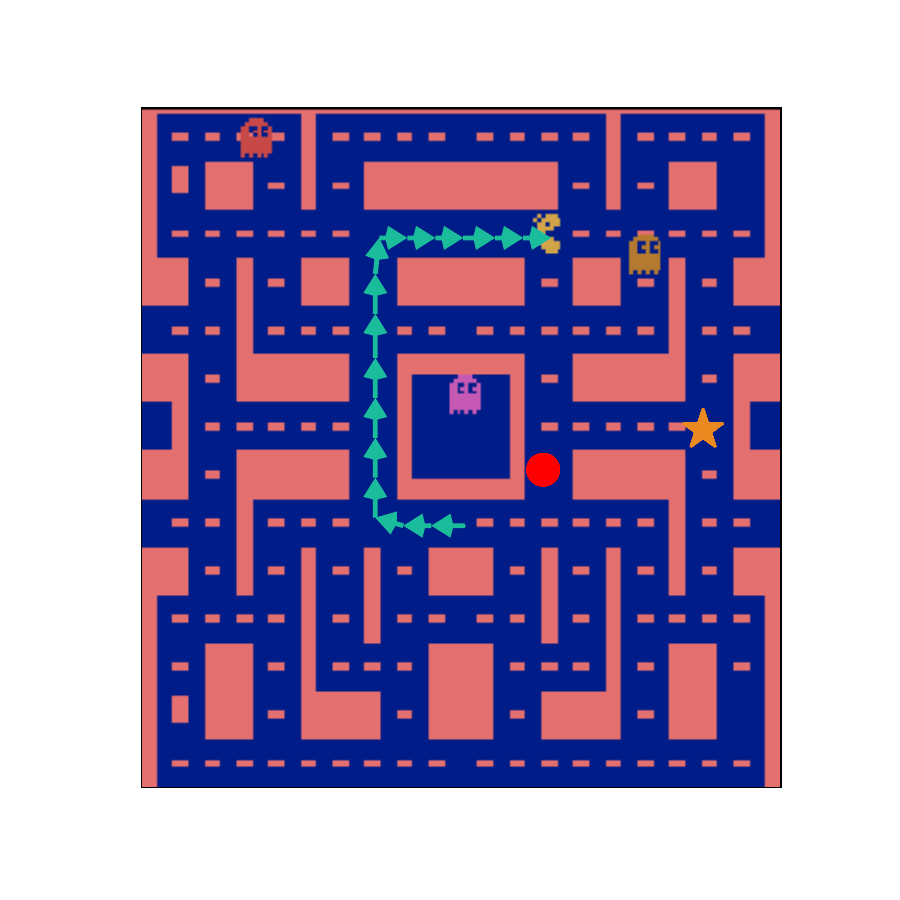}
     \includegraphics[width=0.19\textwidth,  trim = {1cm 1cm 1cm 1cm}]{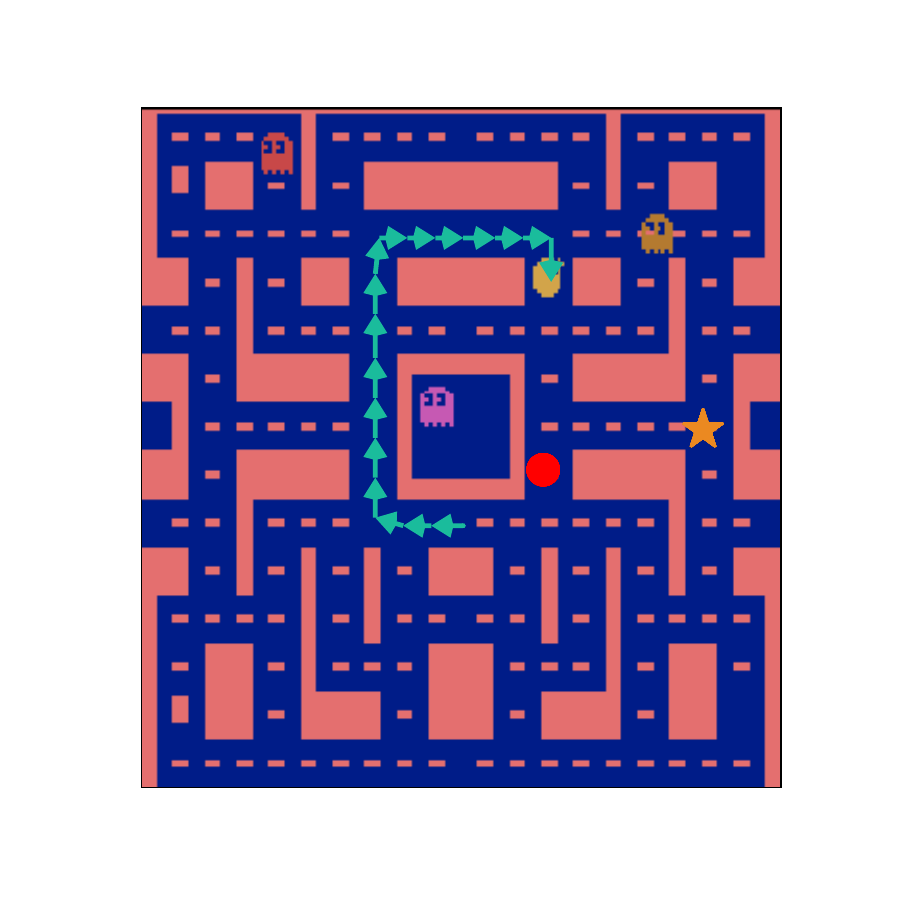} 
     \includegraphics[width=0.19\textwidth,  trim = {1cm 1cm 1cm 1cm}]{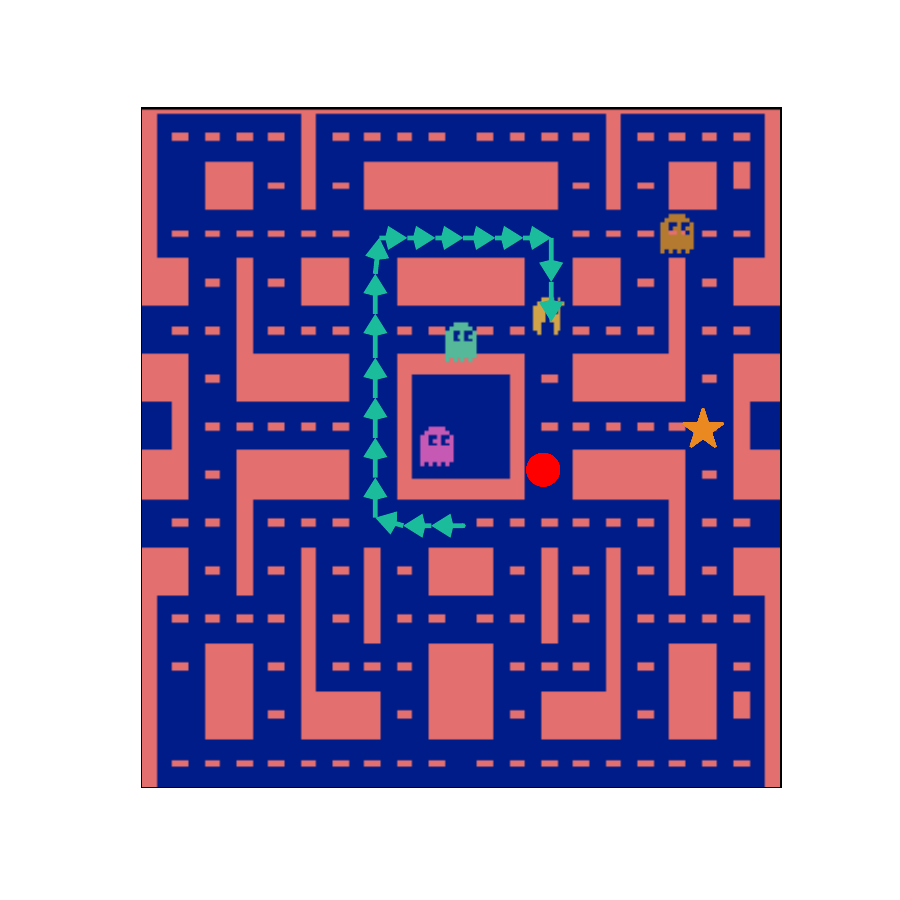}  
     \includegraphics[width=0.19\textwidth,  trim = {1cm 1cm 1cm 1cm}]{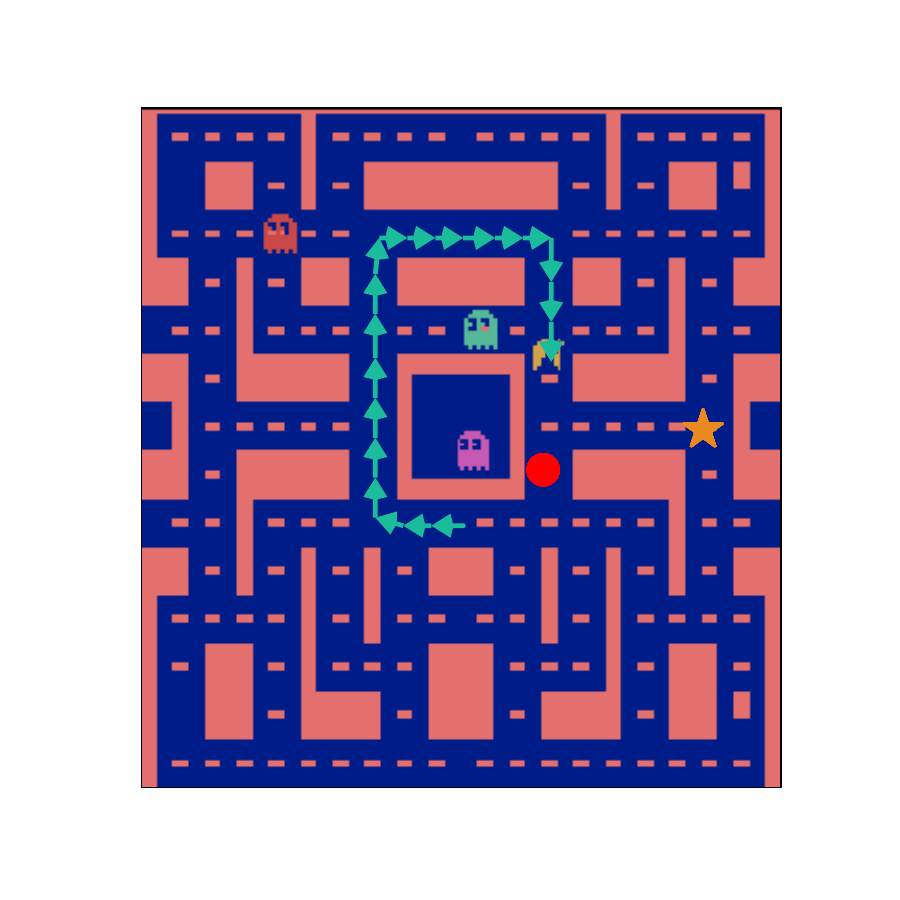}
    \includegraphics[width=0.19\textwidth,  trim = {1cm 1cm 1cm 1cm}]{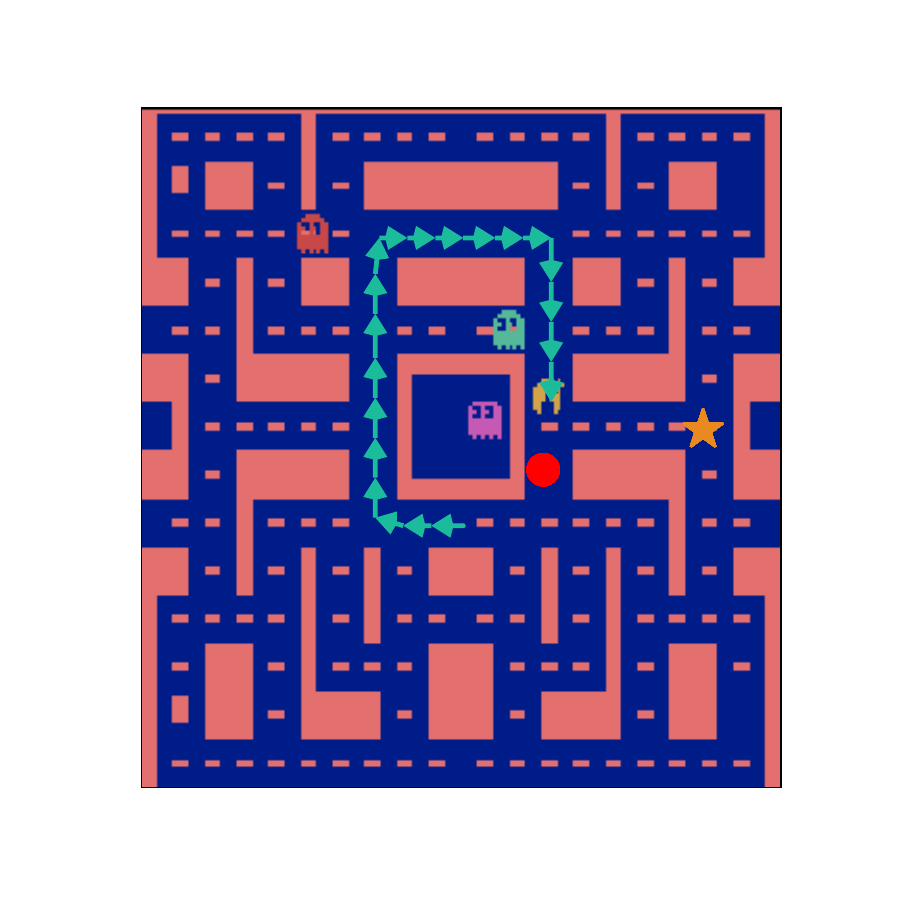} 
     \includegraphics[width=0.19\textwidth,  trim = {1cm 1cm 1cm 1cm}]{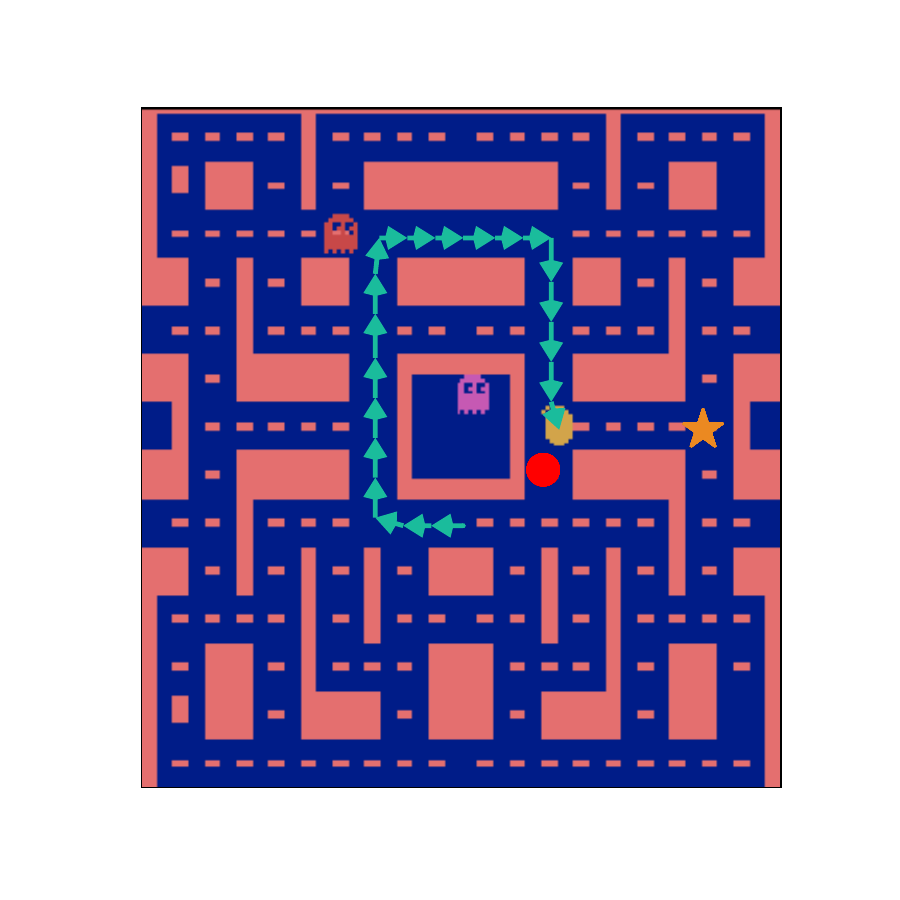}
     \includegraphics[width=0.19\textwidth,  trim = {1cm 1cm 1cm 1cm}]{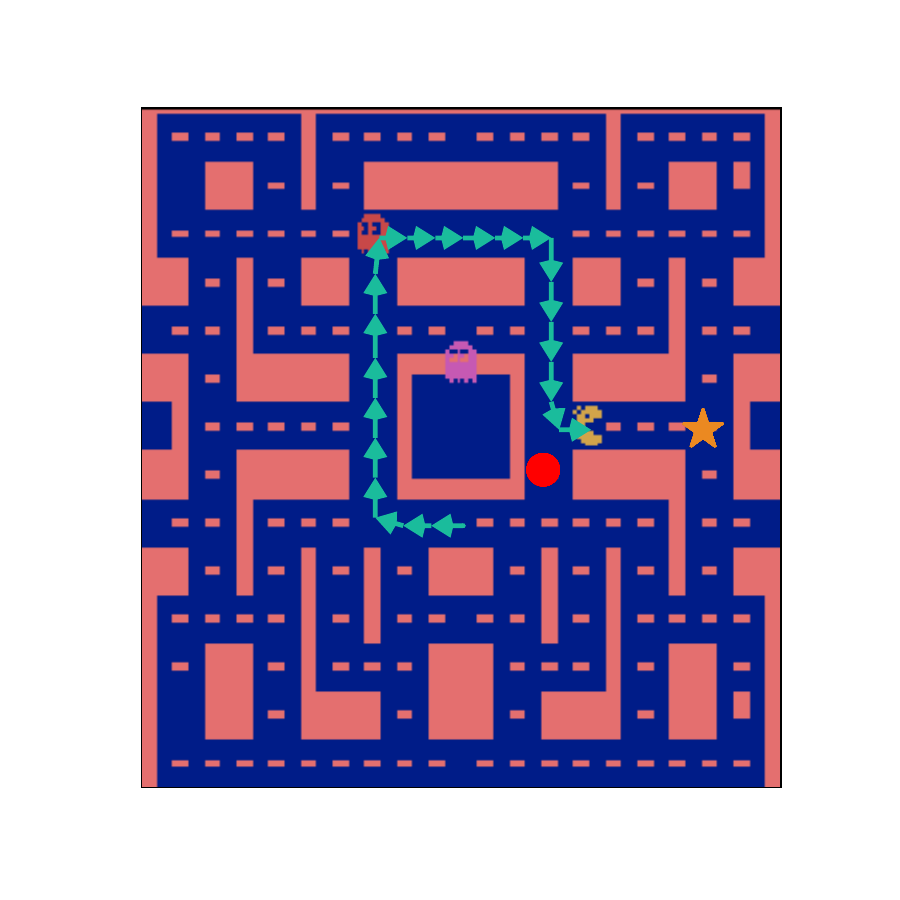} 
     \includegraphics[width=0.19\textwidth,  trim = {1cm 1cm 1cm 1cm}]{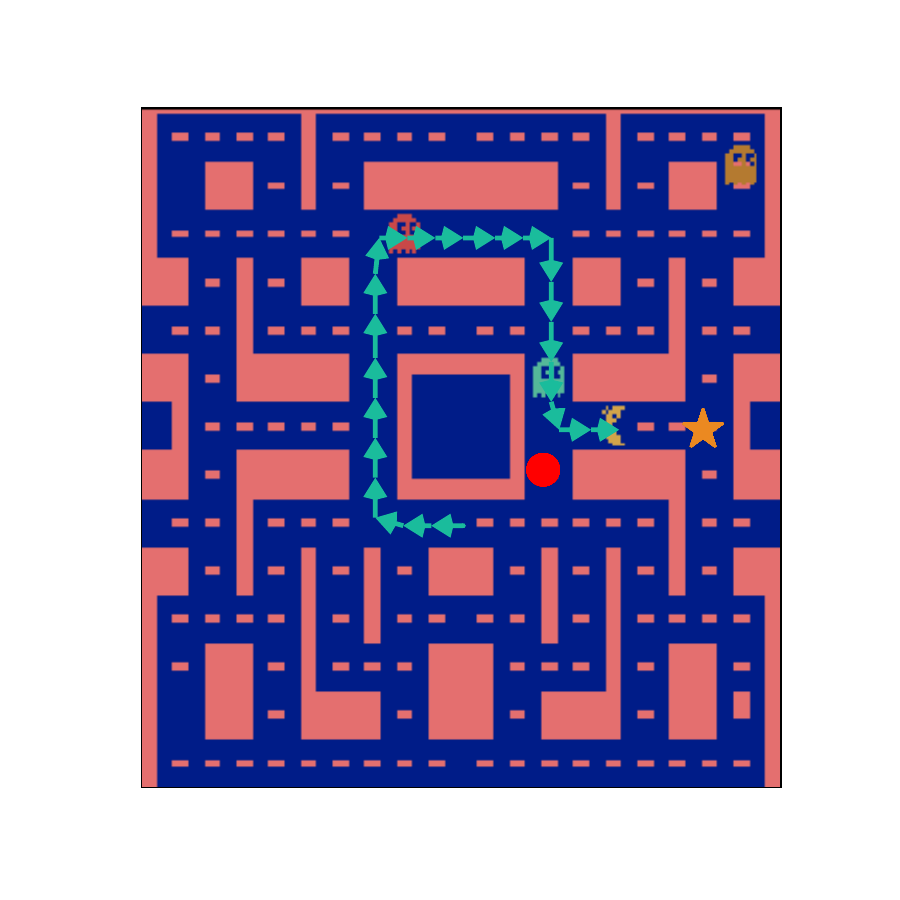} 
     \includegraphics[width=0.19\textwidth,  trim = {1cm 1cm 1cm 1cm}]{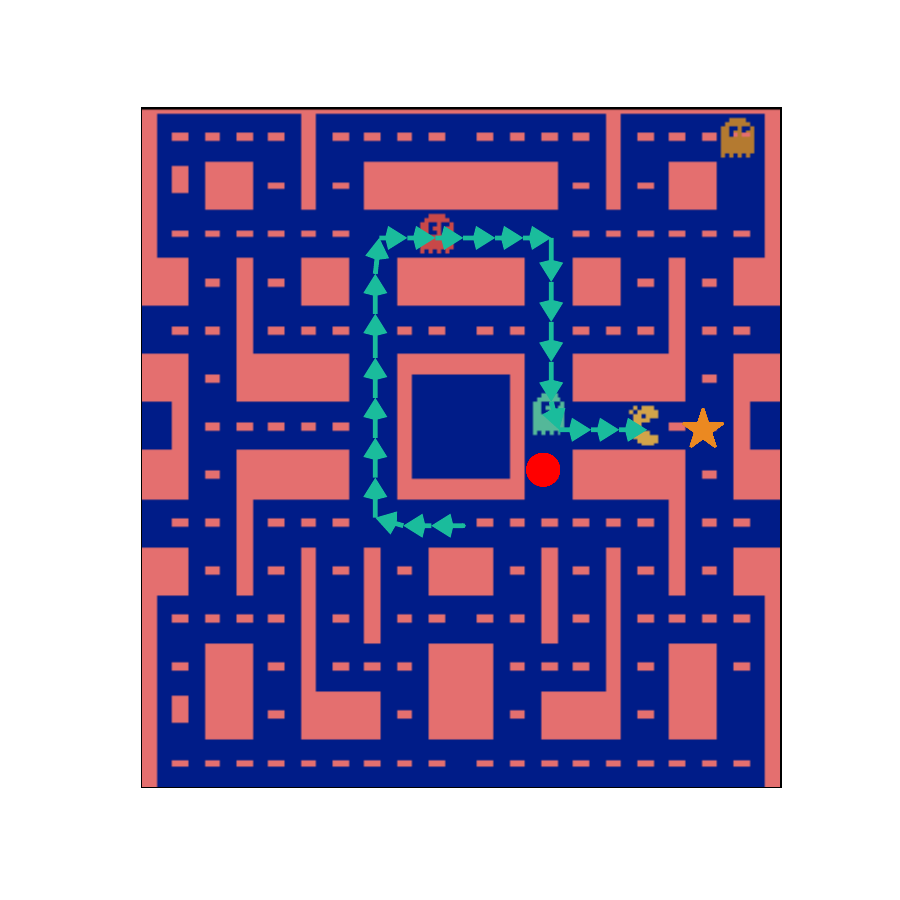}
    \includegraphics[width=0.19\textwidth,  trim = {1cm 1cm 1cm 1cm}]{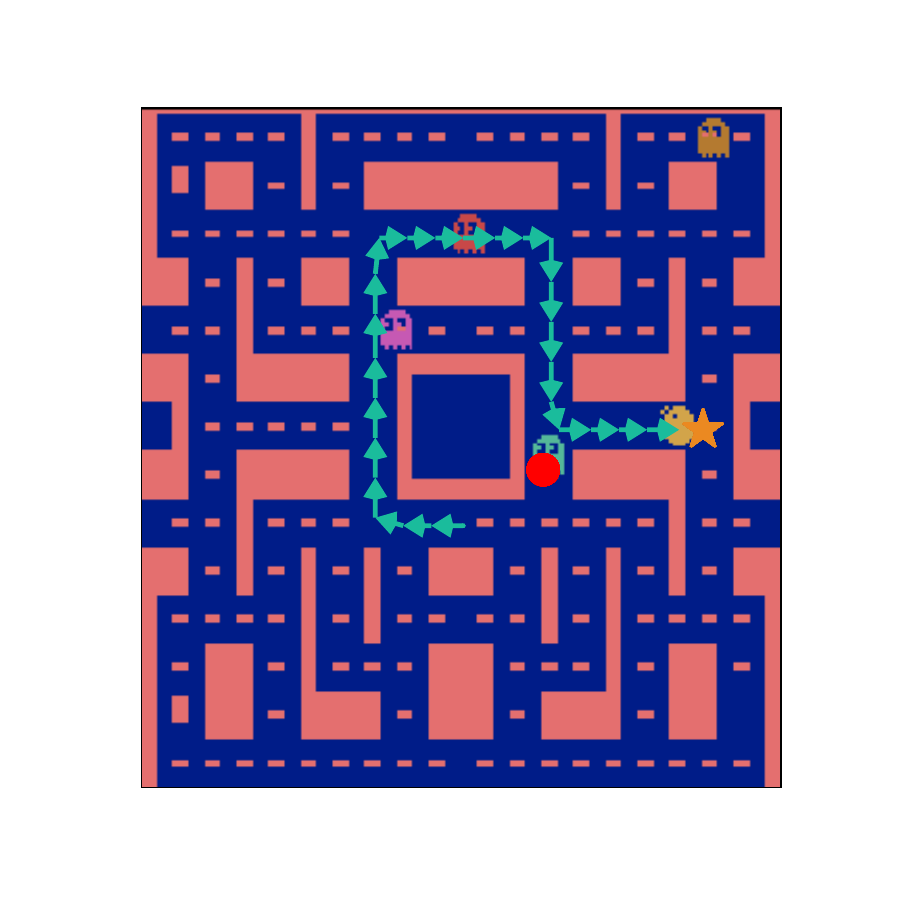} 
    \caption{ Full series of frames in Ms.\ Pacman along the trajectory generated by
    the
    $\transp{F}B$ policy for the task of reaching a target position (star
    shape  \score{1}{1}) while avoiding forbidden positions (red shape
    \protect\tikz\protect\draw[red,fill=red] (0,0) circle (.7ex);).
    }
    \label{fig:fullpacmanframes}
\end{figure}

\subsubsection{Embedding Visualization}

\begin{figure}[h!]
    \centering
    \includegraphics[width=0.3\textwidth, trim = {1cm 1cm 1cm 1cm}]{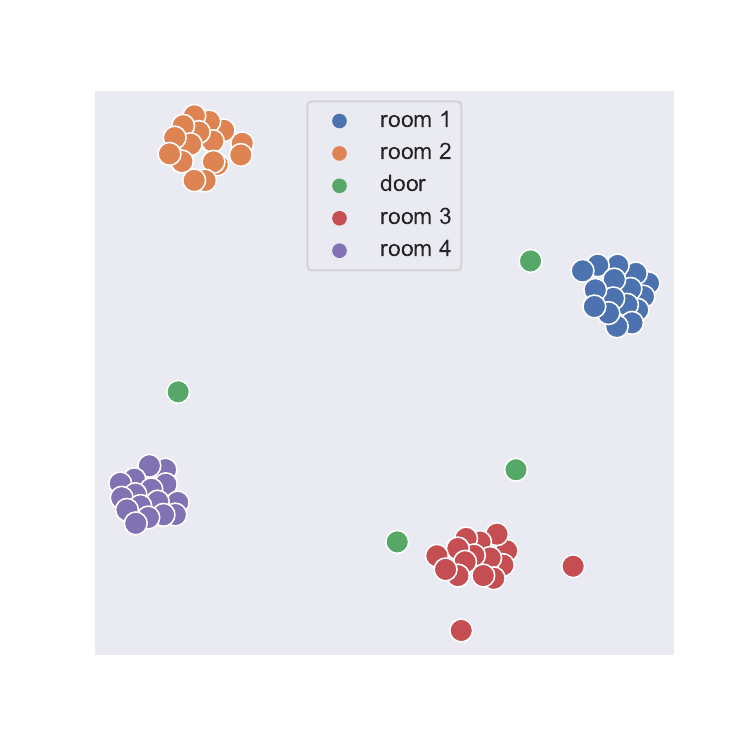} \hspace{10pt}
    \includegraphics[width=0.3\textwidth, trim = {1cm 1cm 1cm 1cm}]{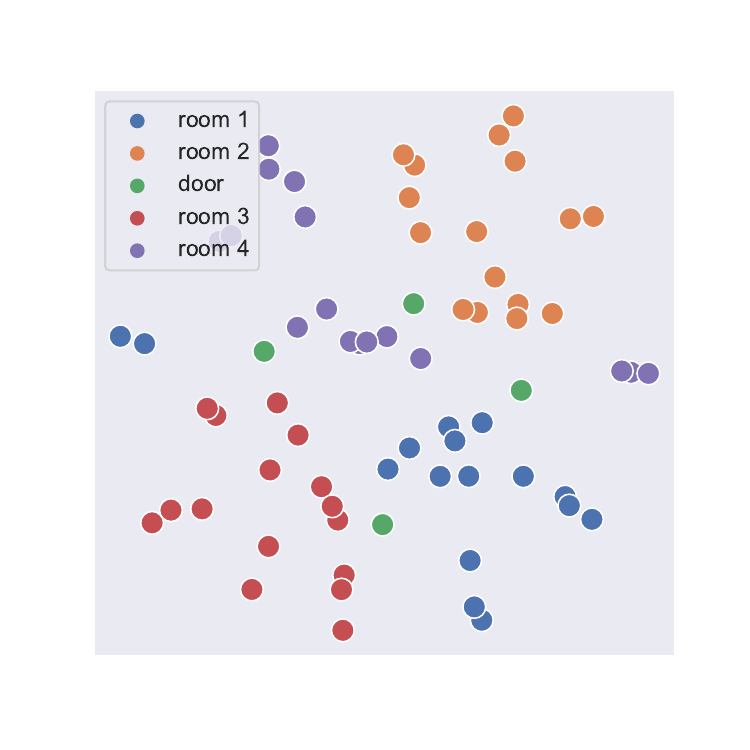}
    \caption{\textbf{Discrete maze}: Visualization of FB embedding vectors after projecting them in two-dimensional space
    with t-SNE. \textbf{Left}: the $F$
    embedding for $z=0$. \textbf{Right}: the $B$ embedding. Note how both embeddings recover the foor-room and door structure of the original environment. The spread of B embedding is due to the regularization that makes B closer to orthonormal.}
\end{figure}

\begin{figure}[h!]
    \centering
    \includegraphics[width=1\textwidth]{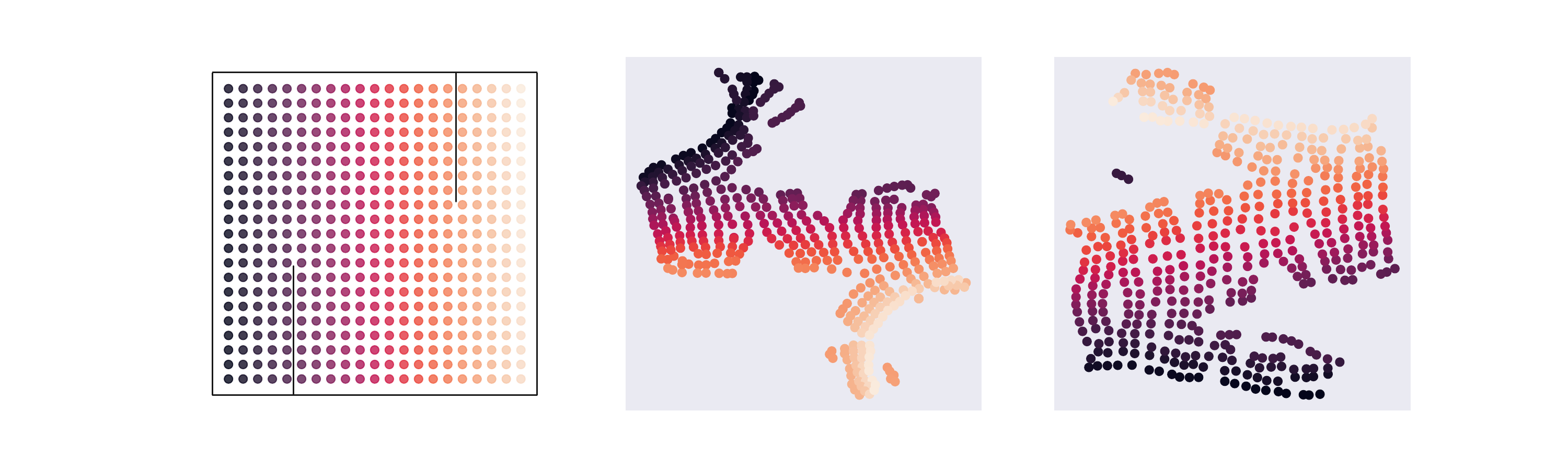} \hspace{10pt}
    \caption{\textbf{Continuous maze}: Visualization of FB embedding vectors after projecting them in two-dimensional space
    with t-SNE. 
    \textbf{Left}: the states to be mapped. \textbf{Middle}: the $F$
    embedding. \textbf{Right}: the $B$ embedding.}
\end{figure}

\begin{figure}[h!]
    \centering
    \includegraphics[width=1\textwidth]{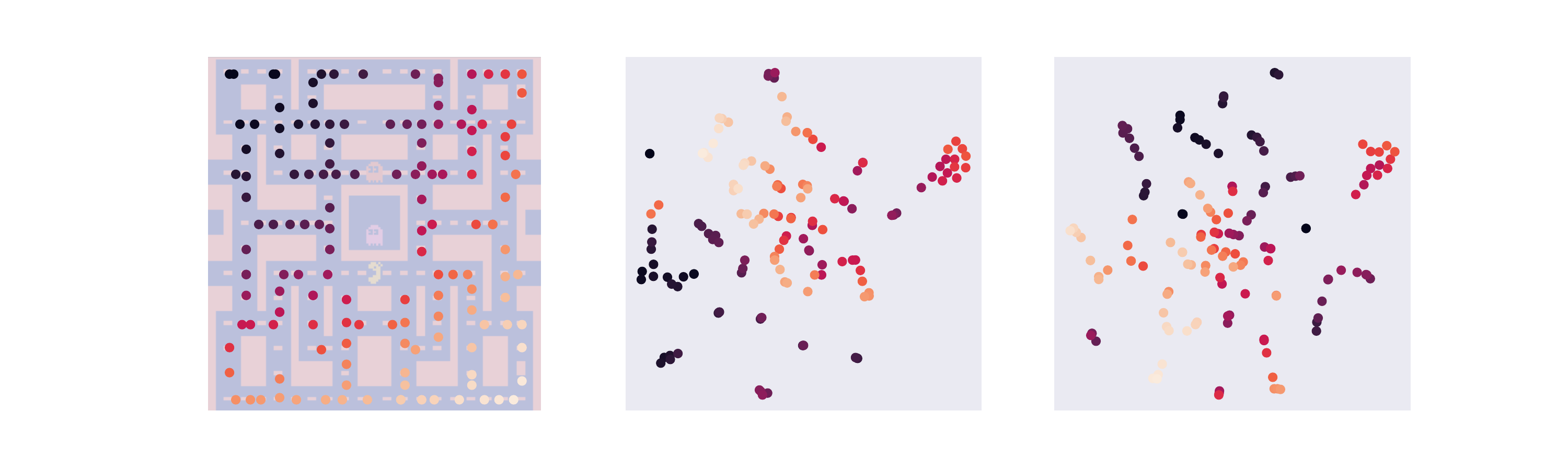} \hspace{10pt}
    \caption{\textbf{Ms.\ Pacman}: Visualization of FB embedding vectors after projecting them in two-dimensional space
    with t-SNE. 
    \textbf{Left}: the agent's position corresponding to the state to be mapped. \textbf{Middle}: the $F$
    embedding for $z=0$. \textbf{Right}: the $B$ embedding. Note how both embeddings recover the cycle structure of the environment. F acts on visual inputs and B acts on the agent's position.}
\end{figure}

\begin{figure}[h!]
    \centering
    \includegraphics[width=1\textwidth]{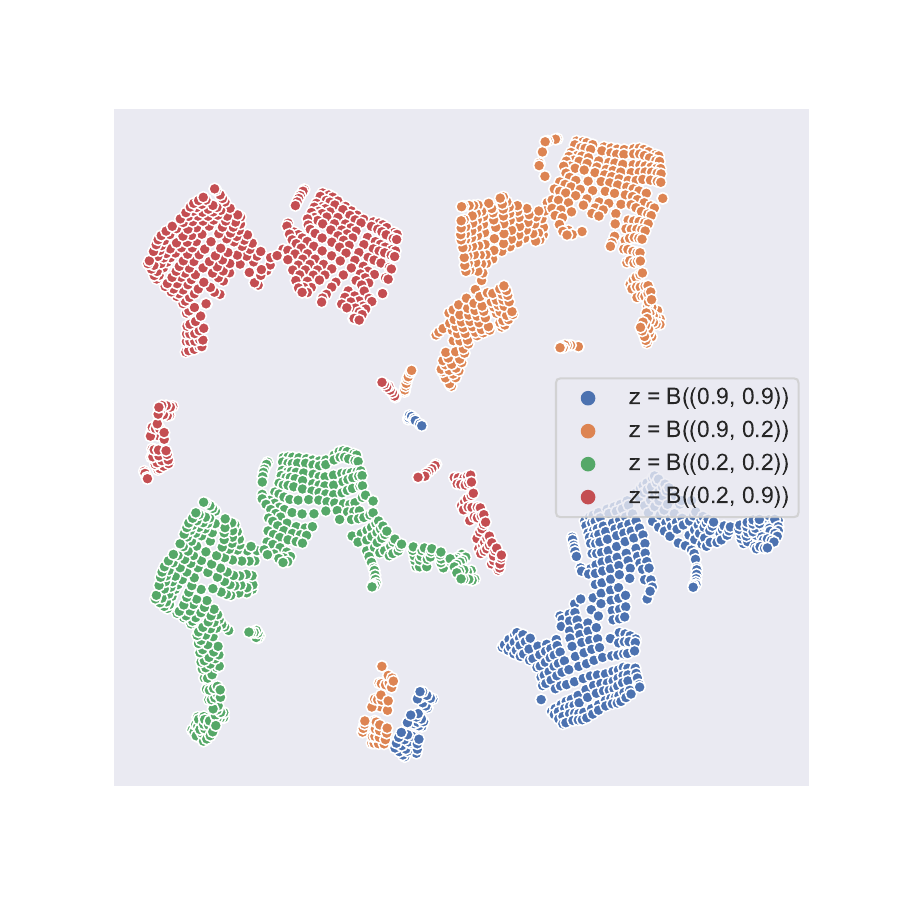} \hspace{10pt}
    \caption{\textbf{Continuous maze}: visualization of $F$ embedding
    vectors for different $z$ vectors, after projecting them in two-dimensional space
    with t-SNE. }
\end{figure}

\end{document}